%% file: main.tex
\pdfoutput=1
\documentclass[11pt]{article}
\usepackage[textwidth=400pt,centering]{geometry}
\usepackage{fullpage}
\usepackage{microtype}
\usepackage{graphicx}
\usepackage{subfigure}
\usepackage{booktabs} 
\usepackage[sort]{natbib}
\usepackage[nottoc,notlot,notlof]{tocbibind}
\usepackage[vlined, linesnumbered, ruled]{algorithm2e}
\input{math_commands.tex}

\usepackage[hidelinks]{hyperref}
\hypersetup{
    colorlinks,
    linkcolor={red!50!black},
    citecolor={blue!50!black},
    urlcolor={blue!80!black}
}

\usepackage{csquotes}
\usepackage{thmtools,thm-restate}
\usepackage[capitalize,noabbrev]{cleveref}

\theoremstyle{plain}
\newtheorem{theorem}{Theorem}
\newtheorem{proposition}[theorem]{Proposition}
\newtheorem{lemma}[theorem]{Lemma}

\theoremstyle{definition}
\newtheorem{definition}[theorem]{Definition}
\newtheorem{assumption}{Assumption}

\theoremstyle{remark}
\newtheorem{remark}[theorem]{Remark}

\usepackage[utf8]{inputenc} 
\usepackage[T1]{fontenc}    
\usepackage{url}            
\usepackage{nicefrac}       
\usepackage{xcolor}    

\usepackage{pifont}
\usepackage{bbm}
\usepackage{multirow}
\usepackage{braket}
\usepackage{enumitem}
\definecolor{darkblue}{rgb}{0,0.08,0.8}
\newcommand\numberthis{\addtocounter{equation}{1}\tag{\theequation}}

\usepackage{verbatim}

\newcommand{\Fr}{Fr\'{e}chet }
\newcommand{\RV}{\mathrm{RV}}
\newcommand{\ul}{\underline{\lambda}}
\newcommand{\uL}{\underline{L}}

\usepackage{calc}  
\usepackage{array} 

\title{Follow-the-Perturbed-Leader with Fr\'{e}chet-type Tail Distributions: \\
            Optimality in Adversarial Bandits and Best-of-Both-Worlds}
\author{Jongyeong Lee$^{1}$ \\ \fontsize{10}{12}\selectfont jongyeong@snu.ac.kr \and Junya Honda$^{2,3}$ \\ \fontsize{10}{12}\selectfont honda@i.kyoto-u.ac.jp \and Shinji Ito$^{3,4}$ \\ \fontsize{10}{12}\selectfont i-shinji@nec.com \and Min-hwan Oh$^1$ \\ \fontsize{10}{12}\selectfont minoh@snu.ac.kr}
\date{\fontsize{10}{12}\selectfont
$^1$ Seoul National University
$^2$ Kyoto University
$^3$ RIKEN AIP
$^4$ NEC Corporation
}
\usepackage{times}

\begin{document}
\maketitle

\begin{abstract}%
This paper studies the optimality of the Follow-the-Perturbed-Leader (FTPL) policy in both adversarial and stochastic $K$-armed bandits.
Despite the widespread use of the Follow-the-Regularized-Leader (FTRL) framework with various choices of regularization, the FTPL framework, which relies on random perturbations, has not received much attention, despite its inherent simplicity.
In adversarial bandits, there has been conjecture that FTPL could potentially achieve $\mathcal{O}(\sqrt{KT})$ regrets if perturbations follow a distribution with a Fr\'{e}chet-type tail.
Recent work by \citet{pmlr-v201-honda23a} showed that FTPL with Fr\'{e}chet distribution with shape $\alpha=2$ indeed attains this bound and, notably logarithmic regret in stochastic bandits, meaning the Best-of-Both-Worlds (BOBW) capability of FTPL.
However, this result only partly resolves the above conjecture because their analysis heavily relies on the specific form of the Fr\'{e}chet distribution with this shape.
In this paper, we establish a sufficient condition for perturbations to achieve $\mathcal{O}(\sqrt{KT})$ regrets in the adversarial setting, which covers, e.g., Fr\'{e}chet, Pareto, and Student-$t$ distributions.
We also demonstrate the BOBW achievability of FTPL with certain Fr\'{e}chet-type tail distributions. 
Our results contribute not only to resolving existing conjectures through the lens of extreme value theory but also potentially offer insights into the effect of the regularization functions in FTRL through the mapping from FTPL to FTRL.
\end{abstract}

\section{Introduction}
In the multi-armed bandit (MAB) problem, an agent plays an arm $I_t$ from a set of $K$ arms at each round $t \in [T] :=\qty{1, \ldots, T}$ over a time horizon $T$.
The agent only observes the loss $\ell_{t,I_t}$ generated from the played arm, where the loss vectors $\ell_t = \qty(\ell_{t,1}, \ldots, \ell_{t,K})^\top \in [0,1]^K$ are determined by the environment.
Given the constraints of partial feedback, the agent must handle the tradeoff between gathering information about the arms and playing arms strategically to minimize total loss.
The performance of the policy is measured by pseudo-regret, defined as $\E[\sum_t \ell_{t,I_t}] - \min_{i} \E[\sum_t \ell_{t,i}]$. 

There are two primary formulations of the environment to determine loss vectors: the stochastic setting~\citep{lai1985asymptotically, katehakis1995sequential}, and the adversarial setting~\citep{auer2002nonstochastic, audibert2009minimax}.
In the stochastic setting, the loss vector $\ell_t$ is independent and identically distributed (i.i.d.) from an unknown but fixed distribution $\gD$ over $[0,1]^K$.
Therefore, one can define the expected losses of arms $\mu_i := \E_{\ell \sim \gD}[\ell_i]$ and the optimal arm $i^* \in \argmin_{i \in [K] } \mu_i$.
The suboptimality gap of each arm is denoted by $\Delta_i = \mu_i - \mu_i^*$ and the optimal problem-dependent regret bound is known to be $\sum_{i: \Delta_i >0}\mathcal{O}\qty(\frac{\log T}{\Delta_i})$~\citep{lai1985asymptotically}, which can be achieved by several policies such as UCB~\citep{auer2002finite} and Thompson sampling~\citep{agrawal2017near,riou2020bandit}.

On the other hand, in the adversarial setting, an (adaptive) adversary determines the loss vector based on the history of the decisions, and thus specific assumptions about the loss distribution
are not
made.
In this particular environment, the optimal regret bound stands at $\mathcal{O}(\sqrt{KT})$~\citep{auer2002nonstochastic} and some Follow-The-Regularized-Leader (FTRL) policies have demonstrated their capability to attain this bound~\citep{audibert2009minimax,zimmert2019connections}.

In
practical scenarios, a priori knowledge regarding the nature of the environment is often unavailable. 
Therefore, there arises a need for an algorithm that can adeptly address both stochastic and adversarial settings at the same time.
While several policies have been proposed to tackle this problem~\citep{bubeck2012best,seldin2017improved}, the Tsallis-INF policy, based on FTRL framework, has demonstrated its effectiveness in achieving optimality in \emph{both} setting~\citep{zimmert2021tsallis}, a status referred to as the Best-of-Both-Worlds (BOBW)~\citep{bubeck2012best}.
Moreover, FTRL framework has been successfully adapted to achieve BOBW in various domains such as combinatorial semi-bandits~\citep{ito21a, pmlr-v206-tsuchiya23a}, linear bandits~\citep{lee2021achieving, dann2023blackbox}, dueling bandits~\citep{saha2022versatile} and partial monitoring~\citep{tsuchiya23a}.

However, FTRL policies require the explicit computation of the probability of arm selections per step, by solving an optimization problem in general.
In light of this limitation, the Follow-the-Perturbed-Leader (FTPL) framework, which simply selects the arm with the minimum cumulative estimated loss along with a random perturbation, has gained attention for its computational efficiency in adversarial bandits~\citep{abernethy2015fighting}, combinatorial semi-bandits~\citep{neu2015first}, and linear bandits~\citep{mcmahan2004online}.
It has been established that FTPL, when coupled with perturbations satisfying several conditions, can achieve nearly optimal $\mathcal{O}\qty(\sqrt{KT\log K})$ regret in adversarial bandits~\citep{abernethy2015fighting, NEURIPS2019_aff0a6a4}.
Subsequently, \citet{NEURIPS2019_aff0a6a4} conjectured that if FTPL achieves minimax optimality, then the corresponding perturbations should be of Fr\'{e}chet-type tail distribution.

Recently, \citet{pmlr-v201-honda23a} showed that FTPL with Fr\'{e}chet perturbations with shape $\alpha = 2$ indeed achieves $\mathcal{O}(\sqrt{KT})$ regret in adversarial bandits and $\mathcal{O}\qty(\sum_i \frac{\log T}{\Delta_i})$ regret
in stochastic bandits, highlighting the effectiveness of FTPL. 
However, their analysis heavily relies on the specific form of Fr\'{e}chet distribution, providing only a partial solution to the above conjecture.
It is noteworthy that any FTPL policy can be expressed as FTRL policy~\citep{abernethy2016perturbation}.
Therefore, investigating the properties of more general perturbations not only extends our understanding of FTPL but also can clarify the impact of regularization functions used in FTRL, where several regularization functions in FTRL beyond Tsallis entropy have been used to achieve BOBW in various settings~\citep{jin2023improved}.
\vspace{-0.1em}
\paragraph{Contribution}
This paper proves that FTPL with Fr\'{e}chet-type tail distributions satisfying some mild conditions can achieve $\mathcal{O}(\sqrt{KT})$ regret in adversarial bandits, which resolves an open question raised by \citet{NEURIPS2019_aff0a6a4} comprehensively.
Moreover, we provide a problem-dependent regret bound in stochastic bandits, demonstrating that some of them can achieve BOBW, which generalizes the results of \citet{pmlr-v201-honda23a}.
Given that our analysis is grounded in the language of extreme value theory, we expect that our analysis can provide insights for constructing an FTPL counterpart of FTRL in settings beyond the standard MAB.
\vspace{-0.6em}
\section{Preliminaries}
\vspace{-0.2em}
In this section, we formulate the problem and provide a brief overview of extreme value theory and the framework of regular variation,
based on which Fr\'echet-type tail is formulated. 
For a thorough understanding of extreme value theory and related discussions, we refer the reader to Appendix~\ref{app: EVT} and the references therein.

\subsection{Problem formulation}
At every round $t \in [T]$, the environment determines the loss vector $\ell_t = (\ell_{t,1}, \ldots, \ell_{t,K}) \in [0,1]^K$ through either a stochastic or adversarial process. 
Then the agent plays an arm $I_t$ according to their policy and observes the corresponding loss $\ell_{t,I_t}$ of the played arm.
Then, the pseudo-regret, a measure to evaluate the performance of a policy, is defined as
\begin{equation*}
    \gR(T) = \E\qty[\sum_{t=1}^T (\ell_{t,I_t} - \ell_{t,i^*})],\quad i^* \in \argmin_{i \in [K]} \E\qty[\sum_{t=1}^T \ell_{t,i}],
\end{equation*}
where $i^*$ denotes the optimal arm.
Since only partial feedback is available, FTRL and FTPL policies use an estimator $\hat{\ell}_t$ of the loss vector $\ell_t$ specified in Section~\ref{sec: FTPL}.
We denote the cumulative loss at round $t$ by $L_{t} = \sum_{s=1}^{t-1} \ell_{t}$ and its estimation by $\hat{L}_{t}=\sum_{s=1}^{t-1} \hat{\ell}_s$.
\vspace{-0.2em}
\subsection{Follow-the-Perturbed-Leader policy}\label{sec: FTPL}
In the MAB problems, FTPL is a policy that plays an arm
\begin{equation*}
    I_t \in \argmin_{i\in [K]} \qty{\hat{L}_{t,i} - \frac{r_{t,i}}{\eta_t}},
\end{equation*}
where $\eta_t$ denotes the learning rate specified later and $r_{t} = (r_{t,1}, \ldots, r_{t,K})$ denotes the random perturbation i.i.d. from a common distribution $\gD$ with a distribution function $F$.
Then, the probability of playing an arm $i \in [K]$ given $\hat{L}_t$ is written as $w_{t,i} = \phi_i(\eta_t \hat{L}_t; \gD)$, where for $\lambda \in [0, \infty)^K$ 
\begin{align*}
     \phi_{i}(\lambda ;\gD) &:= \Pr_{r_1, \ldots, r_K \sim \gD}\qty[i = \argmin_{j \in [K]} \qty{\lambda_j - r_j}] \\
     &= \int_{\nu-\min_{j \in [K]} \lambda_j}^{\infty}\prod_{j\ne i} F(z+\lambda_j)\, \dd F(z+\lambda_i) \\
    &= \int_{\nu}^{\infty}  \prod_{j\ne i} F(z+\ul_j)\, \dd F(z+\ul_i), \numberthis{\label{eq: def_phi}}
\end{align*}
where $\nu$ denotes the left endpoint of the support of $F$.
Here, underlines denote the gap of a vector from its minimum, i.e., $\ul = \lambda - \1 \min_{i\in [K]} \lambda_i$ for all-one vector $\1$.

For the unbiased loss estimator, FTRL policies often employ an importance-weighted estimator, $\hat{\ell}_t = (\ell_{t,I_t}/w_{t,I_t}) e_{I_t}$, where $w_{t,I_t}$ is explicitly computed.
On the other hand in FTPL, we use an unbiased estimator $\widehat{w^{-1}_{t,i}}$ of $w^{-1}_{t,i}$ by geometric resampling~\citep{neu2016importance}, whose pseudo-code is given in Lines 6--10 of Algorithm~\ref{alg: FTPL}.
Simply speaking, the process involves repeated samplings of perturbations $r'$ until $\argmin_i \qty{\hat{L}_{t,i} - r_{t,i}'/\eta_t}$ coincides with $I_t$ and $\widehat{w^{-1}_{t,i}}$ is then set as the number of resampling.
For more details, refer to \citet{neu2016importance} and \citet{pmlr-v201-honda23a}.
\begin{algorithm}[t]
   \caption{FTPL with geometric resampling}
   \label{alg: FTPL}
   \DontPrintSemicolon
   \SetAlgoLined
   \SetKwInOut{Initialization}{Initialization}
   \Initialization{$\hat{L}_1 = 0$ and set distribution $\gD$}
   \For{$t=1$ to $T$}{
      Sample $r_t = (r_{t,1}, \ldots, r_{t,K})$ i.i.d. from $\gD$.
      
      Play $I_t \in \argmin_{i\in[K]} \left\{\hat{L}_{t,i} - \frac{r_{t,i}}{\eta_t}\right\}$.
      
      Observe $\ell_{t,I_t}$ and set $m=0$. \\
      \Repeat{$I_t = \argmin_{i\in[K]} \left\{\hat{L}_{t,i} - \frac{r_{i}'}{\eta_t}\right\}$}{
         $m:= m+1$.  \tcp*{Geometric resampling}
         
         Sample $r' = (r_1', \ldots, r_K')$ i.i.d. from $\gD$.
      }
      Set $\widehat{w_{t,I_t}^{-1}} := m$ and $\hat{L}_{t+1} := \hat{L}_t + \ell_{t,I_t} \widehat{w_{t,I_t}^{-1}} e_{I_t}$.
   }
\end{algorithm}
\vspace{-0.2em}
\subsection{\Fr maximum domain of attraction}
In the adversarial setting, it has been conjectured that FTPL might achieve $\mathcal{O}(\sqrt{KT})$ regrets if perturbations follow a distribution with a Fr\'{e}chet-type tail~\citep{NEURIPS2019_aff0a6a4}.
In the following, we explain the terminology and basic concepts related to this description. 

Extreme value theory is a branch of statistics to study the distributions of maxima of random variables. 
One of the most important results in this theory is that the distribution of the maxima of i.i.d. random variables can \emph{only} converge in distribution to three types of extreme value distributions: Fr\'{e}chet, Gumbel, and Weibull, after appropriate normalization~\citep{fisher1928limiting, gnedenko1943distribution}. 
Among these, a distribution is called Fr\'{e}chet-type if its limiting distribution is Fr\'{e}chet distribution. 
The family of Fr\'{e}chet-type distributions is called Fr\'{e}chet maximum domain of attraction (FMDA), and its representation is known to be associated with the notion of regular variation~\citep{embrechts1997modelling,haan2006extreme, resnick2007heavy} defined as follows.
\vspace{-0.2em}
\begin{definition}[Regular variation~\citep{haan2006extreme}]
    An eventually positive function $g$ is called \emph{regularly varying} at infinity with index $\alpha$, $g\in \RV_{\alpha}$ if \vspace{-0.2em} 
    \begin{equation*}
        \lim_{x \to \infty} \frac{g(tx)}{g(x)} = t^{\alpha} , \quad \forall t >0. \vspace{-0.2em}
    \end{equation*}
    If $g(x)$ is regularly varying with index $0$, then $g$ is called \emph{slowly varying}.
\end{definition}
From the definition, one can see that any regularly varying function with index $\alpha$ can be written with a product of a slowly varying function and $x^{\alpha}$, i.e., if $g \in \RV_\alpha$, then $g = x^\alpha S(x)$ for some $S \in \RV_0$ and all $x>0$.
A necessary and sufficient condition for a distribution to belong to FMDA is known to be expressed in terms of regular variation as shown below. \vspace{-0.2em}
\begin{proposition}[\citet{gnedenko1943distribution, resnick2008extreme}]\label{def: FMDA}
    A distribution $\gD_\alpha$ belongs to FMDA with index $\alpha >0$ if and only if its right endpoint is infinite and the tail function, $1-F$, is regularly varying at infinity with index $-\alpha$, i.e., $1-F \in \RV_{-\alpha}$. 
    In this case, \vspace{-0.2em}
    \begin{equation}\label{eq: limiting}
        F^n(a_n x) \to  \begin{cases}
            \exp(-x^{-\alpha}), &x \geq 0, \\
            0, & x <0,
        \end{cases} \quad n \to \infty,
    \end{equation} 
    where $a_n = \inf\qty{x: F(x) \geq 1-\frac{1}{n}}$. 
\end{proposition}
Let $\fD_\alpha^{\text{all}}$ denote the class of FMDA with index $\alpha>0$.
From its definition, if $\gD \in \fD_{\alpha}^{\text{all}}$, we can express the tail distribution with $S_F \in \RV_0$ as
\begin{equation}\label{eq: tail and S}
    1-F(x) = x^{-\alpha} S_F(x), \quad \forall x > 0.
\end{equation}
In other words, a Fr\'{e}chet-type tail distribution can be characterized by a slowly varying function $S_F$ and an index $\alpha$, where Table~\ref{tab: fr} provides examples of well-known distributions and their associated slowly varying functions.

\begin{table}[t]
    \caption{Some well-known Fr\'{e}chet-type tail distributions with parameters $\alpha,\beta,m,n>0$.
    $S_F(x)$ denotes the corresponding slowly varying function that characterizes the tail distribution.
    More examples such as LogGamma can be found in \citet[Table 2.1]{beirlant2006statistics}.
    Here, $B(a,b)$ and $B(x;a,b)$ denote the Beta function and incomplete Beta function, respectively.}
    \vspace{0.3em}
    \centering
    \resizebox{\textwidth}{!}{
    \begin{tabular}{c|ccc|cc}
        Distribution ($\gD$)  &  $1-F(x)$ & $f(x)$ & $S_F(x)$ & Support &  Index \\
        \hline
        \Fr ($\gF_\alpha$) & $1-e^{-x^{-\alpha}}$ & $\alpha \frac{e^{-x^{-\alpha}}}{x^{\alpha+1}}$ & $x^{\alpha}(1-e^{-x^{-\alpha}})$&  $x >0$ & $\alpha$\\
        Pareto ($\gP_\alpha$) & $x^{-\alpha}$ & $\frac{\alpha}{x^{\alpha+1}}$ &  $1$ & $x \geq 1$ &$\alpha$ \\
        Generalized Pareto ($\gG\gP_{\alpha,\beta}$) & $\qty(1+ \frac{x}{\alpha \beta})^{-\alpha}$ & $\frac{1}{\beta}\qty(1+ \frac{x}{\alpha \beta})^{-(\alpha+1)}$ &  $(\alpha \beta)^{\alpha} \qty(1+\frac{\alpha \beta}{x})^{-\alpha}$ &$x\geq 0$ & $\alpha$ \\ 
        Student-t ($\gT_n$) & $\int_{-\infty}^x \frac{\qty(1+t^2/n)^{-\frac{n+1}{2}}}{\sqrt{n} B(n/2,1/2)} \dd t$ & $\frac{1}{\sqrt{n} B(n/2,1/2)}\qty(1+\frac{x^2}{n})^{-\frac{n+1}{2}}$ & $\frac{\Gamma((n+1)/2)}{\sqrt{\pi n} \Gamma(n/2)} n^{\frac{n-1}{2}}\qty(1-\frac{n^2(n+1)}{2(n+2)}x^{-2}+o(x^{-2}))$ & $\sR$ & $n$ \\
        Snedecor's F ($\gS_{m,n}$) & $1-\frac{B\qty(\frac{mx}{mx+n};\frac{m}{2}, \frac{n}{2})}{B\qty(\frac{m}{2}, \frac{n}{2})}$ & $\frac{\qty(m/n)^{\frac{m}{2}}}{B\qty(\frac{m}{2}, \frac{n}{2})}  x^{\frac{m}{2}-1} \qty(1+\frac{m}{n}x)^{-\frac{m+n}{2}}$ & $\frac{\qty(m/n)^{\frac{m}{2}}}{B\qty(\frac{m}{2}, \frac{n}{2})}\qty(\frac{m}{n} + \frac{1}{x})^{-\frac{m+n}{2}}(1+o(1)) $ & $x>0$ & $\frac{n}{2}$ \\
        \hline
    \end{tabular}
    }
    \vspace{-1em}
    \label{tab: fr}
\end{table}

Notably, $\fD_{\alpha}^{\text{all}}$ encompasses exceptionally diverse distributions since its definition generally allows for any slowly varying functions, even those that are discontinuous.
In this paper, we consider a set of Fr\'{e}chet-type distributions denoted by $\fD_\alpha \subset \fD_\alpha^{\text{all}}$, which is defined as follows.
\vspace{-0.5em}
\begin{definition}\label{def: Fr set}
    $\fD_\alpha$ is a set of distributions that belong to FMDA with index
    $\alpha>0$
    satisfying the following assumptions. 
\vspace{-0.6em}
\begin{assumption}\label{asm: decreasing f}
    $F(x)$ has a density function $f(x)$ that is decreasing in $x\geq z_0$ for some $z_0> \nu$.
\end{assumption}
\begin{assumption}\label{asm: bounded hazard}
    $\fD_\alpha$ is supported over $[\nu,\infty)$ for some $\nu\ge 0$ and the hazard function $\frac{f(x)}{1-F(x)}$ is bounded.
\end{assumption}
\begin{assumption}\label{asm: bounded block}
    There exist positive constants $M=M(\gD_\alpha)$ and $m=m(\gD_\alpha)$ satisfying \vspace{-0.3em}
    \begin{align} 
         \E_{X_1,\dots,X_k\sim \gD_{\alpha}}\qty[\max_{i \in  [k]}X_i/a_k] &\leq M  \label{asm: block upper} \\
       \E_{X_1,\dots,X_k\sim \gD_{\alpha}}\qty[\frac{1}{\max_{i \in  [k]}X_i/a_k}] &\leq m \label{asm: block constant} 
       \vspace{-0.2em}
    \end{align}
    for $a_k=\inf\qty{x: F(x) \geq 1-1/k}$ and and it satisfies $A_l k^{\frac{1}{\alpha}} \leq a_k \leq A_u k^{\frac{1}{\alpha}}$ for some positive constants $A_l, A_u$.
\end{assumption}
\begin{assumption}\label{asm: derivative of f}
    $\lim_{x\to \infty} \frac{-xf'(x)}{f(x)} = \alpha+1$ and $\frac{-f'(x)}{f(x)}$ is bounded almost everywhere on $[\nu,\infty)$.
\end{assumption}
\begin{assumption}\label{asm: I is increasing}
    $\frac{f(x)}{F(x)}$ is monotonically decreasing in $x \geq \nu$.
\end{assumption}
\vspace{-1em}
\end{definition}
These assumptions offer easy-to-check \emph{sufficient} conditions for perturbations to achieve the optimal order and verifying necessary conditions would be interesting for future work.
In the following, we explain the implication of the assumptions in Definition~\ref{def: Fr set}.

Assumption \ref{asm: decreasing f} states that the density eventually monotonically decreases and does not have a fluctuated tail.
This is known as a sufficient condition that
$\gD_\alpha \in \fD_{\alpha}^{\text{all}}$
satisfies
von Mises condition
(\citealp{von1936distribution}, see also \citealp[Proposition 1.15]{resnick2008extreme}),
which is given by \vspace{-0.2em}
    \begin{equation}\label{eq: vMC}
        \lim_{x\to \infty} \frac{xf(x)}{1-F(x)} = \alpha.
    \end{equation}
The von Mises condition is known to play an important role in the analysis of the FMDA.
For example, it is known that
any $\gD_\alpha \in \fD_{\alpha}^{\text{all}}$ (possibly without a density)
is tail-equivalent to some distribution in $\fD_{\alpha}^{\text{all}}$ satisfying von Mises condition \citep[Corollary 3.3.8]{embrechts1997modelling}.
Here, a distribution $F(x)$ is called to be tail-equivalent to $F^*(x)$
if they have the same right endpoint $x_r$ and $\lim_{x\to x_r}(1-F(x))/(1-F^*(x))=c$ for some constant $c>0$.

In Assumption~\ref{asm: bounded hazard}, the bounded hazard function is also assumed in the existing analysis of near-optimality in adversarial bandits~\citep{abernethy2015fighting, NEURIPS2019_aff0a6a4}.
The assumption of the nonnegative left-endpoint $\nu\ge 0$ is mainly for notational simplicity.
This is because $S_F(x)$ in (\ref{eq: tail and S}) is not well-defined for $x\le 0$.
Although the requirements in Assumption~\ref{asm: bounded hazard} are not satisfied for some distributions such as $t$-distribution, we can easily construct a tail-equivalent distribution satisfying the assumption by considering the truncated version $F^*$ of $F$ given by
\vspace{-0.1em}
\begin{equation}\label{eq: conditioning trick}
   F^*(x) = \Pr[X \geq 1+ x | X>1] = \frac{F(x+1)-F(1)}{1-F(1)}, \quad x >0, \vspace{-0.1em}
\end{equation}
which is also considered in~\citet[Appendix B.2]{abernethy2015fighting}.

\begin{table}[t]
    \caption{Verification of distributions in (\ref{eq: conditioning trick}) whether satisfying the assumptions.
    $\checkmark$ and $\cross$ denote whether the distribution satisfies the assumption or not, respectively, regardless of the parameters. $(*)$ denotes that  the truncated distribution in (\ref{eq: conditioning trick}) satisfies the assumption.}
    \vspace{0.3em}
    \centering
   \begin{tabular}{c|cccccc}
          Distribution ($\gD$) & $\gF_\alpha$ & $\gP_\alpha$ & $\gG\gP_{\alpha,\beta}$ & $\gT_n$  & $\gS_{m,n}$ \\
          \hline 
          Assumption~\ref{asm: decreasing f} & $\checkmark$ &  $\checkmark$ &  $\checkmark$  &  $\checkmark$ &   $\checkmark$ \\
          Assumption~\ref{asm: bounded hazard} & $\checkmark$ & $\checkmark$ & $\checkmark$& $\cross$ (*) &  $\cross$ (*)  \\
          Assumption~\ref{asm: bounded block} & $\checkmark$ &  $\checkmark$ &  $\checkmark$  &   $\cross$ (*) &  $\checkmark$ \\
          Assumption~\ref{asm: derivative of f} & $\checkmark$ & $\checkmark$ & $\checkmark$ & $\checkmark$ &$\checkmark$ \\
          Assumption~\ref{asm: I is increasing} &  $\checkmark$ &  $\checkmark$ &  $\checkmark$  & $\cross$ (*) & $\checkmark$ \\
          \hline
    \end{tabular}
    \vspace{-0.3em}
    \label{tab: asm check}
\end{table}

Eq.~(\ref{asm: block constant}) in Assumption~\ref{asm: bounded block} is the term that directly appears in the regret bound.
As described in Proposition~\ref{def: FMDA}, $\max_{i\in[k]}X_i/a_k$ converges weakly to \Fr distribution with shape $\alpha$, which satisfies $\E_{X\sim \mathcal{F}_{\alpha}}[X] = \Gamma\qty(1-\frac{1}{\alpha})$, $\E_{X\sim \mathcal{F}_{\alpha}}[1/X]=\Gamma\qty(1+\frac{1}{\alpha})$ and $a_k \approx k^{\frac{1}{\alpha}}$.
Therefore, (\ref{asm: block upper}) and (\ref{asm: block constant}) roughly require that it also converges in the sense of expectation and expectation of the inverse.
The assumption of $a_k=\Theta(k^{\frac{1}{\alpha}})$ does not hold in general, but it holds if we ignore the sub-polynomial factor.
As a result, if we remove this assumption the bound becomes sub-polynomially worse in terms of $K$.
An easy-to-verify sufficient condition for Assumption~\ref{asm: bounded block} is
\begin{align*}
    \limsup_{x\to \infty} S_F(x) = \limsup_{x\to \infty} x^{\alpha} (1-F(x)) &< \infty \\
    \liminf_{x \to \infty} S_F(x) = \liminf_{x\to \infty} x^{\alpha} (1-F(x)) &> 0, \vspace{-0.2em}\numberthis{\label{suff_tail}} 
\end{align*}
while (\ref{suff_tail}) becomes the necessary condition for $a_k^{-1} =\gO(k^{-\frac{1}{\alpha}})$ if we replace $\liminf$ with $\limsup$.
Note that both $F$ and $F^*$ in (\ref{eq: conditioning trick}) for all distributions in Table~\ref{tab: fr} satisfy 
(\ref{suff_tail})
 with explicit forms of $m$ and $A_l$
as shown in Appendix~\ref{app: table check} and Lemma~\ref{lem: bound of I}.

Assumptions~\ref{asm: derivative of f} and~\ref{asm: I is increasing} may appear somewhat restrictive, but many Fr\'{e}chet-type distributions, including several well-known examples such as $\gF_\alpha$ and $\gP_\alpha$, satisfy this condition, as shown in Table~\ref{tab: asm check}.
Assumption~\ref{asm: derivative of f} is a condition slightly stronger than von Mises condition, because $\frac{-xf'(x)}{f(x)}\to\alpha+1$ implies (\ref{eq: vMC}) by L'h\^{o}pital's rule.
We expect that Assumption~\ref{asm: I is increasing} can be relaxed to the monotonicity of $f(x)/F(x)$ in $x>z_1$ for some $z_1\ge\nu$ as in Assumption~\ref{asm: decreasing f}, which is satisfied in all examples in Table~\ref{tab: asm check}.
Still, this relaxation makes the case-analysis somewhat too long and is left as a future work.

In the rest of this paper, we always assume that the distribution satisfies $\nu\ge 1$ rather than $\nu\ge 0$ for notational simplicity except for the specific analysis for \Fr and Pareto distributions, where the density functions are written in simple forms.
This is without loss of generality because the shifted distribution $G(x)=F(x-1)$ has the left-endpoint $\nu+1\ge 1$ and clearly satisfies Assumptions~\ref{asm: decreasing f}--\ref{asm: I is increasing}, while the arm-selection probability is the same between $F(x)$ and $G(x)$.

\section{Main result}
In this section, we present our main theoretical results that show the optimality of FTPL with perturbation distribution $\gD_\alpha \in \fD_\alpha$ in adversarial bandits.
Furthermore, we provide regret upper bounds of FTPL with perturbations under a mild additional condition on $\gD_\alpha$ in stochastic bandits. \vspace{-0.2em}
\begin{theorem}\label{thm: adv_all}
    In the adversarial bandits, there exist some constants $C_1(\gD_\alpha, c)$, $C_2(\gD_\alpha)$ and $C_3(\gD_\alpha, c, K)$ such that FTPL with $\gD_\alpha \in \fD_\alpha$ and learning rates $\eta_t = \frac{c}{\sqrt{t}}K^{\frac{1}{\alpha}-\frac{1}{2}}$ for $c >0$ and $\alpha>1$ satisfies \vspace{-0.1em}
    \begin{equation*}
        \gR(T) \leq  C_1(\gD_\alpha,c) \sqrt{KT}+ C_2(\gD_\alpha) \log(T+1) + \frac{M A_u \sqrt{K}}{c} .
    \end{equation*} 
    \vspace{-1em}
\end{theorem}
This result shows the minimax optimality of FTPL with the Fr\'{e}chet-type distributions including \Fr distributions and generalized Pareto distributions, which not only generalizes the results of \citet{pmlr-v201-honda23a} but also resolves the open question in \citet{NEURIPS2019_aff0a6a4} in the sense that we provide conditions for a very large class of Fr\'{e}chet-type perturbations.

Here, our result requires that $\alpha > 1$ holds.
This is because (\ref{asm: block upper}) in Assumption~\ref{asm: bounded block}
does not hold for $\alpha\le 1$ since the extreme distribution of $\gD_{\alpha}$ (that is, $\mathcal{F}_{\alpha}$) has infinite mean.
This corresponds to the assumption of the finite expected block maxima
$\E_{X_1,\dots,X_k\sim \mathcal{D}}[\max_i X_i]<\infty$ considered in~\citet{abernethy2015fighting} and \citet{NEURIPS2019_aff0a6a4}.


The following result shows that FTPL with $\fD_2$ can achieve the logarithmic regret in the stochastic bandits.
Note that all Fr\'{e}chet-type tail distributions in Table~\ref{tab: fr} belong to $\fD_\alpha$.
\vspace{-0.3em}
\begin{theorem}\label{thm: sto_all}
    Assume that $i^*  = \argmin_{i \in [K]} \mu_i$ is unique and let $\Delta_i = \mu_i - \mu_i^*$.
    Then, FTPL with learning rate $\eta_t = \frac{c}{\sqrt{t}}$ for $c>0$ and $\gD \in \fD_2$ satisfies \vspace{-0.57em}
    \begin{equation*}
        \gR(T) \leq \mathcal{O}\qty(\sum_{i \ne i^*}\frac{\log T}{\Delta_i}).
    \end{equation*}
    \vspace{-0.9em}
\end{theorem}
This result shows that FTPL achieves BOBW if the limiting distribution of the perturbation under mild conditions is Fr\'{e}chet distribution with shape $\alpha=2$.
It can be interpreted as a counterpart of FTRL with Tsallis entropy regularization, where the logarithmic regret is known only for $1/2$-Tsallis entropy without any knowledge of the gaps~\citep[see][Remarks 5 and 6]{zimmert2021tsallis}, while Tsallis entropy with any parameter achieves the optimal adversarial regret.

Although there is no stochastic perturbation that yields the same arm-selection probability as Tsallis entropy regularizer for $K\geq 4$, in two-armed setting, it has been shown that $\beta$-Tsallis entropy regularizer can be reduced to a Fr\'{e}chet-type perturbation with index $\alpha=\frac{1}{1-\beta}$ satisfying von Mises condition~\citep[Appendix C.2]{NEURIPS2019_aff0a6a4}.
Therefore, the success of $\alpha=2$ perturbation seems intuitive since it roughly corresponds to $1/2$-Tsallis entropy regularizer.
In addition, $\beta$-Tsallis entropy becomes the log-barrier for $\beta \to 0$~\citep{zimmert2021tsallis}, which corresponds to $\alpha\to 1$.
The BOBW achievability of log-barrier regularization without adaptive learning rate has not been known, which seems to correspond to our requirement of $\alpha>1$.

Beyond the case $\alpha =2$, we obtain the following results.
\begin{theorem}\label{thm: sto_bad}
    Assume that $i^*  = \argmin_{i \in [K]} \mu_i$ is unique and let $\Delta_i = \mu_i - \mu_i^*$.
    Then, FTPL with learning rate $\eta_t = \frac{c}{\sqrt{t}}K^{\frac{1}{\alpha}-\frac{1}{2}}$ for $c>0$ and $\gD_\alpha \in \fD_\alpha$ for $\alpha> 2$ satisfies
    \begin{equation*}
        \gR(T) \leq \mathcal{O}\qty(\sum_{i\ne i^*} \frac{1}{\alpha-2} \frac{T^{\frac{\alpha-2}{2(\alpha-1)}}}{\Delta_i^{\frac{1}{\alpha-1}}K^{\frac{\alpha-2}{2(\alpha-1)}}}).  
    \end{equation*}
    If $\alpha \in (1,2)$, then
    \begin{equation*}
        \gR(T) \leq  \mathcal{O}\qty(\sum_{i\ne i^*} \frac{1}{2-\alpha}\frac{T^{1-\frac{\alpha}{2}}}{\Delta_i^{\alpha-1}K^{1-\frac{\alpha}{2}}} ).
    \end{equation*}
\end{theorem}
Although our regret upper bound for FTPL with index $\alpha \ne 2$ does not match the regret lower bound for the stochastic case, this result shows that the regret of FTPL has better dependence on $T$ in the stochastic case than $\mathcal{O}(\sqrt{T})$ in the adversarial case because $\frac{\alpha-2}{2(\alpha-1)}<\frac{1}{2}$ for $\alpha>2$ and $1-\frac{\alpha}{2} < \frac{1}{2}$ for $\alpha \in (1,2)$.

We expect that FTPL with $\alpha \ne 2$ can attain (poly-)logarithmic regret in the stochastic setting by using arm-dependent learning rate as \citet{jin2023improved} showed the BOBW results for FTRL with $\beta$-Tsallis entropy regularization for $\beta \in (0,1)$.
However, the results of \citet{jin2023improved} in the adversarial setting are $\mathcal{O}(\sqrt{KT\log T})$ when $\beta \ne 1/2$, which does not achieve the adversarial optimality in the strict sense.
It is highly nontrivial whether FTPL with $\alpha \ne 2$ can achieve both logarithmic regret in the stochastic case and $\mathcal{O}(\sqrt{KT})$ regret in the adversarial case.
\vspace{-0.3em}
\section{Proof Outline}
\vspace{-0.1em}
In this section, we first provide a proof outline of Theorem~\ref{thm: adv_all} and then sketch the proof of Theorems~\ref{thm: sto_all} and~\ref{thm: sto_bad}, whose detailed proofs are given in Appendices~\ref{app: stability rslt},~\ref{app: penalty} and~\ref{app: stoc all}.

While our analysis draws inspiration from the structure in \citet{pmlr-v201-honda23a}, a naive application of their analysis does not yield a bound for the general case.
This is mainly because, while the use of Fr\'{e}chet distribution in \citet{NEURIPS2019_aff0a6a4} and \citet{pmlr-v201-honda23a} is inspired by the extreme value theory, their actual analysis is not based on this theory.
Instead, it is highly specific to the Fr\'{e}chet distribution with shape $\alpha=2$.
Consequently, the representations of Fr\'{e}chet-type distributions in extreme value theory are not directly associated with their analysis.
To address this challenge, we demonstrate that the general representation in (\ref{eq: tail and S}) under von Mises condition can be specifically tailored for the regret analysis.
\vspace{-0.3em}
\subsection{Regret decomposition}
To evaluate the regret of FTPL, we first decompose regret into three terms, which generalizes Lemma~3 of \citet{pmlr-v201-honda23a}.
The proofs of lemmas in this section are given in Appendix~\ref{app: regret decom}.
\begin{lemma}\label{lem: gen_lem3}
For any $\alpha > 1$ and $\gD_\alpha \in \fD_\alpha$,
\begin{equation}\label{eq: stab and penalty}
    \mathrm{Reg}(T)  \leq \sum_{t=1}  \E \qty[\inp{\hat{\ell}_t}{w_t - w_{t+1}}]
    + \sum_{t=1}^T \qty(\frac{1}{\eta_{t+1}} - \frac{1}{\eta_t}) \E \qty[r_{t+1, I_{t+1}}- r_{t+1, i^*}]  + \frac{M A_u \sqrt{K}}{c} .
\end{equation}
\end{lemma}
The proof of this lemma is essentially the same as that of \citet{pmlr-v201-honda23a}, except that we need to evaluate the block maxima $\E_{X_i\sim\gD_\alpha}[\max_{i\in [K]}X_i]$ for general $\gD_\alpha \in \fD_\alpha$.
Following the convention in the analysis of BOBW policies~\citep{zimmert2021tsallis, ito22a, pmlr-v201-honda23a}, we refer to the first and second terms of (\ref{eq: stab and penalty}) as \emph{stability term} and \emph{penalty term}, respectively.

Here, we can further decompose the stability term into two terms as follows.
\begin{lemma}\label{lem: stab decom}
    For any $\alpha > 1$ and $\gD_\alpha \in \fD_\alpha$,
    \begin{equation}\label{eq: stab left}
        \sum_{t=1}  \E \qty[\inp{\hat{\ell}_t}{w_t - w_{t+1}}] \leq 2C_2(\gD_\alpha) \log(\frac{\eta_1}{\eta_{T+1}})
        +\sum_{t=1}^T \E \qty[\inp{\hat{\ell}_t}{\phi(\eta_t \hat{L}_t)-\phi(\eta_t (\hat{L}_t+\hat{\ell}_t))}], 
    \end{equation}
    where $\phi=(\phi_1, \ldots, \phi_K)$ for $\phi_{i}$ defined in (\ref{eq: def_phi}),
    \begin{equation*}
         C_2(\gF_\alpha) = \frac{\alpha}{2}, \qq{and} C_2(\gD_\alpha) \leq \frac{\rho_1(e^2+1)}{2}, \quad \gD_\alpha \in \fD_\alpha.  
    \end{equation*}
    Here, $\rho_1 = \rho_1(\gD_\alpha)$ is a positive distribution-dependent constant satisfying
    \begin{equation}\label{def: bounded rho}
         \frac{xf(x)}{1-F(x)} \leq \rho_1.
    \end{equation}
\end{lemma}
Note that Assumption~\ref{asm: bounded hazard} under von Mises condition implies the existence of $\rho_1$ in (\ref{def: bounded rho}).
From this result, it remains to derive upper bounds of the second term of (\ref{eq: stab left}) and the penalty term to conclude the proof of Theorem~\ref{thm: adv_all}. 

\subsection{Stability term}
The analysis of the arm-selection probability $\phi$ has been recognized as the central and most challenging aspect of the regret analysis for FTPL~\citep{abernethy2015fighting,pmlr-v201-honda23a}.
The key to the analysis of the stability for general Fr\'{e}chet-type distribution is another representation called Karamata's representation, which is an essential tool to express the slowly varying functions.
In the analysis, we interchangeably use this representation along with the representation in (\ref{eq: tail and S}) and von Mises condition in (\ref{eq: vMC}), which utilizes a coherent connection between general representations and those under von Mises conditions.
See Appendices~\ref{app: Karamata} and~\ref{app: stability rslt} for details of Karamata's representation and the proofs, respectively.

For the arm selection probability function $\phi_i(\lambda)$ in (\ref{eq: def_phi}), define for any $\alpha>0$, $\phi_i'(\lambda;\gD_\alpha) = \pdv{\phi_i}{\lambda_i} \qty(\lambda; \gD_\alpha)$ and 
\begin{align}
    I_{i,n}(\lambda ; \alpha) &= \int_0^\infty \frac{1}{(z+\lambda_i)^n}\exp(-\sum_{j\in[K]} \frac{1}{(z+\lambda_j)^\alpha}) \dd z,\label{eq: def_phi_I} \\
    J_{i}(\lambda; \gD_\alpha) &= \int_{1}^\infty \frac{f(z+\lambda_i)}{(z+ \lambda_i)}  \prod_{j \ne i} F(z+\lambda_j) \dd z.\label{eq: def_phi_J}
\end{align}
We will employ $I_{i,n}$ and $J_i$ to analyze the stability term for $\gF_\alpha$ and $\fD_\alpha \setminus \{\gF_\alpha \}$, respectively. 
Although the analysis for $J_i$ can cover $\gF_\alpha$, we consider the specific form of $\gF_\alpha$ in $I_i$ without any truncation or shift to derive a tighter upper bound.

Note that $\phi_i'(\lambda)\leq 0$ holds since it denotes the probability of $\lambda_i - r_i < \min_{i\ne j} \qty{\lambda_i - r_j}$ when each $r_i$ is generated from $\gD_\alpha$.
By the same reason, $\phi_i(\lambda)$ is non-decreasing with respect to $\lambda_j$ for $i\ne j$.
To derive an upper bound of the stability term, we provide lemmas that are related to the relation between the arm-selection probability and its derivatives, which plays a central role in the regret analysis of FTPL.

\begin{lemma}\label{lem: general monotonicity of I}
For any $\alpha >0$ and $\gD_\alpha \in \fD_\alpha$, $\frac{I_{i,\alpha+2}(\lambda;\alpha)}{I_{i,\alpha+1}(\lambda;\alpha)}$ and $\frac{J_i(\lambda; \gD_\alpha)}{\phi_i(\lambda; \gD_\alpha)}$ are monotonically increasing with respect to $\lambda_j$ for any $j\ne i$.
\end{lemma}
Assumption~\ref{asm: I is increasing} plays a key role in simplifying the proof of this lemma. 
Still, we conjecture that it can be weakened to the monotonicity of $\frac{f(x)}{F(x)}$ in $x\ge z_2$ for some $z_2>0$ rather than the current assumption requiring $z_2=\nu$.
This is because the role of Lemma~\ref{lem: general monotonicity of I} is to control the behavior of the algorithm when the perturbation becomes large.

Based on this result, the following lemma holds.
\begin{lemma}\label{lem: bound of I}
    If $\lambda_i$ is the $\sigma_i$-th smallest among $\lambda_1, \ldots, \lambda_K$ (ties are broken arbitrarily), then
    \begin{equation*}
        \frac{I_{i,\alpha+2}(\ul; \alpha)}{I_{i,\alpha+1}(\ul; \alpha)} \leq \frac{\alpha}{(\alpha+1)\ul_i} \land \frac{\Gamma\qty(1+\frac{1}{\alpha})}{\sqrt[\alpha]{\sigma_i}}
    \end{equation*}
    and
     \begin{equation*}
        \frac{J_{i}(\ul; \gD_\alpha)}{\phi_{i}(\ul; \gD_\alpha)} \leq 
             \frac{m}{A_l} \sigma_i^{-\frac{1}{\alpha}} \land \frac{\alpha}{\alpha+1}\frac{eA_u}{A_l \ul_i}
    \end{equation*}
    where $m$, $A_l$, and $A_u$ are given in Assumption~\ref{asm: bounded block}.
    Moreover, if $\gD_\alpha$ satisfies
    \begin{equation}\label{asm: bounded rho as alpha}
        \frac{xf(x)}{1-F(x)} \leq \alpha,
    \end{equation}
    then, $m \leq 2\Gamma\qty(1+\frac{1}{\alpha})$, $A_l = 1$, and $A_u = \lim_{x\to \infty} S_F^{1/\alpha}(x)$ holds
\end{lemma}
Note that all distributions in Table~\ref{tab: fr} satisfy (\ref{asm: bounded rho as alpha}) as shown in Appendix~\ref{app: table check}.
Similarly to (\ref{def: bounded rho}), from Assumption~\ref{asm: derivative of f}, there exists some constants $\rho_2 >0$ satisfying
\begin{equation}\label{eq: bounded f' with rho}
    \frac{-xf'(x)}{f(x)} \leq \rho_2.
\end{equation}
Then, by Lemma~\ref{lem: bound of I}, we obtain the following lemma. 
\begin{lemma}\label{lem: stability_adv_i}
    For any $i\in [K]$, if $\hat{L}_{t,i}$ is the $\sigma_{t,i}$-th smallest among $\{\hat{L}_{t,j}\}_j$, then for $\alpha>1$ and $\gD_\alpha \in \fD_\alpha$
    \begin{equation}\label{eq: stab i}
        \E \qty[\hat{\ell}_{t,i}\qty(\phi_i\qty(\eta_t \hat{L}_t; \gD_\alpha) - \phi_i\qty(\eta_t \qty(\hat{L}_t + \hat{\ell}_t); \gD_\alpha) ) \eval \hat{L}_t] \leq \psi_s(\hat{\uL}_{t,i}; \gD_\alpha) \land 2\eta_t\frac{\rho_2 mA_u}{A_l\sqrt[\alpha]{\sigma_i}},
    \end{equation}
    where $\rho_2 = \alpha+1$ holds for $\gF_\alpha$ and $\gP_\alpha$, $m(\gF_\alpha)= \Gamma\qty(1+\frac{1}{\alpha})$, and
    \begin{align*}
        \psi_s(\hat{\uL}_{t,i}; \gD_\alpha) = \begin{cases}
            \frac{2\alpha}{\hat{\uL}_{t,i}}  & \text{if } \gD_\alpha = \gF_\alpha, \\
            \frac{2\rho_2 \alpha}{\alpha+1} \frac{eA_u}{A_l\hat{\uL}_{t,i}} & \text{if } \gD_\alpha \in \fD_{\alpha} \setminus \{\gF_\alpha\}.
        \end{cases}
    \end{align*}
\end{lemma}
The second term of RHS of (\ref{eq: stab i}) finally leads to the bound on the stability term, which is used for both the adversarial and stochastic bandits. 
For the stochastic bandits, we use the tighter bound with $\psi_s$ to apply the self-bounding technique.
\begin{lemma}\label{lem: stability}
    For any $\hat{L}_t$ and $\alpha>1$ and $\gD_\alpha \in \fD_\alpha$,
    \begin{equation*}
        \E \qty[\inp{\hat{\ell}_{t}}{\phi\qty(\eta_t \hat{L}_t; \gD_\alpha) - \phi\qty(\eta_t \qty(\hat{L}_t + \hat{\ell}_t); \gD_\alpha) } \eval \hat{L}_t] \leq
        2 \frac{\alpha \rho_2}{\alpha-1}  \frac{mA_u}{A_l} K^{1-\frac{1}{\alpha}} \eta_t.
    \end{equation*}
\end{lemma}
\subsection{Penalty term}
Next, we establish an upper bound for the penalty term.
\begin{lemma}\label{lem: penalty}
    For any $\alpha>1$ and $\gD_\alpha \in \fD_\alpha$, 
    \begin{equation}\label{eq: pen bnd}
         \E \qty[r_{t,I_t} - r_{t,i^*} \eval \hat{L}_t ] \leq \psi_p(\hat{\uL}_{t,i},\gD_\alpha) \land C_{1,1}(\gD_\alpha)\sqrt[\alpha]{K},
    \end{equation}
    where $C_{1,1}(\gD_\alpha)$ is a distribution-dependent constant, which satisfies $C_{1,1}(\gF_\alpha) = C_{1,1}(\gP_\alpha)/e$ for
    \begin{equation*}
         C_{1,1}(\gP_\alpha) = \frac{2\alpha^3  + (e-2)\alpha^2}{(\alpha-1)(2\alpha-1)},
    \end{equation*}
    and
    \begin{align*}
    \psi_p(\hat{\uL}_{t,i}; \gD_\alpha) = \begin{cases}
        \sum_{i\ne i^*} \frac{1}{(\eta_t \hat{\uL}_{t,i})^{\alpha-1}}  & \text{if } \gD_\alpha = \gF_\alpha, \\
        \frac{e\rho_1 A_u^{\alpha}}{\alpha-1} \sum_{i\ne i^*} \frac{1}{(\eta_t \hat{\uL}_{t,i})^{\alpha-1}} & \text{if } \gD_\alpha \in \fD_{\alpha} \setminus \{\gF_\alpha\}.
    \end{cases}
\end{align*}
\end{lemma}
The expression of $C_{1,1}(\gD_\alpha)$ for general $\gD_\alpha$ is given in the proof of this lemma in Appendix~\ref{app: penalty gen}, which is expressed in terms of $A_u$.
For the adversarial bandits, we only utilize the bound with $K^{1/\alpha}$ in (\ref{eq: pen bnd}), which induces $\mathcal{O}(\sqrt{KT})$ regret by using learning rate $\eta_{T} = \mathcal{O}(K^{\frac{1}{\alpha}-\frac{1}{2}} T^{-\frac{1}{2}})$.
Similarly to the stability term, we use $\psi_p$ to apply the self-bounding technique for the stochastic bandits.
\subsection{Proof of Theorem~\ref{thm: adv_all}}
By combining Lemmas~\ref{lem: gen_lem3},~\ref{lem: stab decom},~\ref{lem: stability} and~\ref{lem: penalty} with $\eta_t = \frac{c}{\sqrt{t}}K^{\frac{1}{\alpha}-\frac{1}{2}}$, we have
\begin{align*}
    \gR(T) &\leq \frac{2\alpha\rho_2 m A_u c \sqrt{K}}{A_l(\alpha-1)}\sum_{t=1}^T \frac{1}{\sqrt{t}} + \frac{C_{1,1}(\gD_\alpha)\sqrt{K}}{c} \sum_{t=1}^T \qty(\sqrt{t+1}-\sqrt{t}) \\
    &\hspace{14em}+ 2C_2(\gD_\alpha) \log(\sqrt{T+1}) + \frac{M A_u \sqrt{K}}{c} \\
    & \leq \qty( \frac{4\alpha\rho_2 m A_u c \sqrt{K}}{A_l(\alpha-1)}+ \frac{C_{1,1}(\gD_\alpha)}{c}) \sqrt{KT} \\
    &\hspace{14em} + C_2(\gD_\alpha)\log(T+1) +\frac{M A_u \sqrt{K}}{c} ,
\end{align*}
where letting $C_{1}(\gD_\alpha, c) =\frac{4\alpha\rho_2 m A_u c \sqrt{K}}{A_l(\alpha-1)} + \frac{C_{1,1}(\gD_\alpha)}{c}$ concludes the proof.

\subsection{Proof sketch of Theorems~\ref{thm: sto_all} and~\ref{thm: sto_bad}}
Since the overall proof for $\alpha \geq 2$ and $\alpha \in (1,2)$ are very similar, we provide a sketch for the case $\alpha \geq 2$.
Let us begin by restating the regret in stochastic bandits, which is
\begin{equation*}
    \gR(T) = \E \qty[\sum_{t=1} \sum_{i \ne i^*} \Delta_i w_{t,i}].
\end{equation*} 
To apply the proof techniques in \citet{pmlr-v201-honda23a}, we define an event $D_t$ based on the tail quantile function where $\hat{L}_{t,i}$ is sufficiently large compared to that of the optimal arm so that $\hat{\uL}_{t,i^*} = 0$.

In Appendix~\ref{app: regret optimal}, we show that the stability term corresponding to the optimal arm is bounded by $\mathcal{O}\qty(\sum_{i\ne i^*} 1/ \hat{\uL}_{t,i})$ on $D_t$, which provides for $\alpha \geq 2$
\begin{equation*}
    \gR(T) \leq \E\qty[ \sum_{t=1}^T \mathcal{O}\qty(\I[D_t] \sum_{i\ne i^*} \frac{1}{\hat{\uL}_{t,i}}) + \I[D_t^c]\sqrt{K/t}].
\end{equation*}
To apply the self-bounding technique, we obtain
\begin{equation*}
    \gR(T) \geq \E\qty[\sum_{t=1}^T \mathcal{O}\qty(\I[D_t] \sum_{i\ne i^*} \frac{t^{\frac{\alpha}{2}}\Delta_i}{K^{1-\frac{\alpha}{2}} \hat{\uL}_{t,i}^{\alpha}}+ \I[D_t^c]\Delta)],
\end{equation*}
where $\Delta = \min_{i\ne i^*} \Delta_i$ and the proof is given in Appendix~\ref{app: regret lb}.
By combining these results, we have
\begin{align*}
    \frac{\gR(T)}{2} \leq \E\qty[\sum_{t=1}^T \mathcal{O}\qty(\I[D_t] \sum_{i\ne i^*} \qty(\frac{1}{\hat{\uL}_{t,i}} - \frac{t^{\frac{\alpha}{2}}\Delta_i}{2K^{1-\frac{\alpha}{2}}\hat{\uL}_{t,i}^{\alpha}}))]+ \E\qty[\sum_{t=1}^T \mathcal{O}\qty(\I[D_t^c](\sqrt{K/t}-\Delta/2))].
\end{align*}
Since $Ax - Bx^{\alpha} \leq A \frac{\alpha-1}{\alpha} \qty(\frac{A}{\alpha B})^{\frac{1}{\alpha-1}}$ holds for $A, B>0$ and $\alpha >1$, we obtain
\begin{equation*}
    \gR(T) \leq \sum_{t=1}^T \mathcal{O}\qty(\sum_{i \ne i^*} \frac{K^{\frac{2-\alpha}{2(\alpha-1)}}}{\Delta_i^{\frac{1}{\alpha-1}}t^{\frac{\alpha}{2(\alpha-1)}}}) + \mathcal{O}(K),
\end{equation*}
which concludes the proof.
Note that the dependency on $K$ in the leading term stems from the choice of learning rate.

\section{Conclusion}
In this paper, we considered FTPL policy with perturbations belonging to FMDA in the adversarial and stochastic settings.
We provided a sufficient condition for perturbation distributions to achieve optimality, which solves the open problem by \citet{NEURIPS2019_aff0a6a4} in a comprehensive direction.
Furthermore, we provide the stochastic regret bound for FTPL, where Fr\'{e}chet-type distributions with mild assumptions can achieve BOBW.
While our analysis for FTPL with index $\alpha \ne 2$ does not attain logarithmic stochastic regrets, these findings align with observations in FTRL policies, offering insights that might help understand the effect of regularization of FTRL through the lens of FTPL.

\bibliographystyle{plainnat}
\bibliography{ref}

\newpage
\appendix
\crefalias{section}{appendix} 

\section{Details on extreme value theory}\label{app: EVT}
When contemplating the asymptotic properties of sample statistics, the sample means and central limit theorem often comes to mind, elucidating the behavior of partial sums of samples. 
Conversely, interest might shift towards extremes, focusing on maxima or minima of samples, particularly when singular rare events pose challenges, such as substantial insurance claims arising from catastrophic events like earthquakes and tsunamis.
Extreme value theory is the field of studying the behavior of maxima of random variables, especially the behavior of the distribution function in the tail.
One of the fundamental results in extreme value theory is the Fisher–Tippett–Gnedenko theorem, which provides a general result regarding the asymptotic distribution of normalized extreme order statistics of i.i.d. sequence of random variables.
\begin{proposition}[Fisher–Tippett–Gnedenko theorem~\citep{fisher1928limiting, gnedenko1943distribution}]
    Let $M_n = \vee_{i=1}^n X_i$ where $\qty{X_i}_{i=1}^n$ be an i.i.d. sequence of random variables with common distribution function $F(x)$.
    Suppose there exist $a_n >0$, $b_n \in \sR$, $n\geq 1$ such that
    \begin{equation*}
        \Pr[(M_n -b_n)/a_n \leq x] = F^{n}(a_n x + b_n) \to G(x),
    \end{equation*}
    weakly as $n \to \infty$ where $G$ is assumed nondegenerate.
    Then, $G$ is of the type of one of the following three classes:
    \begin{enumerate}[label=(\roman*)]
        \item (Fr\'{e}chet-type) $\Phi_\alpha(x) = \begin{cases}
                0, & x <0, \\
                \exp(-x^{-\alpha}), & x\geq 0,
            \end{cases}\quad $ for some $\alpha >0$.
        \item (Weibull-type) $\Psi_\alpha(x) = \begin{cases}
            \exp(-(-x)^\alpha), & x<0, \\
            1 , & x\geq 0,
        \end{cases}\quad$ for some $\alpha >0$.
        \item (Gumbel-type) $\Lambda(x) = \exp(-e^{-x})$ for $x\in \sR$. 
    \end{enumerate}
\end{proposition}
Among these three types of extreme value distributions, we are interested in Fr\'{e}chet-type distributions, where the equivalence was established in Proposition~\ref{def: FMDA}, which states
\begin{equation*}
    F^n(a_nx) \to \Phi_\alpha(x)
\end{equation*}
with $a_n = \inf\qty{x: F(x) \geq 1-\frac{1}{n}}.$

However, verifying whether a distribution belongs to a domain of attraction can often be challenging.
Therefore, a convenient sufficient condition, known as the von Mises condition, is often considered~\citep{von1936distribution, beirlant2006statistics}, which is
\begin{equation*}
    \lim_{x\to \infty} \frac{xf(x)}{1-F(x)} =\alpha.
\end{equation*}
It is worth noting that $\fD_{\alpha}^{\text{all}}$ consists of distributions satisfying von Mises condition and their tail-equivalent distributions~\citep{embrechts1997modelling}.

\paragraph{Existence of the density}
Here, it is known that if $g \in \RV_\alpha$, for $\alpha \ne 0$, then there exists $g^*$ that is absolutely continuous, strictly monotone, and $g(x) \sim g^*(x)$ as $x\to \infty$, i.e., tail-equivalent~\citep[Propostition 0.8.]{resnick2008extreme}.
Therefore, Assumption~\ref{asm: decreasing f} implies that we fix our interest solely on distribution with a continuous density among their tail-equivalent distributions.

\paragraph{Tail quantile function}
When $1-F \in \RV_{-\alpha}$, its tail quantile function $U$ is regularly varying with index $\frac{1}{\alpha}$, i.e., $U \in \RV_{1/\alpha}$, where $U(t) = \inf \qty{x: F(x) \geq 1-1/t}$ on $[1, \infty)$~\citep{beirlant2006statistics}.
Therefore, one can directly obtain that $a_n = n^{\frac{1}{\alpha}}S_U(n)$, where $S_U$ denotes the corresponding slowly varying function.
Here, it is known that $S_U$ is the de Bruijn conjugate (or de Bruyn in some literature) of $S_F^{-1/\alpha}$, which satisfies $S_U(x)S_F^{-1/\alpha}(xS_U(x)) \sim S_F^{-1/\alpha}(x)S_U(xS_F^{-1/\alpha}(x)) \to 1$.
This implies that if $S_F$ is upper-bounded by some constants, then $S_U$ is also upper-bounded regardless of $K$. 
For more details, we refer readers to \citet{charras2013extreme}, which provides a concise introduction to the extreme value theory.

\paragraph{Karamata's theorem}
Since all tail distributions in FMDA are regularly varying, the following results are useful to represent the regularly varying functions.
\begin{proposition}[Karamata's theorem~{\citep[Theorem B.1.5]{haan2006extreme}}]\label{prop: Karamata theorem}
    Suppose $f \in \RV_{\alpha}$.
    There exists $t_0 >0$ such that $g(t)$ is positive and locally bounded for $t \geq t_0$.
    If $\alpha \geq -1$, then
    \begin{equation*}
        \lim_{t \to \infty} \frac{tg(t)}{\int_{t_0}^t g(s) \dd s} = \alpha +1.
    \end{equation*}
    If $\alpha <-1$ and $\int_0^\infty g(s) \dd s <\infty$, then
    \begin{equation}\label{eq: Karamata theorem}
        \lim_{t \to \infty} \frac{tg(t)}{\int_t^{\infty} g(s) \dd s} = -\alpha -1.
    \end{equation}
    Conversely, if (\ref{eq: Karamata theorem}) holds with $\alpha \in (-\infty, -1)$, then $g \in \RV_{\alpha}$.
\end{proposition}
Therefore, one can see that von Mises condition and the existence of density imply $f \in \RV_{-\alpha -1}$.
Furthermore, from (\ref{eq: Karamata theorem}), Assumption~\ref{asm: derivative of f} is equivalent to $-f' \in \RV_{-\alpha-2}$ and boundedness of $-f'(x)/f(x)$.

\subsection{Karamata's representation}\label{app: Karamata}
From (\ref{eq: tail and S}), one can specify a distribution in FMDA with index $\alpha$ and the slowly varying function $S_F(x)$.
Here, several representations of slowly varying functions can be considered~\citep{galambos1973regularly}, and we follow Karamata's representation described in \citet{resnick2008extreme}, which is
\begin{equation}\label{def: Karamata representation}
    S_F(x) = c(x) \exp( \int_{1}^x \frac{\varepsilon_F(t)}{t} \dd t), \quad x \geq 1
\end{equation}
where $c(x)$ and $\varepsilon(x)$ are bounded functions such that $\lim_{x \to \infty} c(x) = c >0$ and $\lim_{x \to \infty} \varepsilon_F(x) = 0$.
Here, the representation is not unique and it depends on the choice of $c(x)$, $\varepsilon_F(x)$, and the interval of the integral.
For example, $c(x)$ and $\varepsilon_F(x)$ can be written as~\citep[Corollary of Theorem 0.6.]{resnick2008extreme}
\begin{align*}
    c(x) &=  \frac{xS_F(x)}{\int_0^x S_F(t) \dd t} \int_0^1 S_F(t) \dd t,\\
    \varepsilon_F(x) &= \frac{xS_F(x)}{\int_0^x S_F(t) \dd t} - 1.
\end{align*}
One can check that $\lim_{x\to \infty}\varepsilon_F(x) \to 0$ from Proposition~\ref{prop: Karamata theorem} with $\alpha =0$.

On the other hand, when $F$ is absolutely continuous, we can rewrite the tail distribution as for $x \geq 1$
\begin{equation}\label{eq: 1-F and exp}
    1-F(x) = \exp(\log(1-F(x)))= \exp(\int_1^x \frac{-f(t)}{1-F(t)} \dd t).
\end{equation}
Since $1-F(x) = x^{-\alpha}S_F(x)$ holds for $x\geq 1$, it holds that
\begin{align*}
    S_F(x) = x^{\alpha}(1-F(x)) &= x^\alpha \exp(\int_1^x \frac{-f(t)}{1-F(t)} \dd t) \tag*{by (\ref{eq: 1-F and exp})} \\
    &= \exp( \alpha \log x - \int_1^x \frac{f(t)}{1-F(t)} \dd t) \\
    &= \exp( \int_1^x \frac{\alpha}{t} \dd t - \int_1^x \frac{f(t)}{1-F(t)} \dd t).
\end{align*}
By letting $\varrho(t) = \frac{tf(t)}{1-F(t)}$, we obtain
\begin{align*}
    S_F(x) &= \exp( \int_1^x \frac{\alpha}{t} \dd t - \int_1^x \frac{\varrho(t)}{t} \dd t) \\
    &= \exp( \int_1^x \frac{\alpha - \varrho(t)}{t} \dd t). \numberthis{\label{def: SF representation}}
\end{align*}
Here, from the definition of $\varrho$, von Mises condition can be written as $\varrho(t) \to \alpha,$ as $ t \to \infty$, which satisfies $\lim_{t \to \infty} \alpha- \varrho(t) =0$ and thus indicates the existence of the upper bound of $S_F$.
In this paper, we use the representation of $S_F$ in (\ref{def: SF representation}), where $c(x)$ is given as the ultimate constant.
Therefore, when $F$ satisfies (\ref{asm: bounded rho as alpha}), one can see that $S_F$ is monotonically increasing for $x\geq 1$.
Note that $\varepsilon_F(t)$ in (\ref{def: Karamata representation}) are not necessarily the same as $\varrho(t) -\alpha$ unless $c(x) = 1-F(1)$.

The von Mises condition (\ref{eq: vMC}) with Assumption~\ref{asm: decreasing f} implies $f \in \RV_{-1-\alpha}$ from Proposition~\ref{prop: Karamata theorem}~\citep[see][Proposition A3.8]{embrechts1997modelling}, i.e., $f = x^{-\alpha+1} S_f(x)$.
Therefore, from $1-F(x) = x^{-\alpha} S_F(x)$ with (\ref{def: SF representation}), we have
\begin{align*}
    f(x)& = \frac{S_F(x)\alpha}{x^{\alpha+1}} - \frac{S_F'(x)}{x^{\alpha}} 
    = \frac{S_F(x)}{x^{\alpha+1}}\varrho(x),
\end{align*}
which implies
\begin{equation}\label{eq: Sf and SF}
    S_f(x) = S_F(x) \varrho(x).
\end{equation}
One can check that $S_f \in \RV_{0}$ since $\lim_{t \to \infty}\varrho(x) = \alpha$ holds by von Mises condition and $S_F \in \RV_0$.

\subsection{Proofs for Table~\ref{tab: asm check}}\label{app: table check}
It is straightforward to check whether the Fr\'{e}chet, Pareto, and Generalized Pareto satisfy the assumptions.
Therefore, we showed that the Student-$t$ distribution satisfies Assumption~\ref{asm: bounded block} but not Assumption~\ref{asm: I is increasing} and the Snedecor's $F$ distribution satisfies all Assumptions.
In this section, we prove Assumption~\ref{asm: bounded block} by showing $\frac{xf(x)}{1-F(x)} \leq \alpha$, which implies $S_F(x)$ is increasing so that satisfying the sufficient condition (\ref{suff_tail}).
\subsubsection{Student-$t$}
Since it is easy to verify Assumptions~\ref{asm: decreasing f},~\ref{asm: bounded hazard} and~\ref{asm: derivative of f}, we focus on Assumptions~\ref{asm: bounded block} and~\ref{asm: I is increasing}
\paragraph{Assumption~\ref{asm: bounded block}}
Here, we show that (\ref{asm: bounded rho as alpha}) holds.
Since $\frac{xf(x)}{1-F(x)} \leq 0$ is obvious for $x \leq 0$, let us consider the case $x >0$.
In this case, $1-F(x) = \frac{1}{2}\frac{B\qty(\frac{n}{x^2 + n};\frac{n}{2}, \frac{1}{2})}{B\qty(\frac{n}{2}, \frac{1}{2})}$ holds.
Therefore,
\begin{align*}
    \frac{xf(x)}{1-F(x)} &= \frac{x\qty(\frac{n+x^2}{n})^{-\frac{n+1}{2}}}{\sqrt{n}B\qty(\frac{n}{2}, \frac{1}{2})} \qty(\frac{1}{2}\frac{B\qty(\frac{n}{x^2 + n};\frac{n}{2}, \frac{1}{2})}{B\qty(\frac{n}{2}, \frac{1}{2})})^{-1} \\
    &= \frac{2}{\sqrt{n}}\frac{x\qty(\frac{x^2+n}{n})^{-\frac{n+1}{2}}}{B\qty(\frac{n}{x^2 + n};\frac{n}{2}, \frac{1}{2})} \\
    &= \frac{2}{\sqrt{n}} \frac{x\qty(\frac{x^2+n}{n})^{-\frac{n+1}{2}}}{\frac{2}{n}\qty(\frac{n}{x^2 +n})^{\frac{n}{2}}\qty(\frac{x^2}{x^2+n})^{\frac{1}{2}}\prescript{}{2}F_1(\frac{n+1}{2}, 1; \frac{n+2}{2}; \frac{n}{x^2+n})} \numberthis{\label{eq: hyper t}} \\
    &= \frac{n}{\prescript{}{2}F_1(\frac{n+1}{2}, 1; \frac{n+2}{2}; \frac{n}{x^2+n})},
\end{align*}
In (\ref{eq: hyper t}), we used the results in \citep[8.17.8]{olver2010nist} that provide the relationship between the incomplete Beta function and the (Gaussian) hypergeometric function $\prescript{}{2}F_1$, which is 
\begin{equation}\label{eq: olver beta hyper}
    B(x; a,b) = \frac{x^a (1-x)^{b}}{a} \prescript{}{2}F_1(a+b, 1; a+1; x).
\end{equation}
Here, the hypergeometric function is defined by the Gauss series, which is defined for $|x| < 1$ and $c>0$ by
\begin{equation*}
    \prescript{}{2}F_1(a, b; c; x) = \sum_{s=0}^\infty \frac{(a)_s(b)_s}{(c)_s s!} x^s = 1 + \frac{ab}{c}z + \cdots,
\end{equation*}
where $(a)_n$ denotes the rising factorial, i.e., $(a)_n = a(a+1)\cdots(a+n-1)$ and $(a)_0 =1$.
Therefore, we have for $x \geq 0$
\begin{equation*}
    \frac{xf(x)}{1-F(x)}  \leq n,
\end{equation*}
which verifies that $\gT_n$ satisfies Assumption~\ref{asm: bounded block} by (\ref{asm: bounded rho as alpha}).
Here, one can see that the hazard function $\frac{f(x)}{1-F(x)}$ diverges as $x\to 0$, while $\frac{f^*(x)}{1-F^*(x)} \leq n$ holds.

\paragraph{Assumption~\ref{asm: I is increasing}}
Since the density of $\gT_n$ is symmetric, it holds for any $t \geq 0$ that 
\begin{equation*}
    f(t) = f(-t), \qquad F(t) = 1-F(-t).
\end{equation*}
Then, we have 
\begin{equation*}
    \frac{f(-t)}{F(-t)} = \frac{f(t)}{1-F(t)} \geq \frac{f(t)}{F(t)},
\end{equation*}
where the inequality follows from $F(t) \geq \frac{1}{2}$ for $t \geq 0$.
Therefore, $\gT_n$ does not satisfy Assumption~\ref{asm: I is increasing}.
However, when one considers only for $t \geq 0$, $f(t)$ is decreasing while $F(t)$ is increasing, which implies that $f/F$ is decreasing for $x \geq 0$.
This implies that both the half-$t$ distribution, $|\gT_n|$ and truncated one in (\ref{eq: conditioning trick}) satisfy Assumption~\ref{asm: I is increasing}.

\subsubsection{$F$ distribution}
Since it is easy to verify Assumptions~\ref{asm: decreasing f},~\ref{asm: bounded hazard} and~\ref{asm: derivative of f}, we focus on Assumptions~\ref{asm: bounded block} and~\ref{asm: I is increasing}
\paragraph{Assumption~\ref{asm: bounded block}}
Here, we show that (\ref{asm: bounded rho as alpha}) holds.
Let $I(x; a,b) = \frac{B(x;a,b)}{B(a,b)}$ denote the regularized incomplete beta function.
From the definition of the incomplete beta function, one can see that $I(x;a,b) = 1-I(1-x; b,a)$ holds.
Then, it holds that
\begin{align*}
    \frac{xf(x)}{1-F(x)} = \frac{\qty(\frac{m}{n})^{\frac{m}{2}} x^{\frac{m}{2}} \qty(\frac{mx+n}{n})^{-\frac{m+n}{2}}}{B\qty(\frac{m}{2}, \frac{n}{2})I\qty(\frac{n}{mx+n};\frac{n}{2}, \frac{m}{2})}.
\end{align*}
Since $B(a,b) = B(b,a)$, we obtain
\begin{align*}
    \frac{\qty(\frac{m}{n})^{\frac{m}{2}} x^{\frac{m}{2}} \qty(\frac{mx+n}{n})^{-\frac{m+n}{2}}}{B\qty(\frac{m}{2}, \frac{n}{2})I\qty(\frac{n}{mx+n};\frac{n}{2}, \frac{m}{2})} &= \frac{\qty(\frac{m}{n})^{\frac{m}{2}} x^{\frac{m}{2}} \qty(\frac{mx+n}{n})^{-\frac{m+n}{2}}}{B\qty(\frac{n}{mx+n};\frac{n}{2}, \frac{m}{2})} \\
    &= \frac{\qty(\frac{m}{n})^{\frac{m}{2}} x^{\frac{m}{2}} \qty(\frac{mx+n}{n})^{-\frac{m+n}{2}}}{\frac{2}{n} \qty(\frac{n}{mx+n})^{\frac{n}{2}} \qty(\frac{mx}{mx+n})^{\frac{m}{2}} \prescript{}{2}F_1(\frac{m+n}{2}, 1; \frac{n}{2}+1; \frac{n}{mx+n})} \tag*{by (\ref{eq: olver beta hyper})} \\
    &= \frac{n}{2} \frac{1}{\prescript{}{2}F_1(\frac{m+n}{2}, 1; \frac{n}{2}+1; \frac{n}{mx+n})} \leq \frac{n}{2},
\end{align*}
which verifies Assumption~\ref{asm: bounded block} by (\ref{asm: bounded rho as alpha}).
Here, one can observe that the hazard function $\frac{f(x)}{1-F(x)}$ diverges as $x\to 0$, while $\frac{f^*(x)}{1-F^*(x)} \leq n$ holds.

\paragraph{Assumption~\ref{asm: I is increasing}}
If $f/F$ is monotonically decreasing, it should hold that for any $x \geq y > 0$
\begin{equation*}
    \frac{F(y)}{F(x)} \leq \frac{f(y)}{f(x)}.
\end{equation*}
Here, it holds that
\begin{align*}
    \frac{F(y)}{F(x)} = \frac{B\qty(\frac{my}{my+n}; \frac{m}{2}, \frac{n}{2})}{B\qty(\frac{mx}{my+n}; \frac{m}{2}, \frac{n}{2})} &= \frac{\qty(\frac{my}{my+n})^{\frac{m}{2}} \qty(\frac{n}{my+n})^{\frac{n}{2}}}{\qty(\frac{mx}{mx+n})^{\frac{m}{2}} \qty(\frac{n}{mx+n})^{\frac{n}{2}}} \frac{\prescript{}{2}F_1\qty(\frac{m+n}{2},1;1+\frac{m}{2}; \frac{my}{my+n})}{\prescript{}{2}F_1\qty(\frac{m+n}{2},1;1+\frac{m}{2}; \frac{mx}{mx+n})} \tag*{by (\ref{eq: olver beta hyper})} \\
    &= \qty(\frac{y}{x})^{\frac{m}{2}} \qty(\frac{mx+n}{my+n})^{\frac{m+n}{2}} \frac{\prescript{}{2}F_1\qty(\frac{m+n}{2},1;1+\frac{m}{2}; \frac{my}{my+n})}{\prescript{}{2}F_1\qty(\frac{m+n}{2},1;1+\frac{m}{2}; \frac{mx}{mx+n})},
\end{align*}
Therefore, we have for any $x \geq y > 0$
\begin{equation*}
    \frac{F(y)}{F(x)} \leq \qty(\frac{y}{x})^{\frac{m}{2}} \qty(\frac{mx+n}{my+n})^{\frac{m+n}{2}}.
\end{equation*}
since $\frac{mx}{mx+n}$ is increasing with respect to $x>0$.
We have for $x \geq y >0$
\begin{align*}
    \frac{f(y)}{f(x)} = \qty(\frac{y}{x})^{\frac{m}{2}-1}\qty(\frac{mx+n}{my+n})^{\frac{m+n}{2}} &=  \qty(\frac{y}{x})^{\frac{m}{2}}\qty(\frac{mx+n}{my+n})^{\frac{m+n}{2}} \frac{x}{y} \\
    &\geq \qty(\frac{y}{x})^{\frac{m}{2}}\qty(\frac{mx+n}{my+n})^{\frac{m+n}{2}} \geq \frac{F(y)}{F(x)},
\end{align*}
which verifies Assumption~\ref{asm: I is increasing}.

\section{Proofs for regret decomposition}\label{app: regret decom}
Here, we provide the proofs for Lemmas~\ref{lem: gen_lem3} and~\ref{lem: stab decom}.
\subsection{Proof of Lemma~\ref{lem: gen_lem3}}
Firstly, we present the regret decomposition that can be applied to general distributions.
\begin{lemma}[Lemma 3 of \citet{pmlr-v201-honda23a}]\label{lem: hnd_3}
    \begin{equation*}
    \gR(T)  \leq \sum_{t=1}  \E \qty[\inp{\hat{\ell}_t}{w_t - w_{t+1}}] 
    + \sum_{t=1}^T \qty(\frac{1}{\eta_{t+1}} - \frac{1}{\eta_t}) \E_{r_{t+1}\sim \gD} \qty[r_{t+1, I_{t+1}}- r_{t+1, i^*}] + \frac{1}{\eta_1} \E_{r_1 \sim \gD}[r_{1,I_1}],
    \end{equation*}
    where $r_{1,I_1} = \max_{i \in [K]} r_{1,i}$.
\end{lemma}
Here, notice that $\E_{r_1 \sim \gD}[r_{1,I_1}]$ is the expected block maxima when $K$ samples are given.
For the \Fr distributions and Pareto distributions, we can explicitly compute the upper bound $\E[M_K]$ as follows.
\begin{lemma}\label{lem: block maxima}
    For $\alpha >1$,
    \begin{equation*}
         \E_{r_{1,1}, \ldots, r_{1,K} \sim \gD_\alpha}[r_{1,I_1}] \leq \begin{cases}
         M A_u K^{\frac{1}{\alpha}} & \text{if } \gD_\alpha \in \fD_\alpha \\
          K^{\frac{1}{\alpha}}\Gamma\qty(1-\frac{1}{\alpha}) & \text{if } \gD_\alpha = \gF_\alpha, \\
           K^{\frac{1}{\alpha}} \Gamma\qty(1-\frac{1}{\alpha}) \frac{\alpha}{\alpha-1}  & \text{if } \gD_\alpha = \gP_\alpha.
        \end{cases}
    \end{equation*}
\end{lemma}
\begin{proof}
    The proof for the general $\fD_{\alpha}$ can be directly obtained by (\ref{asm: block upper}) in Assumption~\ref{asm: bounded block}.
    As explained in Appendix~\ref{app: EVT}, the tail quantile function $U$ is regularly varying with index $\frac{1}{\alpha}$, which implies 
    \begin{equation*}
        a_K = K^{\frac{1}{\alpha}} S_U(K)
    \end{equation*}
    for some $S_U \in \RV_0$.
    Thus, Assumption~\ref{asm: bounded block} implies the boundedness of $S_U$.
    Here, from the definition of $a_K$, it holds that
    \begin{equation*}
        1-F(a_K) = \frac{1}{K} = \frac{S_F(a_K)}{(a_K)^{\alpha}}, 
    \end{equation*}
    which implies
    \begin{equation}\label{eq: SF and SU}
        a_K = K^{\frac{1}{\alpha}} S_F^{\frac{1}{\alpha}}(a_K).
    \end{equation}
    Therefore, $S_U(K) = S_F^{\frac{1}{\alpha}}(a_K)$ holds.
    The upper-bounded assumption (the existence of $A_u$) is not restrictive from Karamata's representation with von Mises condition in (\ref{def: SF representation}), where $\frac{\alpha- \varrho(t)}{t} \to 0$ as $t\to \infty$ and $c(x)$ is given as ultimate constants. 
    \paragraph{Case 1. \Fr distribution}
    It is well-known that when $X_i \sim \gF(\alpha, s, m)$ where $(\alpha, s, m) \in \sR_{+} \times \sR_{+} \times \sR $ denotes the shape, scale, and location of the \Fr distribution, then $Y = \max\qty(X_1, \ldots, X_n)$ follows $\gF\qty(\alpha, n^{1/\alpha}, m)$.
    One can easily check by observing its CDF is given by $e^{-K/x^{\alpha}}$ or the max-stability of \Fr distributions.
    The fact that the expected value of $\gF(\alpha, s, m) = m+s \Gamma\qty(1-\frac{1}{\alpha})$ for $\alpha >1$ and $\gF_\alpha = \gF(\alpha, 1, 0)$ completes the proof.

    \paragraph{Case 2. Pareto distribution}
    Since $r_{1,I_1} = \max_{i\in[K]}r_{1,i}$, its CDF is $\qty(1-z^{-\alpha})^K$ with density $\frac{\alpha K}{z^{\alpha+1}}(1-z^{-\alpha})^{K-1}$.
    By letting $w=z^{-\alpha}$,
    \begin{align*}
        \E_{r \sim \gF_\alpha}[r_{1, I_1}] &= \int_1^\infty \frac{\alpha K}{z^{\alpha}}(1-z^{-\alpha})^{K-1} \dd z \\
        &= K \int_0^1 w^{-\frac{1}{\alpha}} (1-w)^{K-1} \dd w \\
        &= K B\qty(1-\frac{1}{\alpha}, K) = K\frac{\Gamma\qty(1-\frac{1}{\alpha})\Gamma(K)}{\Gamma\qty(K+1-\frac{1}{\alpha})}, \numberthis{\label{eq: Par lem3}}
    \end{align*}
    where $B(z_1, z_2) := \int_0^1 w^{z_1-1} (1-w)^{z_2-1} \dd w$ denotes the Beta function.
    Then, by applying Lemma~\ref{lem: Gautschi_lemma}, Gautschi's inequality, we obtain for $\alpha>1$
    \begin{align*}
         \frac{\Gamma\qty(1-\frac{1}{\alpha})\Gamma(K+1)}{\Gamma\qty(K+1-\frac{1}{\alpha})} &= \frac{K}{K-\frac{1}{\alpha}} \frac{\Gamma(K)}{\Gamma\qty(K-\frac{1}{\alpha})} \\
         &\leq \frac{K}{K-\frac{1}{\alpha}} \Gamma\qty(1-\frac{1}{\alpha}) K^{\frac{1}{\alpha}} \\
         &\leq \frac{\alpha}{\alpha-1}  \Gamma\qty(1-\frac{1}{\alpha}) K^{\frac{1}{\alpha}} ,
    \end{align*}
    where the last inequality follows from $K\geq 1$.
    Here, one can directly apply Gautschi's inequality in (\ref{eq: Par lem3}), which results in $\Gamma\qty(1-\frac{1}{\alpha}) (K+1)^{\frac{1}{\alpha}}$.
\end{proof}

\subsection{Proof of Lemma~\ref{lem: stab decom}}
 From the definition of $w_{t} = \phi(\eta_t \hat{L}_t; \gD_\alpha)$, we have
    \begin{align*}
        w_t - w_{t+1} &= \phi(\eta_t \hat{L}_t) - \phi(\eta_{t+1} \hat{L}_{t+1}) \\
        &= \phi(\eta_t \hat{L}_t) - \phi(\eta_{t+1} (\hat{L}_{t}+\hat{\ell}_t)) \\
        &= \phi(\eta_t \hat{L}_t) - \phi(\eta_{t} (\hat{L}_{t}+\hat{\ell}_t)) + \phi(\eta_{t} (\hat{L}_{t}+\hat{\ell}_t)) - \phi(\eta_{t+1} (\hat{L}_{t}+\hat{\ell}_t)),
    \end{align*}
    which implies
    \begin{align*}
        \sum_{t=1}  \E \qty[\inp{\hat{\ell}_t}{w_t - w_{t+1}}] &\leq  \sum_{t=1}^T \E \qty[\inp{\hat{\ell}_t}{\phi(\eta_t \hat{L}_t)-\phi(\eta_t (\hat{L}_t)+\hat{\ell}_t)}] \\
        & \qquad + \sum_{t=1}^T \E \qty[\inp{\hat{\ell}_t}{\phi(\eta_{t} (\hat{L}_{t}+\hat{\ell}_t)) - \phi(\eta_{t+1} (\hat{L}_{t}+\hat{\ell}_t))}]. \numberthis{\label{eq: L4_decomp}}
    \end{align*}
    Therefore, it remains to bound the second term of (\ref{eq: L4_decomp}).

    \paragraph{Case 1. \Fr distribution}
    By explicitly substituting the density function and CDF of $\gF_\alpha$, $\phi_i(\lambda; \gF_\alpha)$ is expressed by
    \begin{align*}
         \phi_{i}(\lambda; \gF_\alpha) := \Pr_{r \sim \gF_\alpha}\qty[i = \argmin_{j \in [K]} \qty{\lambda_j - r_j}] &= \int_{-\min_{j\in [K]} \lambda_j}^{\infty}
     \frac{\alpha}{(z+\lambda_i)^{\alpha+1}} \exp(-\sum_{l\in [K]}\frac{1}{(z+\lambda_l)^\alpha}) \dd z \\
     &= \int_{0}^{\infty}
     \frac{\alpha}{(z+\ul_i)^{\alpha+1}} \exp(-\sum_{l\in [K]}\frac{1}{(z+\ul_l)^\alpha}) \dd z,
     \end{align*}
     Then, for generic $L \in \sR^K$, $\uL = L - \1 \min_i L_i$, and any $i \in [K]$
    \begin{align*}
        &\pdv{\eta}\phi_i(\eta L; \gF_\alpha) \\
        &= \alpha \int_0^\infty \qty[\frac{1}{(z+\eta \uL_i)^{\alpha+1}} \sum_{j\in [K]} \frac{\alpha \uL_j}{(z+\eta \uL_j)^{\alpha+1}} - \frac{(\alpha+1)\uL_i}{(z+\eta \uL_i)^{\alpha+2}} ]\exp(-\sum_{j\in [K]} \frac{1}{(z+\eta \uL_j)^{\alpha}}) \dd z\\
        &\leq \alpha \int_0^\infty \frac{1}{(z+\eta \uL_i)^{\alpha+1}} \exp(-\sum_{j\in [K]} \frac{1}{(z+\eta \uL_j)^{\alpha}}) \sum_{j\in [K]} \frac{\alpha \uL_j}{(z+\eta \uL_j)^{\alpha+1}}  \dd z\\
        &\leq \alpha \int_0^\infty \frac{1}{(z+\eta \uL_i)^{\alpha+1}} \exp(-\sum_{j\in [K]} \frac{1}{(z+\eta \uL_j)^{\alpha}})  \max_{l \in [K]}\frac{\alpha \uL_l}{(z+\eta \uL_l)} \sum_{j\in [K]} \frac{1}{(z+\eta \uL_j)^{\alpha}} \dd z \\
        &\leq \alpha \int_0^\infty \frac{1}{(z+\eta \uL_i)^{\alpha+1}} \exp(-\sum_{j\in [K]} \frac{1}{(z+\eta \uL_j)^{\alpha}}) \frac{\alpha}{\eta} \sum_{j\in [K]} \frac{1}{(z+\eta \uL_j)^{\alpha}} \dd z.
    \end{align*}
    Let $L = \hat{L}_{t}+\hat{\ell}_t$. Since $\hat{\ell}_t = l_{t}\widehat{w^{-1}_t} e_{I_t}$ and $l_{t,i} \in [0,1]$,
    \begin{align*}
        \sum_{t=1}^T \E \qty[ \inp{\hat{\ell}_t}{\phi(\eta_{t} (\hat{L}_{t}+\hat{\ell}_t); \gF_\alpha) - \phi(\eta_{t+1} (\hat{L}_{t}+\hat{\ell}_t)); \gF_\alpha}]  
        \hspace{-22em}&  \numberthis{\label{eq: case 1 to case 2}}
        \\ &= \sum_{t=1}^T \sum_{i\in[K]} \E \qty[ \I[I_t = i] l_{t,i}\widehat{w^{-1}_{t,i}} (\phi_i (\eta_t L; \gF_\alpha) - \phi_i(\eta_{t+1}L; \gF_\alpha))] \\
        &= \sum_{t=1}^T \E \qty[  \int_{\eta_{t+1}}^{\eta_t}  \sum_{i\in[K]} l_{t,i} \pdv{\eta}\phi_i(\eta L; \gF_\alpha) \dd \eta] \\
        &\leq \alpha \sum_{t=1}^T \E \qty[ \int_{\eta_{t+1}}^{\eta_t} \frac{1}{\eta} \int_0^\infty \sum_{i\in[K]} l_{t,i} \frac{\alpha}{(z+\eta \uL_i)^{\alpha+1}} \exp(-\sum_{j\in [K]} \frac{1}{(z+\eta \uL_j)^{\alpha}}) \sum_{j\in [K]} \frac{1}{(z+\eta \uL_j)^{\alpha}} \dd z \dd \eta] \\
        &\leq \alpha \sum_{t=1}^T \E \qty[\int_{\eta_{t+1}}^{\eta_t} \frac{1}{\eta} \int_0^{\infty} w e^{-w} \dd w \dd \eta] \\
        &= \alpha \sum_{t=1}^T \log\qty(\frac{\eta_t}{\eta_{t+1}}) = \alpha \log\qty(\frac{\eta_1}{\eta_{T+1}}).
    \end{align*}

    \paragraph{Case 2. Distributions in $\fD_\alpha$}
    From the definition of $\phi$ in (\ref{eq: def_phi}), for generic $L \in \sR^K$, $\uL = L - \1 \min_i L_i$, and any $i \in [K]$
    \begin{multline}
        \pdv{\eta}\phi_i(\eta L) = \int_1^\infty \uL_i f' (z+\eta \uL_i) \prod_{j \ne i} F(z+\eta \uL_j) \dd z  \\+ \int_1^\infty  \sum_{j \ne i} \qty(  \uL_j f(z+\eta \uL_i) f(z+\eta \uL_j)  \prod_{l \ne i,j} F(z+\eta \ul_l)) \dd z. \label{eq: two terms in gen lem4}
    \end{multline}
    Recall the definition of $\varrho(x) = \frac{xf(x)}{1-F(x)}$, which implies
    \begin{align}\label{eq: F and dF with rho}
      f(x) = \frac{\varrho(x)}{x}(1-F(x)).
    \end{align}
    Then, the first term of (\ref{eq: two terms in gen lem4}) can be bounded by
    \begin{align*}
        \int_1^\infty & \uL_i  f'_\alpha (z+\eta \uL_i) \prod_{j \ne i} F(z+\eta \uL_j) \dd z
        \\ &\leq \int_1^{z_0} \uL_i f'_\alpha (z+\eta \uL_i) \prod_{j \ne i} F(z+\eta \uL_j) \dd z \tag*{(by Assumption~\ref{asm: decreasing f})}\\
        &= \uL_i f(z+\eta \uL_i) \prod_{j \ne i} F(z+\eta \uL_j) \eval_{z=1}^{z=z_0}- \int_1^{z_0} \uL_i f(z+\eta \uL_i) \sum_{j\ne i} f(z+\eta \uL_j) \prod_{l \ne i,j} F(z + \eta \uL_l) \dd z  \\
        &\leq \uL_i f(z_0+\eta \uL_i) \prod_{j \ne i} F(z_0+\eta \uL_j) \\
        &\leq \frac{\uL_i\varrho(z_0+\eta \uL_i)}{z_0+\eta \uL_i} (1-F(z_0 + \eta \uL_i)) \leq \frac{\rho_1}{\eta}, \tag*{(by (\ref{def: bounded rho}) and (\ref{eq: F and dF with rho}))}
    \end{align*}
    Next, for the second term of (\ref{eq: two terms in gen lem4}), by representation in (\ref{def: SF representation}), we obtain
    \begin{align*}
         \prod_{l \ne i,j} F(z+\eta \ul_l)) &=  \prod_{l \ne i, j} \qty(1- \frac{S_F(z+\eta \uL_l)}{(z+\eta \uL_l)^\alpha}) \\
         & \leq  \exp ( - \sum_{l \ne i, j} \frac{S_F(z+\eta \uL_l)}{(z+\eta \uL_l)^\alpha}) \qquad (\because 1-x \leq e^{-x}, \forall x\geq0) \\
         & \leq e^2 \exp ( - \sum_{j \in [K]} (1-F(z+\eta \uL_j))),
    \end{align*}
    where the last inequality follows from $F(x) \in [0,1]$ for all $x \in [1,\infty)$, i.e., $e^{1-F(x)} \leq e$ for any $x \in [1, \infty)$.
    
    Then, we have 
    \begin{align*}
        \int_1^\infty &\sum_{j \ne i} \qty(  \uL_j f(z+\eta \uL_i) f(z+\eta \uL_j) \prod_{l \ne i,j} F(z+\eta \ul_l)) \dd z \\
        &\leq e^2 \int_1^\infty f(z+\eta \uL_i) \sum_{j \in [K]} \qty(  \uL_j  f(z+\eta \uL_j))  \exp ( - \sum_{l \in [K]} (1-F(z+\eta \uL_j)))\dd z \\
        &= e^2 \int_1^\infty f(z+\eta \uL_i) \qty(\sum_{j \in [K]}  \uL_j  f(z+\eta \uL_j)) \exp ( - \sum_{j\in[K]} (1-F(z+\eta \uL_j))) \dd z.
    \end{align*}
    Here, by (\ref{eq: F and dF with rho}) again, for generic $L \in \sR^K$, we obtain for $z \in [1, \infty)$
    \begin{align*}
        \sum_{j \in [K]}  \uL_j  f(z+\eta \uL_j))  &= \sum_{j \in [K]}  \frac{\uL_j \varrho(z+\eta \uL_j)}{z+\eta \uL_j} (1-F(z+\eta \uL_j)) \\
        &\leq  \sum_{j \in [K]}  \frac{\uL_j \rho_1}{z+\eta \uL_j} (1-F(z+\eta \uL_j)) \leq \sum_{j \in [K]}  \frac{\rho_1}{\eta} (1-F(z+\eta \uL_j)),
    \end{align*}
    which implies
    \begin{align*}
        \pdv{\eta}\phi_i(\eta L) \leq \frac{\rho_1 e^2}{\eta} \int_1^\infty f(z+\eta \uL_i) \qty(\sum_{j \in [K]} (1-F(z+\eta \uL_j))) \exp ( - \sum_{j\in[K]} (1-F(z+\eta \uL_j))) \dd z + \frac{\rho_1}{\eta}.
    \end{align*}
    By noticing that 
    \begin{equation*}
        \sum_{i \in [K]} -f(z+\eta \uL_i) = \dv{z} \sum_{j \in [K]} (1-F(z+\eta \uL_j)),
    \end{equation*}
    one can reproduce the proof in Case 1 from (\ref{eq: case 1 to case 2}), which implies
    \begin{align*}
        \sum_{t=1}^T \E &\qty[ \inp{\hat{\ell}_t}{\phi(\eta_{t} (\hat{L}_{t}+\hat{\ell}_t); \gD_\alpha) - \phi(\eta_{t+1} (\hat{L}_{t}+\hat{\ell}_t); \gD_\alpha)}]  
         \\
        &\leq  \rho_1 \sum_{t=1}^T \E \Bigg[ l_{t,i} \int_{\eta_{t+1}}^{\eta_t} \frac{e^2}{\eta}  \Bigg\{\int_1^\infty \sum_{i\in[k]}f(z+\eta \uL_i) \\
        & \hspace{6em} \cdot \qty(\sum_{j \in [K]} (1-F(z+\eta \uL_j))) \exp ( - \sum_{j\in[K]} (1-F(z+\eta \uL_j))) \dd z \Bigg\} + \frac{1}{\eta} \dd \eta  \Bigg] \\
        &\leq \rho_1 \sum_{t=1}^T \E \qty[\int_{\eta_{t+1}}^{\eta_t} \frac{1}{\eta} \qty{ e^2 \int_0^{K} w e^{-w} \dd w +1 }\dd \eta] \\
        &\leq \rho_1 \qty(e^2 +1) \sum_{t=1}^T \log\qty(\frac{\eta_t}{\eta_{t+1}}) = \rho_1 \qty(e^2 +1) \log\qty(\frac{\eta_1}{\eta_{T+1}}). 
    \end{align*}

\section{Regret bound for adversarial bandits: Stability}\label{app: stability rslt}
Here, we provide the proofs for Lemmas~\ref{lem: general monotonicity of I}--\ref{lem: stability}.

\subsection{Proof of Lemma~\ref{lem: general monotonicity of I}: monotonicity}
Let us consider the \Fr distributions first.
\subsubsection{\Fr distribution}
From the definitions of $\phi$ and $I$,
\begin{equation}\label{eq: phi_I_relation}
    \phi_i(\lambda; \gF_\alpha) = \alpha I_{i,\alpha+1}(\ul; \alpha), \qquad \phi_i'(\lambda; \gF_\alpha) = - \alpha(\alpha+1)I_{i,\alpha+2}(\ul; \alpha) + \alpha^2 I_{i, 2(\alpha+1)}(\ul; \alpha).
\end{equation}
Here, $\phi_{i}'(\lambda; \gD_\alpha) \leq 0$ holds for any $\alpha >0$ as it denotes the probability of $\{\lambda_i - r_i < \min_{i\ne j} \qty{\lambda_j - r_j}\}$ when each $r_i$ follows $\gD_\alpha$. 

 Define
    \begin{equation*}
        I_{i,j,n}(\lambda; \alpha) = \int_0^\infty \frac{1}{(z+\lambda_i)^n}  \frac{1}{(z+\lambda_j)^{\alpha+1}}  \exp(-\sum_j \frac{1}{(z+\lambda_j)^\alpha}) \dd z. 
    \end{equation*}
    For simplicity, we write $I_{i,j,n}(\lambda; \alpha) = I_{i,j, n}(\lambda)$ and $I_{i,n}(\lambda; \alpha) = I_{i,n}(\lambda)$ when $n$ is written with $\alpha$.
    Then,
    \begin{equation}\label{eq: monotonicity_Iijn}
        \dv{\lambda_j} \frac{I_{i,\alpha+2}(\lambda)}{I_{i,\alpha+1}(\lambda)} = \alpha \frac{I_{i,j,\alpha+2}(\lambda)I_{i,\alpha+1}(\lambda) -  I_{i,j,\alpha+1}(\lambda)I_{i,\alpha+2}(\lambda) }{I_{i,\alpha+1}^2(\lambda)}.
    \end{equation}
    By letting $k(z) = \frac{1}{(z+\lambda_i)^{\alpha+1}}\exp(-\sum_j \frac{1}{(z+\lambda_j)^\alpha})$, each term of the numerator of (\ref{eq: monotonicity_Iijn}) is written as
    \begin{align*}
        I_{i,j,\alpha+2}(\lambda)I_{i,\alpha+1}(\lambda) &= \iint_{z,w\geq 0} \frac{k(z)k(w)}{(z+\lambda_i)(z+\lambda_j)^{\alpha+1}} \dd z \dd w \\
        &= \frac{1}{2}\iint_{z,w\geq 0} k(z)k(w) \qty( \frac{1}{(z+\lambda_i)(z+\lambda_j)^{\alpha+1}} + \frac{1}{(w+\lambda_i)(w+\lambda_j)^{\alpha+1}} )\dd z \dd w, \\
         I_{i,j,\alpha+1}(\lambda)I_{i,\alpha+2}(\lambda) &= \iint_{z,w\geq 0} \frac{k(z)k(w)}{(z+\lambda_i)(w+\lambda_j)^{\alpha+1}} \dd z \dd w \\
         &= \frac{1}{2} \iint_{z,w\geq 0} k(z)k(w) \qty( \frac{1}{(z+\lambda_i)(w+\lambda_j)^{\alpha+1}} + \frac{1}{(w+\lambda_i)(z+\lambda_j)^{\alpha+1}} )\dd z \dd w.
    \end{align*}
    Then, the integrand for $I_{i,j,\alpha+2}(\lambda)I_{i,\alpha+1}(\lambda) -  I_{i,j,\alpha+1}(\lambda)I_{i,\alpha+2}(\lambda)$ is expressed as
    \begin{align*}
        &\frac{1}{(z+\lambda_i)(z+\lambda_j)^{\alpha+1}} + \frac{1}{(w+\lambda_i)(w+\lambda_j)^{\alpha+1}} - \frac{1}{(z+\lambda_i)(w+\lambda_j)^{\alpha+1}} - \frac{1}{(w+\lambda_i)(z+\lambda_j)^{\alpha+1}} \\
        &\hspace{1em}= \frac{(w+\lambda_i)(w+\lambda_j)^{\alpha+1} + (z+\lambda_i)(z+\lambda_j)^{\alpha+1} - (w+\lambda_i)(z+\lambda_j)^{\alpha+1} - (z+\lambda_i)(w+\lambda_j)^{\alpha+1} }{(z+\lambda_i)(z+\lambda_j)^{\alpha+1} (w+\lambda_i)(w+\lambda_j)^{\alpha+1}} \\
        &\hspace{1em} = (w-z) \frac{(w+\lambda_j)^{\alpha+1} - (z+\lambda_j)^{\alpha+1} }{(z+\lambda_i)(z+\lambda_j)^{\alpha+1} (w+\lambda_i)(w+\lambda_j)^{\alpha+1}}.
    \end{align*}
    Here, one can see that when $w\geq z$, the integrand is non-negative since $\lambda_j >0 $ and $\alpha>0$.
    On the other hand, if $w < z$, then both $(w-z)$ and $(w+\lambda_j)^{\alpha+1} - (z+\lambda_j)^{\alpha+1}$ becomes negative, i.e., integrand is again positive.
    Therefore, $I_{i,j,\alpha+2}(\lambda)I_{i,\alpha+1}(\lambda) -  I_{i,j,\alpha+1}(\lambda)I_{i,\alpha+2}(\lambda)$ is an integral of a positive function, which concludes the proof.

\subsubsection{Fr\'{e}chet-type distributions}
As discussed in Appendix~\ref{app: Karamata}, when $F$ is absolute continuous and satisfies von Mises condition, $f \in \RV_{-\alpha-1}$, which implies $f(x) = x^{-\alpha-1}S_f(x)$ for some $S_f \in \RV_0$.
Let $g_i(z) = \frac{S_f(z+\lambda_i)}{(z+\lambda_i)^{\alpha+2}}$.
Then, for $\gD_\alpha \in \fD_\alpha$, we can rewrite $J_i$ as 
\begin{equation*}
    J_i(\lambda; \gD_\alpha) = \int_1^\infty \frac{S_f(z+\lambda_i)}{(z+\lambda_i)^{\alpha+2}} \prod_{j\ne i} F(z+\lambda_j) \dd z.
\end{equation*}
    For simplicity, let $f_i(z) = f(z+ \lambda_i)$ and $F_i(z)=F(z+\lambda_i)$ for any $i\in [K]$, which denotes the density function and CDF of $\gD_\alpha$, respectively.
    From the definition of $\phi$ in (\ref{eq: def_phi}) and $J_i$ in (\ref{eq: def_phi_J}), we have
    \begin{align*}
        \dv{\lambda_j} J_i(\lambda ; \gD_\alpha) &= \int_1^{\infty} \frac{S_f(z+\lambda_i)}{(z+\lambda_i)^{\alpha+2}}   \dv{\lambda_j}\prod_{j \ne i} F(z+\lambda_j) \dd z \\
        &=\int_1^{\infty}  \frac{S_f(z+\lambda_i)}{(z+\lambda_i)^{\alpha+2}} f(z+\lambda_j)   \prod_{l \ne i, j} F(z+\lambda_l) \dd z
    \end{align*}
    and
    \begin{align*}
        \dv{\lambda_j} \phi_i(\lambda; \gD_\alpha) &= \int_1^\infty f(z+\lambda_i) \dv{\lambda_j}\prod_{j \ne i} F(z+\lambda_j) \dd z \\
        &= \int_1^\infty f(z+\lambda_i) f(z+\lambda_j) \prod_{l \ne i,j} F(z+\lambda_l) \dd z.
    \end{align*}
    Then, we have for $k(z) = \prod_{l \ne i,j} F_l(z)$
    \begin{align*}
        \dv{\lambda_j} \frac{J_i(\lambda; \gD_\alpha)}{\phi_i(\lambda; \gD_\alpha)} &= \frac{1}{\phi_i^2(\lambda; \gD_\alpha)} \Bigg(\iint_{w,z \geq 1} g_i(z)f_j(z) \qty(\prod_{l \ne i,j} F_l(z)) f_i(w) \qty(\prod_{l \ne i} F_l(w)) \dd w \dd z \\
        &\hspace{8em} - \iint_{w,z \geq 1} g_i(z) \qty(\prod_{l \ne i} F_l(z)) f_i(w)f_j(w) \qty(\prod_{l \ne i,j} F_l(w))  \dd w \dd z  \Bigg) \\
        &= \frac{1}{\phi_i^2(\lambda; \gD_\alpha)} \Bigg(\iint_{w,z \geq 1} g_i(z)f_j(z) k(z) f_i(w) k(w) F_j(w) \dd w \dd z \\
        &\hspace{8em} - \iint_{w,z \geq 1} g_i(z) k(z)F_j(z) f_i(w)f_j(w)k(w) \dd w \dd z  \Bigg).
    \end{align*}
    Here, one can see that
    \begin{align*}
        \iint_{w,z \geq 1} g_i(z)f_j(z) k(z) f_i(w) k(w) F_j(w) \dd w \dd z \hspace{-10em}& \\
        &= \iint_{w,z \geq 1} \frac{k(z)k(w)}{2} \qty(g_i(z)f_j(z) f_i(w) F_j(w) + g_i(w)f_j(w) f_i(z) F_j(z)) \dd w \dd z, \\
        \iint_{w,z \geq 1} g_i(z) k(z)F_j(z) f_i(w)f_j(w)k(w) \dd w \dd z \hspace{-10em}& \\
        &= \iint_{w,z \geq 1} \frac{k(z)k(w)}{2} \qty(g_i(z)F_j(z) f_i(w) f_j(w) + g_i(w)F_j(w) f_i(z) f_j(z)) \dd w \dd z.
    \end{align*}
    Then, by elementary calculation, we obtain
    \begin{align*}
         &g_i(z)\textcolor{red}{ f_j(z)} f_i(w) \textcolor{red}{F_j(w)} + g_i(w)\textcolor{blue}{f_j(w)} f_i(z) \textcolor{blue}{F_j(z)} - (g_i(z)\textcolor{blue}{F_j(z)} f_i(w) \textcolor{blue}{f_j(w)} +g_i(w)\textcolor{red}{F_j(w)} f_i(z) \textcolor{red}{f_j(z)}) \\
         &\hspace{3em}= \textcolor{blue}{F_j(z)f_j(w)} ( g_i(w)f_i(z)- g_i(z) f_i(w)) + \textcolor{red}{F_j(w)f_j(z)} ( g_i(z)f_i(w) - g_i(w)f_i(z)) \\
         &\hspace{3em}=  (g_i(w)f_i(z)- g_i(z) f_i(w)) \cdot (F_j(z)f_j(w) - F_j(w)f_j(z)). \numberthis{\label{eq: integrated of J phi}}
    \end{align*}
    Obviously, (\ref{eq: integrated of J phi}) becomes $0$ when $z=w$.

    Firstly, let us consider the case $z\geq w$, where Assumption~\ref{asm: I is increasing} implies
    \begin{equation*}
        \frac{f(z+\uL_j)}{F(z+\uL_j)} \leq \frac{f(w+\uL_j)}{ F(w+\uL_j)} \implies F_j(w)f_j(z) \leq F_j(z)f_j(w).
    \end{equation*}
    On the other hand, we have
    \begin{equation*}
        g_i(z)f_i(w) = \frac{S_f(z)}{z^{\alpha+2}} \frac{S_f(w)}{w^{\alpha+1}},
    \end{equation*}
    which implies
    \begin{align*}
          g_i(w)f_i(z)- g_i(z) f_i(w) &= \frac{S_f(z)}{z^{\alpha+2}} \frac{S_f(w)}{w^{\alpha+1}} - \frac{S_f(w)}{w^{\alpha+2}} \frac{S_f(z)}{z^{\alpha+1}} \\
          &= \frac{S_f(w)S_f(z)}{w^{\alpha+1}z^{\alpha+1}} \qty(\frac{1}{w}- \frac{1}{z}) \geq 0, \qquad z \geq w.
    \end{align*}
    Therefore, when $z\geq w$, the integrand becomes positive.
    For the case $z \leq w$, one can easily reverse the inequalities above, which results in the positive integrand again.
    Therefore, $\frac{J_i(\lambda; \gD_\alpha)}{\phi_i(\lambda; \gD_\alpha)}$ is monotonically increasing.

\subsection{Proof of Lemma~\ref{lem: bound of I}}
Here, we assume $\lambda_1 \leq \ldots \leq \lambda_K$ without loss of generality, where $\sigma_i = i$ holds.
\subsubsection{\Fr distribution}
    By the monotonicity of $I_{i,\alpha+2}(\lambda)/I_{i,\alpha+1}(\lambda)$ in Lemma~\ref{lem: general monotonicity of I}, we have
    \begin{equation*}
        \frac{I_{i,\alpha+2}(\ul)}{I_{i,\alpha+1}(\ul)} \leq  \frac{I_{i,\alpha+2}(\lambda^*)}{I_{i,\alpha+1}(\lambda^*)}, \quad \qq{where} \lambda_j^* = \begin{cases}
            \ul_i, &j\leq i, \\
            \infty, &j > i.
        \end{cases}
    \end{equation*}
    From the definition of $I_{i,n}(\lambda; \alpha)$ in (\ref{eq: def_phi_I}), we have
    \begin{align*}
        I_{i,n}(\lambda^*; \alpha) &= \int_0^\infty \frac{1}{(z+\ul_i)^n} \exp(-\frac{i}{(z+\ul_i)^\alpha}) \dd z \\
        &= \frac{i^{-\frac{n-1}{\alpha}}}{\alpha}\int_0^{\frac{i}{\ul_i^\alpha}} u^{\frac{n-1}{\alpha}-1}e^{-u} \dd u \\
        &=  \frac{i^{-\frac{n-1}{\alpha}}}{\alpha} \gamma\qty(\frac{n-1}{\alpha}, \frac{i}{\ul_i^\alpha}),
    \end{align*}
    where $\gamma(n,x) = \int_0^x t^{n-1} e^{-t} \dd t$ denotes the lower incomplete gamma function.
    
    By substituting this result, we obtain
    \begin{align*}
        \frac{I_{i,\alpha+2}(\ul;\alpha)}{I_{i,\alpha+1}(\ul;\alpha)} \leq \frac{1}{\sqrt[\alpha]{i}} \frac{\gamma\qty(1+\frac{1}{\alpha}, \frac{i}{\ul_i^\alpha})}{\gamma\qty(1, \frac{i}{\ul_i^\alpha})}.
    \end{align*}
    Note that $\gamma(1, x) = 1 - e^{-x}$ holds for any $x >0$, and for any $\alpha >0$ 
    \begin{equation*}
        \gamma\qty(1+\frac{1}{\alpha}, x) \leq \frac{\sqrt[\alpha]{x}}{1+1/\alpha}(1-e^{-x}) = \frac{\sqrt[\alpha]{x}}{1+1/\alpha} \gamma(1, x)
    \end{equation*}
    by Lemma~\ref{lem: olver_lemma}, which proves the first inequality of Lemma~\ref{lem: bound of I}.

    Then, let us assume there exists a constant $C < \infty$ satisfying for any $x>0$
    \begin{equation}\label{eq: bnd_I_der_bnd}
        \gamma\qty(1+1/\alpha,x ) - (1-e^{-x})C \leq 0.
    \end{equation}
    The derivative of the LHS of (\ref{eq: bnd_I_der_bnd}) is given as
    \begin{equation*}
        \sqrt[\alpha]{x}e^{-x} - Ce^{-x},
    \end{equation*}
    which achieves the minimum at $x=C^\alpha$, i.e., its maximum is achieved at $x=0$ or $x=\infty$.
    Applying this finding in (\ref{eq: bnd_I_der_bnd}) gives $C \geq \Gamma\qty(1+\frac{1}{\alpha})$, which concludes the proof.

\subsubsection{Fr\'{e}chet-type distributions}
By the monotonicity of $ \frac{J_{i}(\lambda; \gD_\alpha)}{\phi_{i}(\lambda; \gD_\alpha)}$ in Lemma~\ref{lem: general monotonicity of I}, we have
\begin{equation*}
    \frac{J_{i}(\ul; \gD_\alpha)}{\phi_{i}(\ul; \gD_\alpha)} \leq  \frac{J_{i}(\lambda^*; \gD_\alpha)}{\phi_{i}(\lambda^*; \gD_\alpha)}, \quad \qq{where} \lambda_j^* = \begin{cases}
        \ul_i, &j\leq i, \\
        \infty, &j > i.
    \end{cases}
\end{equation*}
From the definition of $J_{i}(\lambda; \gD_\alpha)$ in (\ref{eq: def_phi_J}), we have
\begin{align*}
    J_{i}(\lambda^*; \gD_\alpha) &= \int_1^\infty \frac{S_f(z+\ul_i)}{(z+\ul_i)^{\alpha+2}} F^{i-1}(z+\ul_i) \dd z
\end{align*}
and
\begin{equation*}
    \phi_i(\lambda^*; \gD_\alpha) = \int_{1}^\infty f(z+\ul_i) F^{i-1}(z+\ul_i) \dd z.
\end{equation*}
Here, we begin by examining the Pareto distribution, as the proof for this case offers insights into the generalization of our results.
\paragraph{Pareto distribution}
Let us consider the $\gD_\alpha = \gP_\alpha$, where
    \begin{align*}
         J_{i}(\lambda^*; \gP_\alpha) &= \int_1^\infty \frac{\alpha}{(z+\ul_i)^{\alpha+2}} \qty(1-\frac{1}{(z+\ul_i)^{\alpha}})^{i-1} \dd z \\
         &= \int_0^{\frac{1}{(1+\ul_i)^\alpha}} w^{\frac{1}{\alpha}} (1-w)^{i-1} \dd w \\
         &=  B\qty(\frac{1}{(1+\ul_i)^\alpha}; 1+\frac{1}{\alpha}, i),
    \end{align*}
    where $B(x; a,b) = \int_0^x t^{a-1} (1-t)^{b-1} \dd t$ denotes the incomplete Beta function.
    Similarly,
    \begin{align*}
        \phi_i(\lambda^*; \gP_\alpha) &= \int_{1}^\infty \frac{\alpha}{(z+\ul_i)^{\alpha+1}} \qty(1-\frac{1}{(z+\ul_i)^{\alpha}})^{i-1} \dd z \\
        &= \int_0^{\frac{1}{(1+\ul_i)^\alpha}} w^{0} (1-w)^{i-1} \dd w \\
         &= B\qty(\frac{1}{(1+\ul_i)^\alpha}; 1, i).
    \end{align*}
    Therefore, by Lemma~\ref{lem: incomplete beta}
    \begin{align*}
        \frac{J_{i}(\lambda^*; \gP_\alpha)}{\phi_i(\lambda^*; \gP_\alpha)} = \frac{B\qty(\frac{1}{(1+\ul_i)^\alpha}; 1+\frac{1}{\alpha}, i)}{B\qty(\frac{1}{(1+\ul_i)^\alpha}; 1, i)} \leq \frac{B\qty(1+\frac{1}{\alpha},i)}{B(1,i)}.
    \end{align*}
    Since $B(x,y) = \frac{\Gamma(x)\Gamma(y)}{\Gamma(x+y)}$, $i\geq 1$, and $\alpha > 1$, applying Gautschi's inequality provides
    \begin{align*}
        \frac{B\qty(1+\frac{1}{\alpha},i)}{B(1,i)} = \frac{\Gamma\qty(1+\frac{1}{\alpha})\Gamma(i+1)}{\Gamma\qty(1+\frac{1}{\alpha}+i)} &= \frac{\Gamma\qty(1+\frac{1}{\alpha})}{i+\frac{1}{\alpha}}\frac{\Gamma(i+1)}{\Gamma\qty(i+\frac{1}{\alpha})}\numberthis{\label{eq: Pareto to general}} \\
        &\leq  \frac{\Gamma\qty(1+\frac{1}{\alpha})}{i+\frac{1}{\alpha}} (i+1)^{1-\frac{1}{\alpha}} \\
        &\leq \frac{2\alpha}{\alpha+1}\Gamma\qty(1+\frac{1}{\alpha}) \frac{1}{(i+1)^{\frac{1}{\alpha}}} \\ &\leq 2\Gamma\qty(1+\frac{1}{\alpha}) \frac{1}{\sqrt[\alpha]{i}}.
    \end{align*}
    On the other hand, for $x \in [0,1]$ it holds that $B(x; 1, i) = (1-(1-x)^i)$ and
    \begin{align*}
        B\qty(x; 1+ \frac{1}{\alpha},i) = \int_0^x t^{\frac{1}{\alpha}}(1-t)^{i-1} \dd t & \leq \int_0^x t^{\frac{1}{\alpha}}e^{-t(i-1)} \dd t \\
        &\leq e \int_0^x t^{\frac{1}{\alpha}}e^{-ti} \dd t \\
        &= \frac{e}{i^{1+\frac{1}{\alpha}}} \int_0^{xi} w^{\frac{1}{\alpha}} e^{-w} \dd w =  \frac{e}{i^{1+\frac{1}{\alpha}}} \gamma\qty(1+\frac{1}{\alpha}, xi).
    \end{align*}
    Then, by Lemma~\ref{lem: olver_lemma}, we have
    \begin{equation}\label{eq: Pareto to bounded S}
        \frac{B\qty(x; 1+ \frac{1}{\alpha},i)}{B\qty(x; 1,i)} \leq \frac{e}{i^{1+\frac{1}{\alpha}}} \frac{(xi)^{\frac{1}{\alpha}}}{1+1/\alpha} \frac{1-e^{-xi}}{(1-(1-x)^i)} \leq \frac{e}{i} \frac{(x)^{\frac{1}{\alpha}}}{1+1/\alpha},
    \end{equation}
    where the last inequality follows from $\lim_{x\to 0} \frac{1-e^{-xi}}{(1-(1-x)^i)} = 1$ and $\lim_{x\to 1} \frac{1-e^{-xi}}{(1-(1-x)^i)} < 1$.
    Therefore, by substituting $x= \frac{1}{(1+\ul_i)^{\alpha}}$, we have
    \begin{equation*}
        \frac{B\qty(\frac{1}{(1+\ul_i)^{\alpha}}; 1+ \frac{1}{\alpha},i)}{B\qty(\frac{1}{(1+\ul_i)^{\alpha}}; 1,i)} \leq \frac{e \alpha}{\alpha+1} \frac{1}{1+\ul_i} \leq \frac{e \alpha}{\alpha+1} \frac{1}{\ul_i}.
    \end{equation*}

\paragraph{Generalization to $\fD_\alpha$}
 Let us define a function for $x\geq 1$
    \begin{equation*}
       k(x) = k(x;\gD_\alpha) := \frac{\int_x^\infty \frac{S_f(z)}{z^{\alpha+2}} F^{i-1}(z) \dd z }{\int_{x}^\infty f(z) F^{i-1}(z) \dd z}.
    \end{equation*}
    Then, it holds that
    \begin{align*}
        \dv{k(x)}{x} &= \frac{1}{\qty(\int_{x}^\infty f(z) F^{i-1}(z) \dd z)^2} \Big( f(x)F^{i-1}(x) \int_x^\infty \frac{S_f(z)}{z^{\alpha+2}} F^{i-1}(z) \dd z \\
        &\hspace{17em}-\frac{S_f(x)}{x^{\alpha+2}}F^{i-1}(x) \int_{x}^\infty f(z) F^{i-1}(z) \dd z  \Big) \\
        &= \frac{F^{i-1}(x)}{\qty(\int_{x}^\infty f(z) F^{i-1}(z) \dd z)^2} \qty(f(x)\int_{x}^\infty \frac{S_f(z)}{z^{\alpha+2}} F^{i-1}(z) \dd z - \frac{S_f(x)}{x^{\alpha+2}} \int_x^\infty f(z) F^{i-1}(z) \dd z) \\
        &\leq \frac{F^{i-1}(x)}{\qty(\int_{x}^\infty f(z) F^{i-1}(z) \dd z)^2} \qty(\frac{S_f(x)}{x^{\alpha+1}}\int_{x}^\infty \frac{S_f(z)}{z^{\alpha+2}} F^{i-1}(z) \dd z - \frac{S_f(x)}{x^{\alpha+2}}\int_x^\infty \frac{S_f(z)}{z^{\alpha+1}} F^{i-1}(z) \dd z) \\
        &= \frac{F^{i-1}(x)}{\qty(\int_{x}^\infty f(z) F^{i-1}(z) \dd z)^2} 
        \frac{S_f(x)}{x^{\alpha+2}} 
        \qty(\int_{x}^\infty \frac{xS_f(z)}{z^{\alpha+2}} F^{i-1}(z) \dd z -\int_x^\infty \frac{S_f(z)}{z^{\alpha+1}} F^{i-1}(z) \dd z) \\
        &= \frac{F^{i-1}(x)}{\qty(\int_{x}^\infty f(z) F^{i-1}(z) \dd z)^2} 
        \frac{S_f(x)}{x^{\alpha+2}} 
        \qty(\int_{x}^\infty \qty(\frac{x}{z}-1)\qty(\frac{S_f(z)}{z^{\alpha+1}} F^{i-1}(z)) \dd z) \leq 0,
    \end{align*}
    which implies $k(x)$ is decreasing with respect to $x \geq 1$.
    Therefore,
    \begin{align*}
        \frac{J_{i}(\lambda^*; \gD_\alpha)}{\phi_{i}(\lambda^*; \gD_\alpha)} \leq \frac{\int_1^\infty \frac{S_f(z)}{z^{\alpha+2}} F^{i-1}(z) \dd z }{\int_{1}^\infty f(z) F^{i-1}(z) \dd z} &= i\int_1^\infty \frac{f(z)}{z} F^{i-1}(z) \dd z  \\
        &= \E\qty[\frac{1}{M_i}] \leq \frac{m}{A_l\sqrt[\alpha]{i}}, \numberthis{\label{eq: f sq}}
    \end{align*}
    where (\ref{eq: f sq}) follows from Assumption~\ref{asm: bounded block}.

    Next, let us consider the case (\ref{asm: bounded rho as alpha}) holds, where $S_F(x)$ is increasing.
    Let $1-F(z) = t$, which implies $z = U(1/t)$ for $t \in [1,\infty)$.
    Then, we have
    \begin{align*}
         \int_1^\infty \frac{f(z)}{z} F^{i-1}(z) \dd z &= \int_0^1 \frac{1}{U(1/t)} (1-t)^{(i-1)} \dd t \\
         &= \int_0^1  \frac{t^{\frac{1}{\alpha}}}{S_U(1/t)} (1-t)^{(i-1)} \dd t, \numberthis{\label{eq: S_F alpha bounded}} \\
         &\leq B\qty(1+\frac{1}{\alpha}; i).
    \end{align*}
    where (\ref{eq: S_F alpha bounded}) follows from $S_U \in \RV_{1/\alpha}$.
    Here, $S_U(1/t) = S_F^{\frac{1}{\alpha}}\qty(U(1/t))$ holds from (\ref{eq: SF and SU}), which implies that $\frac{1}{S_U(1/t)}$ is increasing when $S_F$ is increasing function since $U(1/t)$ is decreasing.
    From the definition of $U(1) = 1 = 1^{\frac{1}{\alpha}}S_U(1)$, we obtain $S_U(1) = 1$, i.e., $A_l=1$.
    Therefore, the analysis of the Pareto distributions from (\ref{eq: Pareto to general}) implies that $m \leq 2\Gamma\qty(1+\frac{1}{\alpha})$.

    Next, we obtain
    \begin{align*}
        \int_{1+\ul_i}^\infty f(z) F^{i-1}(z) \dd z = \frac{1}{i}\qty(1-F^{i}(1+\ul_i)) = B(1-F(1+\ul_i); 1,i).
    \end{align*}
    By Assumption~\ref{asm: bounded block} and (\ref{eq: S_F alpha bounded}), we have
    \begin{equation*}
        \int_{1+\ul_i}^\infty \frac{S_f(z)}{z^{\alpha+2}} F^{i-1}(z) \dd z \leq \frac{1}{A_l} B\qty(1-F(1+\ul_i); 1+\frac{1}{\alpha}, i).
    \end{equation*}
    Therefore, following the same steps from (\ref{eq: Pareto to bounded S}), we have
    \begin{align*}
        \frac{\int_{1+\ul_i}^\infty \frac{S_f(z)}{z^{\alpha+2}} F^{i-1}(z) \dd z }{\int_{1+\ul_i}^\infty f(z) F^{i-1}(z) \dd z} &\leq   \frac{1}{A_l} \frac{ B\qty(1-F(1+\ul_i); 1+\frac{1}{\alpha}, i)}{ B\qty(1-F(1+\ul_i); 1, i)} \\
        &\leq\frac{1}{A_l}  \frac{e}{i} \frac{\alpha}{\alpha+1} \qty(\frac{S_F(1+\ul_i)}{(1+\ul_i)^\alpha})^{\frac{1}{\alpha}}\numberthis{\label{eq: bounded S_F with SU}} \\
        &\leq  \frac{\alpha e}{A_l (\alpha+1)}\frac{A_u}{1+\ul_i} \\
        &\leq \frac{\alpha e}{A_l (\alpha+1)} \frac{A_u}{\ul_i} .
    \end{align*}
    where (\ref{eq: bounded S_F with SU}) follows from Assumption~\ref{asm: bounded block}.
    Here, when $S_F$ is an increasing function, then $A_u=\lim_{x\to \infty} S_F^{\frac{1}{\alpha}}(x)$ from (\ref{eq: SF and SU}). 
\begin{remark}
    When one considers the shifted distribution, (\ref{asm: bounded rho as alpha}) does not necessarily hold even when its original distribution satisfies it. 
    In such cases, it suffices to consider the shifted distribution function after the conditioning trick, where we have
    \begin{equation*}
        G(x) = F^*(x-1) = \frac{F(x)-F(1)}{1-F(1)}, \quad x \geq 1,
    \end{equation*}
    which implies
    \begin{equation*}
        1-G(x) = \frac{1-F(x)}{1-F(1)} = x^{-\alpha} S_G(x), \quad x \geq 1. 
    \end{equation*}
    Therefore, $S_G(x) = \frac{1}{1-F(1)}S_F(x)$ holds for $x\geq 1$ and thus they are tail-equivalent.
    Furthermore, if $F$ satisfies (\ref{asm: bounded rho as alpha}), then
    \begin{equation*}
        \frac{xg(x)}{1-G(x)} = \frac{xf(x)}{1-F(x)} \leq \alpha, \quad x \geq 1
    \end{equation*}
    holds.
    Therefore, $S_G(x)$ is monotonically increasing for $x\geq 1$ with $S_G(1) = 1$, which implies that $A=1$ and $m\leq 2\Gamma\qty(1+\frac{1}{\alpha})$.
\end{remark}

\subsection{Proof of Lemma~\ref{lem: stability_adv_i}}\label{app: stab_adv_i}
Although the overall proof is almost the same and follows the proofs of \citet{pmlr-v201-honda23a}, we provide the proofs for completeness.
\subsubsection{\Fr distribution}
From the definition of $\hat{\ell}_t = \qty(\ell_{t, I_t} \widehat{w^{-1}_{t, I_t}}) e_{I_t}$, when $I_t =i$, we have
\allowdisplaybreaks
    \begin{align*}
       \phi_i\qty(\eta_t \hat{L}_t; \gF_\alpha) &- \phi_i \qty(\eta_t \qty(\hat{L}_t + \qty(\ell_{t, i} \widehat{w^{-1}_{t, i}}) e_{i}); \gF_\alpha) \\
       &= \int_{0}^{\eta_t \ell_{t, i} \widehat{w^{-1}_{t, i}}} -\phi_i'(\eta_t \hat{L}_t + x e_{i}; \gF_\alpha) \dd x \\
       &\leq \alpha(\alpha+1) \int_{0}^{\eta_t \ell_{t, i} \widehat{w^{-1}_{t, i}}} I_{i,\alpha+2}(\underline{\eta_t \hat{L}_t + x e_{i}}; \alpha) \dd x \tag*{(by (\ref{eq: phi_I_relation}))} \\
       &\leq \alpha(\alpha+1) \int_{0}^{\eta_t \ell_{t, i} \widehat{w^{-1}_{t, i}}} I_{i,\alpha+2}(\underline{\eta_t \hat{L}_t}; \alpha) \dd x  \numberthis{\label{eq: using monotonicity of I}} \\
       &= \alpha(\alpha+1) \eta_t \ell_{t,i}  \widehat{w^{-1}_{t,i}} I_{i,\alpha+2}(\underline{\eta_t \hat{L}_t}; \alpha),
    \end{align*}
    \allowdisplaybreaks[0]
    where (\ref{eq: using monotonicity of I}) follows from the monotonicity of $I_{i,\alpha}$.
    Since $\widehat{w^{-1}_{t,i}}$ follows the geometric distribution with mean $w_{t,i}^{-1}$ given $\hat{L}_t$ and $I_t$, it holds that
    \begin{equation*}
        \E\qty[\widehat{w^{-1}_{t,I_t}}^{2} \eval \hat{L}_t, I_t] = \frac{2}{w_{t,I_t}^2} - \frac{1}{w_{t,I_t}} \leq \frac{2}{w_{t,I_t}^2}.
    \end{equation*}
    Since $I_t \ne i$ implies $\hat{\ell}_{t,i} = 0$, we obtain
    \begin{align*}
         \E \qty[\hat{\ell}_{t,i}\qty(\phi_i\qty(\eta_t \hat{L}_t) - \phi_i\qty(\eta_t \qty(\hat{L}_t + \hat{\ell}_t)) ) \eval \hat{L}_t] \hspace{-10em}&\\
         &= \E \qty[ \I [I_t = i] \hat{\ell}_{t,i}\qty(\phi_i\qty(\eta_t \hat{L}_t) - \phi_i\qty(\eta_t \qty(\hat{L}_t + \hat{\ell}_t)) ) \eval \hat{L}_t] \\
         &= \E \qty[ \I [I_t = i] \ell_{t,i} \widehat{w^{-1}_{t,i}} \qty(\phi_i\qty(\eta_t \hat{L}_t) - \phi_i\qty(\eta_t \qty(\hat{L}_t + \hat{\ell}_t)) ) \eval \hat{L}_t] \\
         &\leq  \E\qty[w_{t,i} \ell_{t,i} \widehat{w^{-1}_{t,i}} \cdot \alpha(\alpha+1) \eta_t \ell_{t,i} \widehat{w^{-1}_{t,i}} I_{i,\alpha+2}(\eta_t \hat{\uL}_t) \eval \hat{L}_t] \\
         &\leq 2\alpha(\alpha+1) \eta_t \E \qty[w_{t,i} \frac{\ell_{t,i}^2 I_{i,\alpha+2}(\eta_t \hat{\uL}_t) }{w_{t,i}^2} \eval \hat{L}_t] \\
         &\leq 2(\alpha+1) \eta_t \E \qty[\frac{I_{i,\alpha+2}(\eta_t \hat{\uL}_t) }{I_{i,\alpha+1}(\eta_t \hat{\uL}_t)} \eval \hat{L}_t] \quad \qty(\text{by } w_{t,i}=\alpha I_{i,\alpha+1}(\eta_t \hat{\uL}_t), \, \ell_{t,i} \leq 1) \\
         &\leq \frac{2\alpha}{\eta \hat{\uL}_{t,i}}  \land 2(\alpha+1)\eta_t\frac{\Gamma\qty(1+\frac{1}{\alpha})}{\sqrt[\alpha]{\sigma_i}},
    \end{align*}
    where the last inequality follows from Lemma~\ref{lem: bound of I} for $\gF_\alpha$.

\subsubsection{Fr\'{e}chet-type distributions}
From the definition of $\phi$ in (\ref{eq: def_phi}) and (\ref{eq: bounded f' with rho}) from Assumption~\ref{asm: derivative of f}, we have
\begin{align*}
    -\phi_i'(\lambda; \gD_\alpha) &= \int_{1}^{\infty} -f'(z+\lambda_i) \prod_{j\ne i} F(z+\lambda_j) \dd z  \\
    &\leq \int_{1}^{\infty} \rho_2 \frac{f(z+\lambda_i)}{z+\lambda_i} \prod_{j\ne i} F(z+\lambda_j)  \dd z = \rho_2 J_i(\lambda; \gD_\alpha).
\end{align*}
Therefore, we can replace $\alpha(\alpha+1)I_{i,\alpha+2}$ with $\rho_2 J_i$, which gives
\begin{align*}
     \E \qty[\hat{\ell}_{t,i}\qty(\phi_i\qty(\eta_t \hat{L}_t) - \phi_i\qty(\eta_t \qty(\hat{L}_t + \hat{\ell}_t)) ) \eval \hat{L}_t] &\leq \E\qty[w_{t,i} \ell_{t,i} \widehat{w^{-1}_{t,i}}\cdot \eta_t \ell_{t,i} \widehat{w^{-1}_{t,i}} \rho_2 J_{i}(\eta_t \hat{\uL}_t; \gD_\alpha) \eval \hat{L}_t] \\
     &\leq 2 \eta_t \E \qty[w_{t,i} \frac{\ell_{t,i}^2 \rho_2 J_{i}(\eta_t \hat{\uL}_t; \gD_\alpha) }{w_{t,i}^2} \eval \hat{L}_t] \\
     &\leq 2\rho_2 \eta_t \E \qty[\frac{J_{i}(\eta_t \hat{\uL}_t; \gD_\alpha) }{\phi_i(\eta_t \hat{\uL}_t; \gD_\alpha)} \eval \hat{L}_t],
\end{align*}
where Lemma~\ref{lem: bound of I} concludes the proof.

\subsection{Proof of Lemma~\ref{lem: stability}}
By Lemmas~\ref{lem: bound of I} and~\ref{lem: stability_adv_i}, for $\gF_\alpha$, we have
\begin{align*}
    \E \qty[\hat{\ell}_{t}\qty(\phi_i\qty(\eta_t \hat{L}_t) - \phi_i\qty(\eta_t \qty(\hat{L}_t + \hat{\ell}_t)) ) \eval \hat{L}_t] &\leq \sum_{i\in[K]} 2(\alpha+1) \eta_t \frac{\Gamma\qty(1+\frac{1}{\alpha})}{\sqrt[\alpha]{\sigma_i}}  \\
    &\leq 2(\alpha+1)\eta_t \Gamma\qty(1+\frac{1}{\alpha}) \qty(1 + \int_1^K x^{-1/\alpha} \dd x) \\
    &= 2(\alpha+1)\eta_t \Gamma\qty(1+\frac{1}{\alpha}) \frac{\alpha K^{1-1/\alpha}-1}{\alpha-1} \\
    &\leq  \frac{2\alpha(\alpha+1)}{\alpha-1}\eta_t \Gamma\qty(1+\frac{1}{\alpha}) K^{1-1/\alpha}.
\end{align*}
Similarly, for $\gD_\alpha \in \fD_\alpha$, we have
\begin{equation*}
     \E \qty[\hat{\ell}_{t}\qty(\phi_i\qty(\eta_t \hat{L}_t) - \phi_i\qty(\eta_t \qty(\hat{L}_t + \hat{\ell}_t)) ) \eval \hat{L}_t] \leq \frac{2\alpha \rho_2}{\alpha-1} \eta_t  \frac{m}{A_l} K^{1-1/\alpha}.
\end{equation*}

\section{Regret bound for adversarial bandits: Penalty}\label{app: penalty}
This section provides the proofs on Lemma~\ref{lem: penalty}.
\subsection{Penalty term analysis for the \Fr distributions}\label{app: penalty fr}
    By letting $k_\alpha(z) = \sum_{i} \frac{1}{(z+\eta_t \hat{\uL}_{t,i})^\alpha} \in \left(0, \frac{K}{z^\alpha}\right]$, we have
    \begin{align*}
        \E \qty[r_{t,I_t} - r_{t,i^*} \eval \hat{L}_t ]  &\leq \sum_{i \ne i^*} \E\qty[\I[I_t = I] r_{t,i} \eval \hat{L}_t ] \\
        &= \alpha \int_0^\infty \sum_{i \ne i^*} \frac{1}{(z+\eta_t \hat{\uL}_{t,i})^\alpha} e^{-k_\alpha(z)} \dd z \\
        &\leq  \alpha \int_0^\infty \sum_{i \ne i^*} \frac{1}{(z+\eta_t \hat{\uL}_{t,i})^\alpha} \dd z = \frac{\alpha}{\alpha-1} \sum_{i\ne i^*} \frac{1}{(\eta_t \hat{\uL}_{t,i})^{\alpha-1}}.
    \end{align*}
    On the other hand,
        \begin{align*}
        \alpha \int_0^\infty \sum_{i \ne i^*} \frac{1}{(z+\eta_t \hat{\uL}_{t,i})^\alpha} e^{-k_\alpha(z)} \dd z &\leq  \alpha \int_0^\infty k_\alpha(z) e^{-k_\alpha(z)} \dd z \\
        &= \alpha \int_0^{\sqrt[\alpha]{K}} k_\alpha(z) e^{-k_\alpha(z)} \dd z + \alpha \int_{\sqrt[\alpha]{K}}^\infty k_\alpha(z) e^{-k_\alpha(z)} \dd z \numberthis{\label{eq: penalty_bnd_hnd}} \\
        &\leq \alpha \int_0^{\sqrt[\alpha]{K}} e^{-1} \dd z + \alpha \int_{\sqrt[\alpha]{K}}^\infty \frac{K}{z^{\alpha}} e^{- \frac{K}{z^{\alpha}}} \dd z \\
        &= \alpha e^{-1} \sqrt[\alpha]{K} + \sqrt[\alpha]{K} \int_0^1 w^{-\frac{1}{\alpha}}e^{-w} \dd w \\
        &= \qty(\alpha e^{-1} + \gamma\qty(1-\frac{1}{\alpha}, 1))\sqrt[\alpha]{K},
    \end{align*}
    where the first term of (\ref{eq: penalty_bnd_hnd}) follows from the fact that $xe^{-x} \leq e^{-1}$ and the second term follows from the fact that $xe^{-x}$ is increasing for $x\leq 1$ and $k_\alpha(z) \leq 1$ holds for $z \geq \sqrt[\alpha]{K}$.
    From the definition of the lower incomplete gamma function, one can obtain
    \begin{equation*}
        \gamma(s+1, x) = s\gamma(s,x) -x^se^{-x} \implies \gamma\qty(2-\frac{1}{\alpha}, 1) = \qty(1-\frac{1}{\alpha})\gamma\qty(1-\frac{1}{\alpha}, 1) - e^{-1},
    \end{equation*}
    which implies
    \begin{align*}
        \gamma\qty(1-\frac{1}{\alpha}, 1) &= \frac{\alpha}{\alpha-1} \gamma\qty(2-\frac{1}{\alpha}, 1) + \frac{\alpha e^{-1}}{\alpha-1} \\
        &\leq \frac{\alpha}{\alpha-1} \frac{\alpha}{2\alpha-1} (1-e^{-1}) + \frac{\alpha e^{-1}}{\alpha-1}
    \end{align*}
    by Lemma~\ref{lem: olver_lemma} again.
    Therefore, by doing elementary calculations, we obtain that
    \begin{equation*}
        \alpha \int_0^\infty \sum_{i \ne i^*} \frac{1}{(z+\eta_t \hat{\uL}_{t,i})^\alpha} e^{-k_\alpha(z)} \dd z \leq \qty(\frac{0.74\alpha^3  + 0.27\alpha^2}{(\alpha-1)(2\alpha-1)}) \sqrt[\alpha]{K}.
    \end{equation*}
\subsection{Penalty term analysis for the Pareto distributions}\label{app: penalty Pr}
    By letting $k_\alpha(z) = \sum_{i} \frac{1}{(z+\eta_t \hat{\uL}_{t,i})^\alpha} \in \left(0, \frac{K}{z^\alpha}\right]$, we have
    \begin{align*}
        \E \qty[r_{t,I_t} - r_{t,i^*} \eval \hat{L}_t ]  &\leq \sum_{i \ne i^*} \E\qty[\I[I_t = I] r_{t,i} \eval \hat{L}_t ] \\
        &= \alpha \int_1^\infty \sum_{i \ne i^*} \qty(\frac{1}{(z+\eta_t \hat{\uL}_{t,i})^\alpha} \prod_{j \ne i} \qty(1-\frac{1}{(z+\eta_t \hat{\uL}_{t,j})^\alpha})) \dd z \\
        &\leq e \alpha \int_1^\infty \sum_{i \ne i^*} \frac{1}{(z+\eta_t \hat{\uL}_{t,i})^\alpha} e^{-k_\alpha(z)} \dd z.
    \end{align*}
    Therefore, the proof in Section~\ref{app: penalty fr} immediately concludes the Pareto case.
    
\subsection{Penalty term for the Fr\'{e}chet-type distributions}\label{app: penalty gen} 
Here, let us consider the inverse of the tail function, which is the tail quantile function defined as
\begin{equation}\label{def: quantile}
    U(t) := \inf \qty{x: F(x) \geq \frac{1}{t}}.
\end{equation}
Note that when $F$ and $U$ are continuous, $1-F(U(t)) = \frac{1}{t}$ holds.
Then, as in the other cases, we have
    \begin{align*}
        \E \qty[r_{t,I_t} - r_{t,i^*} \eval \hat{L}_t ] &\leq \sum_{i \ne i^*} \E\qty[\I[I_t = I] r_{t,i} \eval \hat{L}_t ] \\
        &= \int_1^\infty \sum_{i \ne i^*} \qty( (z+\eta_t \hat{\uL}_{t,i}) f(z+\eta_t \hat{\uL}_{t,i})  \prod_{j \ne i} F(z+\eta_t \hat{\uL}_{t,j})) \dd z \\
        &= \int_1^{U(K)} \sum_{i \ne i^*} \qty( \frac{S_f(z+\eta_t \hat{\uL}_{t,i})}{(z+\eta_t \hat{\uL}_{t,i})^{\alpha}}  \prod_{j \ne i} F(z+\eta_t \hat{\uL}_{t,j})) \dd z \\
        &\hspace{2em}+ \int_{U(K)}^\infty \sum_{i \ne i^*} \qty( \frac{S_f(z+\eta_t \hat{\uL}_{t,i})}{(z+\eta_t \hat{\uL}_{t,i})^{\alpha}}  \prod_{j \ne i} F(z+\eta_t \hat{\uL}_{t,j})) \dd z.\numberthis{\label{eq: penalty decomp gen}}
    \end{align*}
The first term of (\ref{eq: penalty decomp gen}) can be bounded as 
\allowdisplaybreaks
\begin{align*}
    \int_1^{U(K)} \sum_{i \ne i^*} \qty( \frac{S_f(z+\eta_t \hat{\uL}_{t,i})}{(z+\eta_t \hat{\uL}_{t,i})^{\alpha}}  \prod_{j \ne i} F(z+\eta_t \hat{\uL}_{t,j})) \dd z \hspace{-16em}&\\
    &=  \int_1^{U(K)} \sum_{i \ne i^*} \qty( \frac{\varrho(z+\eta_t \hat{\uL}_{t,i}) S_F(z+\eta_t \hat{\uL}_{t,i})}{(z+\eta_t \hat{\uL}_{t,i})^{\alpha}}  \prod_{j \ne i} F(z+\eta_t \hat{\uL}_{t,j})) \dd z  \tag*{by (\ref{eq: Sf and SF})}\\
    &\leq \int_1^{U(K)} \sum_{i \ne i^*} \qty( \frac{\rho_1 S_F(z+\eta_t \hat{\uL}_{t,i})}{(z+\eta_t \hat{\uL}_{t,i})^{\alpha}}  \prod_{j  \ne i} F(z+\eta_t \hat{\uL}_{t,j})) \dd z \\
    &\leq e \rho_1 \int_1^{U(K)} \qty( \sum_{i\in [K]} \frac{S_F(z+\eta_t \hat{\uL}_{t,i})}{(z+\eta_t \hat{\uL}_{t,i})^{\alpha}})  \exp \qty(-\sum_{i\in [K]}\frac{S_F(z+\eta_t \hat{\uL}_{t,i})}{(z+\eta_t \hat{\uL}_{t,i})^{\alpha}}) \dd z \\
    &\leq \rho_1 e\int_1^{U(K)} e^{-1} \dd z  \leq  \rho_1 U(K) \leq \rho_1 A_u K^{\frac{1}{\alpha}}
\end{align*}
\allowdisplaybreaks[0]
where the second last inequality follows from $xe^{-x}\leq e^{-1}$.

For the second term of (\ref{eq: penalty decomp gen}), from $S_f(x) = \varrho(x) S_F(x)$ in (\ref{eq: Sf and SF}), we have
\begin{align*}
    \int_{U(K)}^\infty \sum_{i \ne i^*} \qty( \frac{S_f(z+\eta_t \hat{\uL}_{t,i})}{(z+\eta_t \hat{\uL}_{t,i})^{\alpha}}  \prod_{j \ne i} F(z+\eta_t \hat{\uL}_{t,j})) \dd z \hspace{-15em}&\\
    &= \int_{U(K)}^\infty \sum_{i \ne i^*} \qty( \frac{ \varrho(z+\eta_t \hat{\uL}_{t,i}) S_F(z+\eta_t \hat{\uL}_{t,i})}{(z+\eta_t \hat{\uL}_{t,i})^{\alpha}}  \prod_{j \ne i} F(z+\eta_t \hat{\uL}_{t,j})) \dd z \\
    &\leq \rho_1 \int_{U(K)}^\infty \sum_{i \ne i^*} \qty( \frac{ S_F(z+\eta_t \hat{\uL}_{t,i})}{(z+\eta_t \hat{\uL}_{t,i})^{\alpha}}  \prod_{j \ne i} F(z+\eta_t \hat{\uL}_{t,j})) \dd z \\
    &=  \rho_1 \int_{U(K)}^\infty \qty(\sum_{i \in [K]} (1-F(z+\eta_t \hat{\uL}_{t,i})) \prod_{j \ne i} F(z+\eta_t \hat{\uL}_{t,j})) \dd z \\
    &\leq \rho_1 \int_{U(K)}^\infty \qty(\sum_{i \in [K]} (1-F(z+\eta_t \hat{\uL}_{t,i})) \exp(-\sum_{j  \ne i} (1-F(z+\eta_t \hat{\uL}_{t,j})))) \dd z \\
    &\leq e\rho_1 \int_{U(K)}^\infty \qty(\sum_{i  \in [K]} (1-F(z))) \exp(-\sum_{j \in [K]} (1-F(z))) \dd z  \numberthis{\label{eq: pen gen}} \\
    &= e\rho_1 \int_{U(K)}^\infty K(1-F(z)) \exp(-K(1-F(z))) \dd z \\
    &= e\rho_1 \int_{U(K)}^\infty K\frac{S_F(z)}{z^{\alpha}} \exp(-K\frac{S_F(z)}{z^{\alpha}}) \dd z,
\end{align*}
where (\ref{eq: pen gen}) holds since $xe^{-x}$ is increasing with respect to $x \in [0,1]$ and $\sum_{i \in [K]} (1-F(z+\eta_t \hat{\uL}_{t,i}))) \leq \sum_{i \in [K]}(1-F(z)) \leq 1$ for $z \geq U(K)$.
Here, $S_F(z)$ is increasing function with respect to $z \geq \nu$, which implies
\begin{align*}
    e\rho_1 \int_{U(K)}^\infty K\frac{S_F(z)}{z^{\alpha}} & \exp(-K\frac{S_F(z)}{z^{\alpha}}) \dd z \\
    &\leq e\rho_1 \int_{U(K)}^\infty K\frac{S_F(z)}{z^{\alpha}} \exp(-K\frac{S_F(U(K))}{z^{\alpha}}) \dd z \\
    & =\frac{e\rho_1}{\alpha}\int_{U(K)}^\infty \frac{S_F(z)z}{S_F(U(K))}\frac{K\alpha  S_F(U(K))}{z^{\alpha+1}} \exp(-K\frac{S_F(U(K))}{z^{\alpha}}) \dd z \\
    &=\frac{e\rho_1}{\alpha} \int_{U(K)}^\infty \frac{S_F(z)z}{S_F(U(K))} \qty(-\dv{x} \frac{KS_F(U(K))}{z^{\alpha}})\exp(-K\frac{S_F(U(K))}{z^{\alpha}}) \dd z.
\end{align*}
By Potter's bound (Lemma~\ref{lem: potter}) with arbitrary chosen $\delta>0$, there exists some constants $b_\delta $ such that for any $z \geq U(K)$
\begin{equation*}
    \frac{S_F(z)}{S_F(U(K))} \leq b_\delta \qty(\frac{z}{U(K)})^\delta.
\end{equation*}
Therefore, for $\delta > 0$
\begin{align*}
    e\rho_1 \int_{U(K)}^\infty K\frac{S_F(z)}{z^{\alpha}} \exp(-K\frac{S_F(z)}{z^{\alpha}}) \dd z \hspace{-13em}&
    \\ &\leq \frac{e\rho_1}{\alpha}\int_{U(K)}^\infty b_\delta \frac{z^{1+\delta} }{U^\delta(K)} \qty(-\dv{x} \frac{KS_F(U(K))}{z^{\alpha}})\exp(-K\frac{S_F(U(K))}{z^{\alpha}}) \dd z \\
    &= \frac{e\rho_1}{\alpha}\int_{U(K)}^\infty b_\delta  \frac{K^{\frac{1+\delta}{\alpha}}}{U^\delta(K)} S_F(U(K))^{\frac{1+\delta}{\alpha}} \qty(K\frac{S_F(U(K))}{z^{\alpha}})^{-\frac{1+\delta}{\alpha}} \\
    &\hspace{8em} \cdot \qty(-\dv{x} \frac{KS_F(U(K))}{z^{\alpha}})\exp(-K\frac{S_F(U(K))}{z^{\alpha}}) \dd z \\
    &= \frac{e\rho_1}{\alpha}\int_{U(K)}^\infty b_\delta  K^{\frac{1}{\alpha}}  S_F(U(K))^{\frac{1}{\alpha}} \qty(K\frac{S_F(U(K))}{z^{\alpha}})^{-\frac{1+\delta}{\alpha}} \qty(-\dv{x} \frac{KS_F(U(K))}{z^{\alpha}}) \\
    & \hspace{8em} \cdot \exp(-K\frac{S_F(U(K))}{z^{\alpha}}) \dd z,
\end{align*}
where the last equality follows from the definition of the tail quantile function,
\begin{equation*}
    1-F(U(K)) = \frac{1}{K} = \frac{S_F(U(K))}{U^\alpha(K)} \iff \frac{S_F^{\frac{\delta}{\alpha}}(U^\delta(K))}{U(K)} = K^{-\frac{\delta}{\alpha}}.
\end{equation*}
By letting $w = K\frac{S_F(U(K))}{z^{\alpha}}$, we have for any $\delta \in (0, \alpha-1)$ and $K\geq 2$
\begin{align*}
     e\rho_1 \int_{U(K)}^\infty K\frac{S_F(z)}{z^{\alpha}} \exp(-K\frac{S_F(z)}{z^{\alpha}}) \dd z &\leq \frac{e\rho_1}{\alpha} \int_{0}^1 b_\delta  K^{\frac{1}{\alpha}}  S_F^{\frac{1}{\alpha}}(U(K)) w^{-\frac{1+\delta}{\alpha}}e^{-w} \dd w \\
     &= \frac{e\rho_1}{\alpha} b_\delta S_F^{\frac{1}{\alpha}}(U(K)) K^{\frac{1}{\alpha}} \gamma\qty(1-\frac{1+\delta}{\alpha}, 1) \\
     &\leq \frac{e\rho_1}{\alpha} b_\delta A_u \gamma\qty(1-\frac{1+\delta}{\alpha}, 1) K^{\frac{1}{\alpha}}.
\end{align*}
Letting $C_{1,1}(\gD_\alpha) = \min_{\delta \in (0, \alpha-1)} \frac{e\rho_1}{\alpha} b_\delta A_u \gamma\qty(1-\frac{1+\delta}{\alpha}, 1) +\rho_1 A_u$ concludes the proof.

\subsection{Penalty term analysis dependent on the loss estimation}\label{app: penalty bounded}
Similarly to Section~\ref{app: penalty gen}, we have
\begin{align*}
        \E \qty[r_{t,I_t} - r_{t,i^*} \eval \hat{L}_t ] &\leq \sum_{i \ne i^*} \E\qty[\I[I_t = i] r_{t,i} \eval \hat{L}_t ] \\
        &= \int_1^{\infty} \sum_{i \ne i^*} \qty( \frac{S_f(z+\eta_t \hat{\uL}_{t,i})}{(z+\eta_t \hat{\uL}_{t,i})^{\alpha}}  \prod_{j \ne i} F(z+\eta_t \hat{\uL}_{t,j})) \dd z \\
        &\leq \int_1^{\infty} \sum_{i \ne i^*} \frac{S_f(z+\eta_t \hat{\uL}_{t,i})}{(z+\eta_t \hat{\uL}_{t,i})^{\alpha}} \dd z \\
        &\leq \int_1^{\infty} \sum_{i \ne i^*} \qty( \frac{\rho_1 A_u^{\alpha}}{(z+\eta_t \hat{\uL}_{t,i})^{\alpha}} ) \dd z \numberthis{\label{eq: Sf bound}} \\
        &\leq \frac{\rho_1 A_u^{\alpha}}{\alpha-1} \sum_{i\ne i^*} \frac{1}{(\eta_t \hat{\uL}_{t,i})^{\alpha-1}},
`\end{align*}
where (\ref{eq: Sf bound}) follows from (\ref{eq: Sf and SF}), $S_f(x) = S_F(x) \varrho(x)$, and the boundedness of $S_F(x)$ and $\varrho(x)$.

\begin{remark}\label{rm: Penalty negative support}
    When $\nu <0$, the perturbation $r_{t,i}$ can be negative.
    In such cases, we have 
    \begin{align*}
        \sum_{i \ne i^*} \E\qty[\I[I_t = I] r_{t,i} - r_{t,i^*}\eval \hat{L}_t ] &\leq \sum_{i\ne i^*}\E\qty[ \I[I_t = i] r_{t,i}\eval \hat{L}_t] - \E\qty[r_{t,i^*}\eval \hat{L}_t] \\
        &= \sum_{i\ne i^*}\E\qty[ \I[I_t = i] r_{t,i} \eval \hat{L}_t] - \E[r_{t,i^*}] \\
        &\leq \int_0^{\infty} \sum_{i \ne i^*} \qty( \frac{S_f(z+\eta_t \hat{\uL}_{t,i})}{(z+\eta_t \hat{\uL}_{t,i})^{\alpha}}  \prod_{j \ne i} F(z+\eta_t \hat{\uL}_{t,j})) \dd z - \E[r_{t,i^*}].
    \end{align*}
    Therefore, when $\nu < 0$, adding a constant is enough (at most) to provide the upper bound.
\end{remark}

\section{Regret bound for stochastic bandits}\label{app: stoc all}
In this section, we provide the proof of Theorem~\ref{thm: sto_all} based on the self-bounding technique, which requires a regret lower bound of the policy~\citep{zimmert2021tsallis}.
We first generalize the results of \citet{pmlr-v201-honda23a} to \Fr distributions with index $\alpha >1$ and then generalize it to $\fD_\alpha$.
Here, we consider two events $F_t$ and $D_t$, which are defined by
\begin{align*}
    F_t &:= \qty{\sum_{i \ne i^*} \frac{1}{(\eta_t \hat{\uL}_{t,i})^\alpha} \leq 1}, \\
    D_t &:= \qty{\sum_{i \ne i^*} 1-F(U(2) +\eta_t \hat{\uL}_{t,i}) \leq 1-F(U(2)+1) },
\end{align*}
where $U(2)$ denotes the median of $\gD_\alpha$.
Note that $F(U(2)+1) <1$ holds since $F(x)<1$ holds for any finite $x$ if $\gD_\alpha \in \fD_{\alpha}^{\text{all}}$.
The key property on these events are 
\begin{equation}\label{eq: key property}
    \hat{\uL}_{t,i^*}=0, \qq{and} \eta_t \hat{\uL}_{t,j} \geq 1, \, \forall j \ne i^*.
\end{equation}
Note that the choice of RHS, $1$ and $1-F(U(2)+1)$ is not mandatory, and thus one can choose any real values for $F_t$ and $1-F(U(b)+1)$ with $b>1$ for $D_t$.

\subsection{Regret lower bounds}\label{app: regret lb}
Here, we provide the regret lower bounds for $\gF_\alpha$ and $\fD_\alpha$, respectively.
\begin{lemma}\label{lem: lb frechet}
Let $\Delta := \min_{i\ne i^*} \Delta_i$.
Then, for any $\alpha >1$, there exists some constants $c_{s,1}(\gF_\alpha)\in (0,1)$ that only depend on $\alpha$ such that
\begin{enumerate}[label=(\roman*)]
    \item On $F_t$, $\sum_{i\ne i^*}\Delta_i w_{t,i} \geq c_{s,1}(\gF_\alpha)\sum_{i\ne i^*} \frac{\Delta_i}{(\eta_t \hat{\uL}_{t,i})^\alpha}$ and $w_{t,i^*}\geq 1/e$.
    \item On $F_t^c$, $\sum_{i\ne i^*}\Delta_i w_{t,i} \geq \frac{\Delta}{2^{\alpha+1}+1} $.
\end{enumerate}
\end{lemma}
\begin{proof}
    Let $\hat{\uL}' = \min_{i \ne i^*} \hat{\uL}_{t,i}$. Then, for any $b >0$ we have
    \begin{align*}
        \sum_{i\ne i^*} \Delta_i w_{t,i} &= \alpha \int_0^\infty \qty(\sum_{i\ne i^*} \frac{\Delta_i}{(z+\eta_t \hat{\uL}_{t,i})^{\alpha+1}}) \exp(-\sum_{i\in[K]}\frac{1}{(z+\eta_t \hat{\uL}_{t,i})^{\alpha}}) \dd z 
        \\
        &\geq\alpha \int_{b \eta_t \hat{\uL}'}^\infty \qty(\sum_{i\ne i^*} \frac{\Delta_i}{(z+\eta_t \hat{\uL}_{t,i})^{\alpha+1}}) \exp(-\sum_{i\in[K]}\frac{1}{(z+\eta_t \hat{\uL}_{t,i})^{\alpha}}) \dd z.
    \end{align*}

    (i) Consider the case $\sum_{i \ne i^*} \frac{1}{(\eta_t \hat{\uL}_{t,i})^\alpha} \leq 1$, we have
    \allowdisplaybreaks
    \begin{align*}
         \sum_{i\ne i^*} \Delta_i w_{t,i}  &\geq  \alpha \int_{b \eta_t \hat{\uL}'}^\infty \qty(\sum_{i\ne i^*} \frac{\Delta_i}{(z+\eta_t \hat{\uL}_{t,i})^{\alpha+1}}) \exp(-\sum_{i\in[K]}\frac{1}{(z+\eta_t \hat{\uL}_{t,i})^{\alpha}}) \dd z \\
         &\geq  \alpha\int_{b \eta_t \hat{\uL}'}^\infty \qty(\sum_{i\ne i^*} \frac{\Delta_i}{(z+\eta_t \hat{\uL}_{t,i})^{\alpha+1}}) \exp(- \frac{1}{(b\eta_t \hat{\uL}')^\alpha} - \sum_{i \ne i^*}\frac{1}{(z+\eta_t \hat{\uL}_{t,i})^{\alpha}}) \dd z \\
         &\geq \alpha \int_{b \eta_t \hat{\uL}'}^\infty \qty(\sum_{i\ne i^*} \frac{\Delta_i}{(z+\eta_t \hat{\uL}_{t,i})^{\alpha+1}}) \exp(-\qty(1+\frac{1}{b^\alpha}) \sum_{i \ne i^*}\frac{1}{(z+\eta_t \hat{\uL}_{t,i})^{\alpha}}) \dd z \\
         &\geq \alpha \int_{b \eta_t \hat{\uL}'}^\infty \qty(\sum_{i\ne i^*} \frac{\Delta_i}{(z+\eta_t \hat{\uL}_{t,i})^{\alpha+1}}) \exp(-\qty(1+\frac{1}{b^\alpha})) \dd z \\
         &= \exp(-\qty(1+\frac{1}{b^\alpha})) \qty(\sum_{i\ne i^*} \frac{\Delta_i}{(b \eta_t \hat{\uL}'+\eta_t \hat{\uL}_{t,i})^{\alpha}})   \\
         &\geq \qty(\sum_{i\ne i^*} \frac{\Delta}{((1+b) \eta_t \hat{\uL}_{t,i})^{\alpha}}) \exp(-\qty(1+\frac{1}{b^\alpha})) \\
         &=  \frac{\exp(-\qty(1+\frac{1}{b^\alpha}))}{(1+b)^\alpha} \qty(\sum_{i\ne i^*} \frac{\Delta}{(\eta_t \hat{\uL}_{t,i})^{\alpha}}).
    \end{align*}
    \allowdisplaybreaks[0]
    Since $b>0$ is arbitrary chose, we can set $c_{s,1}(\gF_\alpha) = \max_{b >0} \frac{\exp(-\qty(1+\frac{1}{b^\alpha}))}{(1+b)^\alpha} \in (0,1)$.

    Since $\hat{\uL}_{t,i^*} = 0$ holds on $F_t$, we have \vspace{-0.3em}
    \begin{align*}
        w_{t,i^*} &= \int_0^\infty \frac{\alpha}{z^{\alpha+1}} \exp(-\sum_{i\in[K]} \frac{1}{(z+\eta_t \hat{\uL}_{t,i})^\alpha}) \dd z  \\
        &\geq \int_0^\infty \frac{\alpha}{z^{\alpha+1}} \exp(-\sum_{i\ne i^*} \frac{1}{(z+\eta_t \hat{\uL}_{t,i})^\alpha} - \frac{1}{z^{\alpha}}) \dd z \\
        &\geq e^{-1} \int_0^\infty \frac{\alpha}{z^{\alpha+1}} \exp(- \frac{1}{z^{\alpha}}) \dd z = \frac{1}{e},
    \end{align*}
    which concludes the proof of the case (i).
    
    (ii) When $\sum_{i \ne i^*} \frac{1}{(\eta_t \hat{\uL}_{t,i})^\alpha} \geq 1$, we have for any $z \geq b \eta_t \hat{\uL}'$ 
    \begin{align*}
        \sum_{i\in[K]}\frac{1}{(z+\eta_t \hat{\uL}_{t,i})^{\alpha}} &\leq \sum_{i\ne i^*} \frac{1}{(z+\eta_t \hat{\uL}_{t,i})^{\alpha}} + \frac{1}{z^\alpha} \\
        &\leq \sum_{i\ne i^*} \frac{1}{(z+\eta_t \hat{\uL}_{t,i})^{\alpha}} + \frac{1}{(\frac{z+b \eta_t \hat{\uL}'}{2})^\alpha} \\
        &\leq \sum_{i\ne i^*} \frac{1}{(z+\eta_t \hat{\uL}_{t,i})^{\alpha}} +  \sum_{i\ne i^*}\frac{2^\alpha}{(z+b \eta_t \hat{\uL}_{t,i})^{\alpha}}.
    \end{align*}
    Therefore, by letting $b=1$, we obtain that
    \begin{align*}
        \sum_{i\ne i^*} \Delta_i w_{t,i}  &\geq \alpha \Delta \int_{\eta_t \hat{\uL}'}^\infty \qty(\sum_{i\ne i^*} \frac{1}{(z+\eta_t \hat{\uL}_{t,i})^{\alpha+1}}) \exp(-\sum_{i\ne i^*}\frac{2^\alpha+1}{(z+\eta_t \hat{\uL}_{t,i})^{\alpha}}) \dd z \\
        &=  \frac{\Delta}{2^\alpha +1} \qty(1-\exp(-\sum_{i\ne i^*}\frac{2^\alpha+1}{(\eta_t \hat{\uL}'+\eta_t \hat{\uL}_{t,i})^{\alpha}}) \dd z ) \\
        &\geq \frac{\Delta}{2^\alpha +1} \qty(1-\exp(-\sum_{i\ne i^*}\frac{2^\alpha+1}{2^\alpha (\eta_t \hat{\uL}_{t,i})^{\alpha}}) \dd z ) \\
        &\geq \frac{\Delta}{2^\alpha +1} \qty(1-e^{-\frac{2^\alpha+1}{2^\alpha}}) \\
        &\geq \frac{\Delta}{2^\alpha +1} \frac{2^\alpha+1}{2^{\alpha+1}+1} = \frac{\Delta}{2^{\alpha+1}+1},
    \end{align*}
    where the last inequality follows from $\frac{x}{1+x} < 1-e^{-x}$ for $x> -1$.
\end{proof}

\begin{lemma}\label{lem: lb bounded}
    Let $\Delta := \min_{i\ne i^*} \Delta_i$.
    Then, for any $\alpha >1$ and $\gD_\alpha \in \fD_\alpha$, there exists some distribution-dependent constants $c_{s,1}(\gD_\alpha), c_{s,2}(\gD_\alpha)\in (0,1)$ such that
\begin{enumerate}[label=(\roman*)]
    \item On $D_t$, $\sum_{i\ne i^*}\Delta_i w_{t,i} \geq c_{s,1}(\gD_\alpha)\sum_{i\ne i^*} \frac{\Delta_i}{(\eta_t \hat{\uL}_{t,i})^\alpha}$ and $w_{t,i^*}\geq 0.14$.
    \item On $D_t^c$, $\sum_{i\ne i^*}\Delta_i w_{t,i} \geq c_{s,2}(\gD_\alpha) \Delta$.
\end{enumerate}
\end{lemma}
\begin{proof}
    Here, for any $\hat{L}_t$, we have
    \begin{align*}
        \sum_{i\ne i^*} \Delta_i w_{t,i} &= \int_1^\infty \qty(\sum_{i\ne i^*} \Delta_i f(z+\eta_t \hat{\uL}_{t,i}) \prod_{j \ne i} F(z+\eta_t \hat{\uL}_{t,j})) \dd z \\
        &\geq  \int_1^\infty \qty(\sum_{i\ne i^*} \Delta_i f(z+\eta_t \hat{\uL}_{t,i})) \prod_{j \in [K]} F(z+\eta_t \hat{\uL}_{t,j}) \dd z  \\
        &\geq \int_1^\infty \qty(\sum_{i\ne i^*} \Delta_i f(z+\eta_t \hat{\uL}_{t,i})) \exp(-\sum_{j \in [K]} \frac{1-F(z+\eta_t \hat{\uL}_{t,i})}{F(z+\eta_t \hat{\uL}_{t,i})}) \dd z \numberthis{\label{eq: Par sto e}} \\
        &\geq \int_1^\infty \qty(\sum_{i\ne i^*} \Delta_i f(z+\eta_t \hat{\uL}_{t,i})) \exp(-\sum_{j \ne i^*} \frac{1-F(z+\eta_t \hat{\uL}_{t,i})}{F(z+\eta_t \hat{\uL}_{t,i})}) \exp(-\frac{1-F(z)}{F(z)}) \dd z
    \end{align*}
    where (\ref{eq: Par sto e}) holds since $e^{-\frac{x}{1-x}}< 1-x$ holds for $x <1$.

    (i) When $D_{t}$ holds, we obtain
    \begin{align*}
        \int_1^\infty \qty(\sum_{i\ne i^*} \Delta_i f(z+\eta_t \hat{\uL}_{t,i})) \exp(-\sum_{j \ne i^*} \frac{1-F(z+\eta_t \hat{\uL}_{t,i})}{F(z+\eta_t \hat{\uL}_{t,i})}) \exp(-\frac{1-F(z)}{F(z)}) \dd z 
        \hspace{-25em}&\\
        &\geq  e^{-1} \int_{U(2)}^\infty \qty(\sum_{i\ne i^*} \Delta_i f(z+\eta_t \hat{\uL}_{t,i})) \exp(-2\sum_{j \ne i^*} (1-F(z+\eta_t \hat{\uL}_{t,i}))) \dd z \\
        &\geq e^{-1} \int_{U(2)}^\infty \qty(\sum_{i\ne i^*} \Delta_i f(z+\eta_t \hat{\uL}_{t,i})) \dd z \\
        &= e^{-1} \sum_{i\ne i^*} \Delta_i\qty(1-F\qty(U(2)+\eta_t \hat{\uL}_{t,i})) \\
        &=  e^{-1} \sum_{i\ne i^*}\Delta_i \frac{S_F\qty(U(2)+\eta_t \hat{\uL}_{t,i})}{\qty(U(2)+\eta_t \hat{\uL}_{t,i})^\alpha} \\
        &\geq e^{-1} \sum_{i\ne i^*} \Delta_i \frac{A_l^{\alpha}}{\qty(U(2)+\eta_t \hat{\uL}_{t,i})^\alpha} \\
        &\geq e^{-1} \frac{A_l^{\alpha}}{(U(2)+1)^{\alpha}} \sum_{i\ne i^*} \frac{ \Delta_i}{(\eta_t \hat{\uL}_{t,i})^\alpha} = c_{s,1}(\gD_\alpha)\frac{ \Delta_i}{(\eta_t \hat{\uL}_{t,i})^\alpha},
    \end{align*}
    where the last inequality holds for $\eta_t \hat{\uL}_{t,j} \geq 1$ holds for $j\ne i^*$ on $D_{t}$.
    When $S_F$ is increasing, one can replace $A_l^{\alpha}$ with $S_F(U(2))$, where $c_{s,1}(\gD_\alpha) \approx \frac{e^{-1}}{2}$ holds.
    Note that one can replace $U(2)$ with $U(b)$ for any $b > 1$ and choose
    \begin{equation*}
         c_{s,1}(\gD_\alpha) = \min_{b >1} e^{1-b} \frac{S_F(U(b))}{(U(b)+1)^\alpha} \in (0,1),
    \end{equation*}
    which will provide a tighter lower bound.

    Since $\hat{\uL}_{t,i^*} = 0$ holds on $D_t$, we have
    \begin{align*}
        w_{t,i^*} &\geq e^{-1} \int_{U(2)}^\infty f(z) \exp(-\sum_{i \ne i^*} 1-F(z +\eta_t \hat{\uL}_{t,i})) \dd z  \\
        &\geq e^{-1}  \int_{U(2)}^\infty f(z) \exp(-\sum_{i\ne i^*}  1-F(z +\eta_t \hat{\uL}_{t,i}) - (1-F(z))) \dd z \\
        &\geq e^{-1} \int_1^\infty f(z) \exp(F(z)-1) \exp(F(U(2)+1)-1) \dd z \\
        &\geq e^{-1} \int_1^\infty f(z) \exp(F(z)-1) \exp(F(U(2))-1) \dd z\\
        &= e^{-\frac{3}{2}}(1-e^{-1}) \geq 0.14
    \end{align*}
    which concludes the proof of the case (i).
        
    (ii) Recall the definition of the tail quantile function $U(x)$ defined in (\ref{def: quantile}).
    Then, we have
    \begin{align*}
        \int_1^\infty &\qty(\sum_{i\ne i^*} \Delta_i f(z+\eta_t \hat{\uL}_{t,i})) \exp(-\sum_{j \ne i^*} \frac{1-F(z+\eta_t \hat{\uL}_{t,i})}{F(z+\eta_t \hat{\uL}_{t,i})}) \exp(-\frac{1-F(z)}{F(z)}) \dd z \\
        &\geq \Delta \int_{U(2) }^\infty \qty(\sum_{i\ne i^*}  f(z+\eta_t \hat{\uL}_{t,i})) \exp(-\sum_{j \ne i^*} \frac{1-F(z+\eta_t \hat{\uL}_{t,i})}{F(z+\eta_t \hat{\uL}_{t,i})}) \exp(-\frac{1-F(z)}{F(z)}) \dd z  \\
        &\geq  \Delta e^{-1} \int_{U(2) }^\infty \qty(\sum_{i\ne i^*} f(z+\eta_t \hat{\uL}_{t,i})) \exp(-2\sum_{j \ne i^*} (1-F(z+\eta_t \hat{\uL}_{t,i}))) \dd z \numberthis{\label{eq: lb gen 1}} \\
        &= \Delta \frac{e^{-1}}{2} \qty(1- \exp(-2\sum_{j\ne i^*} (1-F(U(2) +\eta_t \hat{\uL}_{t,j})))) \\
        &\geq \Delta \frac{e^{-1}}{2} \qty(1- \exp(-2(1-F(U(2)+1)))) = c_s(\gD_\alpha) \Delta
    \end{align*}
    where (\ref{eq: lb gen 1}) holds since $e^{-\frac{1-x}{x}}$ is increasing with respect to  $x \in [0,1]$, and $z \geq U(b)$ and $F(z) \geq \frac{1}{b}$ for $z \geq B$.
    Note that $c_s(\gD_\alpha) \in (0,1)$ is a distribution-dependent constant and can be approximated as $\frac{e^{-1}}{2}(1-e^{-1})$.
\end{proof}

\subsection{Regret for the optimal arm}\label{app: regret optimal}
To apply the self-bounding technique to FTPL, it is necessary to represent the regret associated with the optimal arm in terms of statistics of the other arms. 
We begin by extending the findings of \citet{pmlr-v201-honda23a} to \Fr distributions with an index $\alpha >1$ and subsequently generalize it to $\fD_\alpha$.
Before diving into the proofs, we first introduce the lemma by \citet{pmlr-v201-honda23a}.
\begin{lemma}[Partial result of Lemma 11 in \citet{pmlr-v201-honda23a}]\label{lem: hnd stoc}
    For any $\hat{L}_t$ and $\zeta \in (0,1)$, it holds that
    \begin{equation*}
        \E\qty[\I \qty[\hat{\ell}_{t,i^*} > \frac{\zeta}{\eta_t}] \hat{\ell}_{t,i^*} \eval \hat{L}_t] \leq \frac{1}{1-e^{-1}} (1-e^{-1})^{\frac{\zeta}{\eta_t}}\qty(\frac{\zeta}{\eta_t} + e)
    \end{equation*}
    and when $\eta_t = \frac{cK^{\frac{1}{\alpha}- \frac{1}{2}}}{\sqrt{t}}$
    \begin{equation*}
        \sum_{t=1}^\infty \frac{1}{1-e^{-1}} (1-e^{-1})^{\frac{\zeta}{\eta_t}}\qty(\frac{\zeta}{\eta_t} + e) \leq \mathcal{O}\qty(c^2 K^{\frac{2}{\alpha}-1}).
    \end{equation*}
\end{lemma}

\begin{lemma}\label{lem: sto opt fr}
    On $F_t$, for any $\zeta \in (0,1)$ and $\alpha >1$, we have
    \begin{multline*}
        \E\qty[ \hat{\ell}_{t,i^*} \qty( \phi_{i^*}(\eta_t \hat{L}_t; \gF_\alpha) - \phi_{i^*}(\eta_t (\hat{L}_t+\hat{\ell}_t); \gF_\alpha)) \eval \hat{L}_t] \\ 
        \leq \frac{2\alpha e}{(1-\zeta)^{\alpha+1}} \sum_{j\ne i^*} \frac{1}{\hat{\uL}_{t,j}}+ \frac{1}{1-e^{-1}}(1-e^{-1})^{\frac{\zeta}{\eta_t}}\qty(\frac{\zeta}{\eta_t} + e).
    \end{multline*}
\end{lemma}
\begin{proof}
Recall (\ref{eq: key property}), which shows that any $j \ne i^*$ satisfies $\hat{\uL}_{t,j} \geq \frac{1}{\eta_t}$ and $\argmin_{i \in [K]}\hat{L}_t = \hat{L}_{t,i^*}$ holds on $F_t$.
Following \citet{pmlr-v201-honda23a}, we consider the cases (a) $\widehat{w_{t,i^*}^{-1}} \leq \frac{\zeta}{\eta_t}$ and (b) $\widehat{w_{t,i^*}^{-1}} > \frac{\zeta}{\eta_t}$, separately.

(a) Let us consider the first case, where $\argmin_{i \in [K]}\hat{L}_t + x e_{i^*} = i^*$ holds since $\hat{\uL}_{t} \geq \frac{1}{\eta_t}$ and 
\begin{equation*}
    \hat{\ell}_{t,i^*} = \ell_{t,i^*}\widehat{w_{t,i^*}^{-1}} \leq  \frac{\zeta}{\eta_t}< \frac{1}{\eta_t} \leq \min_{i\ne i^*} \hat{\uL}_{t,i}.
\end{equation*}
Therefore, we have for $x \leq \frac{\zeta}{\eta_t}$
\begin{align*}
    \phi_{i^*}(\eta_t (\hat{L}_t + x e_{i^*})) = \int_0^\infty \frac{\alpha}{z^{\alpha+1}} \exp(-\sum_{i\in [K]} \frac{1}{(z+\eta_t(\hat{\uL}_{t,j}-x))^{\alpha}}) \dd z,
\end{align*}
which implies 
\begin{align*}
    \dv{x}\phi_{i^*}(\eta_t(\hat{\uL}_t + xe_{i^*})) &= \int_0^\infty -\frac{\alpha}{z^{\alpha+1}} \sum_{j\ne i^*}\frac{\alpha \eta_t}{(z+\eta_t(\hat{\uL}_{t,j}-x))^{\alpha+1}} \exp(-\sum_{j \ne i^*} \frac{1}{(z+\eta_t(\hat{\uL}_{t,j}-x))^{\alpha}} - \frac{1}{z^\alpha}) \dd z \\
    &\geq \int_0^\infty -\frac{\alpha}{z^{\alpha+1}} \sum_{j\ne i^*}\frac{\alpha \eta_t}{(z+\eta_t(\hat{\uL}_{t,j}-x))^{\alpha+1}} \exp(-\sum_{j \ne i^*} \frac{1}{(z+\eta_t(\hat{\uL}_{t,j}-x))^{\alpha}} - \frac{1}{z^\alpha}) \dd z 
\end{align*}
Then, we obtain
\begin{align*}
    \hat{\ell}_{t,i^*} &\qty( \phi_{i^*}(\eta_t \hat{L}_t) - \phi_{i^*}(\eta_t (\hat{L}_t+\hat{\ell}_t)))  \\ 
    &= \hat{\ell}_{t,i^*} \int_0^{\hat{\ell}_t} -\dv{x}\phi_{i^*}(\eta_t(\hat{\uL}_t + xe_{i^*})) \dd x \\
    &\leq \hat{\ell}_{t,i^*} \int_0^{\hat{\ell}_t} \int_0^\infty \frac{\alpha}{z^{\alpha+1}} \sum_{j\ne i^*}\frac{\alpha \eta_t}{(z+\eta_t(\hat{\uL}_{t,j}-x))^{\alpha+1}} \exp(-\sum_{j \ne i^*} \frac{1}{(z+\eta_t(\hat{\uL}_{t,j}-x))^{\alpha}} - \frac{1}{z^\alpha}) \dd z \dd x \\
    &\leq \hat{\ell}_{t,i^*} \int_0^{\hat{\ell}_t} \int_0^\infty \frac{\alpha}{z^{\alpha+1}} \sum_{j\ne i^*}\frac{\alpha \eta_t}{(z+\eta_t(\hat{\uL}_{t,j}-x))^{\alpha+1}} \exp(- \frac{1}{z^\alpha}) \dd z \dd x \\
    &\leq \hat{\ell}_{t,i^*} \int_0^{\hat{\ell}_t} \int_0^\infty \frac{\alpha}{z^{\alpha+1}} \sum_{j\ne i^*} \frac{1}{(1-\zeta)^{\alpha+1}}\frac{\alpha \eta_t}{(\eta_t\hat{\uL}_{t,j})^{\alpha+1}} \exp(- \frac{1}{z^\alpha}) \dd z \dd x \qquad (\text{by } x \leq \zeta/\eta_t, \text{ and } \hat{L}_{t,i} \geq 1/\eta_t) \\
    &= \hat{\ell}_{t,i^*} \int_0^{\hat{\ell}_t} \sum_{j\ne i^*} \frac{1}{(1-\zeta)^{\alpha+1}}\frac{\alpha \eta_t}{(\eta_t\hat{\uL}_{t,j})^{\alpha+1}} \dd x \\
    &= \hat{\ell}_{t,i^*}^2 \sum_{j\ne i^*} \frac{1}{(1-\zeta)^{\alpha+1}}\frac{\alpha \eta_t}{(\eta_t\hat{\uL}_{t,j})^{\alpha+1}} \\
    &\leq \hat{\ell}_{t,i^*}^2 \sum_{j\ne i^*} \frac{1}{(1-\zeta)^{\alpha+1}}\frac{\alpha}{\hat{\uL}_{t,j}},
\end{align*}
where the last inequality comes from $\hat{\uL}_{t,i} \geq \frac{1}{\eta_t}$.
Therefore, we have 
\begin{align*}
    \E\qty[ \I[\hat{\ell}_{t,i^*} \leq \zeta / \eta_t] \hat{\ell}_{t,i^*} \qty( \phi_{i^*}(\eta_t \hat{L}_t) - \phi_{i^*}(\eta_t (\hat{L}_t+\hat{\ell}_t))) \eval \hat{L}_t] \hspace{-12em}& \\
    &\leq \E\qty[ \I[\hat{\ell}_{t,i^*} \leq \zeta / \eta_t] \hat{\ell}_{t,i^*}^2 \sum_{j\ne i^*} \frac{\alpha}{(1-\zeta)^{\alpha+1}\hat{\uL}_{t,j}} \eval \hat{L}_t] \\
    &\leq \E\qty[ \frac{2\ell_{t,i^*}^2}{w_{t,i^*}} \sum_{j\ne i^*} \frac{\alpha}{(1-\zeta)^{\alpha+1}\hat{\uL}_{t,j}} \eval \hat{L}_t] \\
    &\leq 2\alpha e \sum_{j\ne i^*} \frac{1}{(1-\zeta)^{\alpha+1}\hat{\uL}_{t,j}} .\numberthis{\label{eq: case a}}
\end{align*}
(b) When $\widehat{w_{t,i^*}^{-1}} > \frac{\zeta}{\eta_t}$, by Lemma~\ref{lem: hnd stoc}, we have
\begin{align*}
    \E\qty[ \I[\hat{\ell}_{t,i^*} > \zeta / \eta_t] \hat{\ell}_{t,i^*} \qty( \phi_{i^*}(\eta_t \hat{L}_t) - \phi_{i^*}(\eta_t (\hat{L}_t+\hat{\ell}_t))) \eval \hat{L}_t] &\leq \E\qty[ \I[\hat{\ell}_{t,i^*} > \zeta / \eta_t] \hat{\ell}_{t,i^*} \eval \hat{L}_t] \\
    &\leq \frac{1}{1-e^{-1}} (1-e^{-1})^{\frac{\zeta}{\eta_t}}\qty(\frac{\zeta}{\eta_t} + e). \numberthis{\label{eq: case b}}
\end{align*}
Combining (\ref{eq: case a}) and (\ref{eq: case b}) concludes the proof.
\end{proof}

\begin{lemma}\label{lem: sto opt gen}
    On $D_t$, for any $\zeta \in (0,1)$, $\gD_\alpha \in \fD_{\alpha}$ and $\alpha >1$, we have
    \begin{multline*}
        \E\qty[ \hat{\ell}_{t,i^*} \qty( \phi_{i^*}(\eta_t \hat{L}_t; \gD_\alpha) - \phi_{i^*}(\eta_t (\hat{L}_t+\hat{\ell}_t); \gD_\alpha)) \eval \hat{L}_t] \\
        \leq \frac{14.4 A_u^{\alpha}\rho_1 e (1-e^{-1}) }{(1-\zeta)^{\alpha+1}} \sum_{j\ne i^*} \frac{1}{\hat{\uL}_{t,j}}+ \frac{1}{1-e^{-1}}(1-e^{-1})^{\frac{\zeta}{\eta_t}}\qty(\frac{\zeta}{\eta_t} + e).
    \end{multline*}
\end{lemma}
\begin{proof}
    As the proof of Lemma~\ref{lem: sto opt fr}, we consider two cases (a) $\widehat{w_{t,i^*}^{-1}} \leq \frac{\zeta}{\eta_t}$ and (b) $\widehat{w_{t,i^*}^{-1}} > \frac{\zeta}{\eta_t}$, separately.
    For case (b), one can see that Lemma~\ref{lem: hnd stoc} can be directly applied as Lemma~\ref{lem: sto opt fr}.

    (a) When $\widehat{w_{t,i^*}^{-1}} \leq \frac{\zeta}{\eta_t}$, we have
    \begin{align*}
    \phi_{i^*}(\eta_t (\hat{L}_t + e_{i^*}x); \gD_\alpha) = \int_1^\infty f(z) \prod_{j \ne i^*} F\qty(z+\eta_t (\hat{\uL}_{t,j} -x)) \dd z,
\end{align*}
    which implies for $x \leq \frac{\zeta}{\eta_t}$,
\begin{align*}
    -&\dv{x} \phi_{i^*}\qty(\eta_t (\hat{L}_t + e_{i^*}x); \gD_\alpha)\\
    &= \int_1^\infty f(z) \sum_{i\ne i^*} \qty(\eta_t f\qty(z+\eta_t (\hat{\uL}_{t,j} -x)) \prod_{j \ne i, i^*} F\qty(z+\eta_t (\hat{\uL}_{t,j} -x))) \dd z \\
    &\leq \int_1^\infty f(z) \sum_{i\ne i^*} \qty(\eta_t f\qty(z+\eta_t (\hat{\uL}_{t,i} -x)) \exp(-\sum_{j \ne i, i^*} \qty(1-F\qty(z+\eta_t (\hat{\uL}_{t,j} -x))))) \dd z \\
    &\leq e^2 \int_1^\infty f(z) \sum_{i\ne i^*} \eta_t f\qty(z+\eta_t (\hat{\uL}_{t,i} -x)) \exp(-\sum_{j \ne  i^*} \qty(1-F\qty(z+\eta_t (\hat{\uL}_{t,j} -x))) - (1-F(z))) \dd z \\
    &\leq  e^2 \int_1^\infty f(z) \sum_{i\ne i^*} \eta_t f\qty(z+\eta_t (\hat{\uL}_{t,i} -x)) \exp(-(1-F(z))) \dd z \\
    &= e^2 \int_1^\infty f(z) \sum_{i\ne i^*} \eta_t \frac{S_F(z+\eta_t (\hat{\uL}_{t,i} -x))\varrho(z+\eta_t (\hat{\uL}_{t,i} -x))}{(z+\eta_t (\hat{\uL}_{t,i} -x))^{\alpha+1}} \exp(-(1-F(z))) \dd z \\
    &\leq e^2  \eta_t A_u^{\alpha} \rho_1  \int_1^\infty f(z) \exp(-(1-F(z))) \sum_{i\ne i^*} \frac{1}{(z+\eta_t (\hat{\uL}_{t,i} -x))^{\alpha+1}} \dd z \numberthis{\label{eq: sto gen bnd}}\\
    &\leq e^2  \eta_t A_u^{\alpha} \rho_1  \int_1^\infty f(z) \exp(-(1-F(z))) \sum_{i\ne i^*} \frac{1}{(1-\zeta)^{\alpha+1}(\eta_t\hat{\uL}_{t,i})^{\alpha+1}} \dd z \\
    &\leq A_u^{\alpha} \rho_1  e^2 \sum_{i\ne i^*} \frac{\eta_t}{(1-\zeta)^{\alpha+1}(\eta_t\hat{\uL}_{t,i})^{\alpha+1}} (1-e^{-1}) \\
    &\leq A_u^{\alpha} \rho_1  e^2 (1-e^{-1}) \sum_{i\ne i^*} \frac{1}{(1-\zeta)^{\alpha+1}\hat{\uL}_{t,i}} \qquad (\text{by } \eta_t \hat{\uL}_{t,i}\geq 1, \, \forall i \ne i^*),
\end{align*}
where (\ref{eq: sto gen bnd}) follows from the boundedness of $S_F \leq S$ and Assumption~\ref{asm: bounded hazard}.
Therefore, we have
\begin{multline*}
     \E\qty[\hat{\ell}_{t,i^*} \qty( \phi_{i^*}(\eta_t \hat{L}_t; \gD_\alpha) - \phi_{i^*}(\eta_t (\hat{L}_t+\hat{\ell}_t); \gD_\alpha)) \eval \hat{L}_t] \\
     \leq  \sum_{i\ne i^*} \frac{14.4 A_u^{\alpha} \rho_1  e (1-e^{-1})}{(1-\zeta)^{\alpha+1}\hat{\uL}_{t,i}} + \frac{1}{1-e^{-1}}(1-e^{-1})^{\frac{\zeta}{\eta_t}}\qty(\frac{\zeta}{\eta_t} + e).
\end{multline*}
Here, $14.4$ is introduced by $\frac{2}{0.14}$ by following the same steps in (\ref{eq: case a}).
\end{proof}

\subsection{Proof of Theorems~\ref{thm: sto_all} and~\ref{thm: sto_bad}}\label{app: thm stoc}
Although the overall proofs are identical for $\gF_\alpha$ and $\fD_{\alpha}$ in essential, we provide the proof of $\gF_\alpha$ first and then $\fD_\alpha$ for completeness.

\subsubsection{\Fr distribution with $\alpha \geq 2$}
For simplicity, let $K_\alpha = K^{\frac{1}{\alpha}-\frac{1}{2}}$ so that $\eta_t = \frac{cK_\alpha}{\sqrt{t}}$.
Combining the results obtained thus far, the regret is bounded by
\begin{align*}
    \gR(T) &\leq \sum_{t=1} \E \qty[\inp{\hat{\ell}_t}{w_t - w_{t+1}}] 
    + \sum_{t=1}^T \qty(\frac{1}{\eta_{t+1}} - \frac{1}{\eta_t}) \E \qty[r_{t+1, I_{t+1}}- r_{t+1, i^*}] \\
    & \hspace{5em}+  \frac{K^{\frac{1}{\alpha}}\Gamma\qty(1-\frac{1}{\alpha})}{\eta_1}  + \frac{\alpha}{2} \log(T+1) \tag*{(by Lemmas~\ref{lem: gen_lem3} and~\ref{lem: stab decom})}\\
    &\leq \sum_{t=1}^T \E\qty[\E\qty[\inp{\hat{\ell}_t}{w_t - w_{t+1}} + \qty(\frac{1}{\eta_{t+1}} - \frac{1}{\eta_t})(r_{t+1, I_{t+1}}- r_{t+1, i^*}) \eval \hat{L}_t]] \\
    & \hspace{5em}+ \frac{\sqrt{K} \Gamma\qty(1-\frac{1}{\alpha})}{c} + \frac{\alpha}{2} \log(T+1) \\
    &\leq \sum_{t=1}^T \E\qty[\E\qty[\inp{\hat{\ell}_t}{w_t - w_{t+1}} + \frac{r_{t+1, I_{t+1}}- r_{t+1, i^*}}{2cK_\alpha \sqrt{t} } \eval \hat{L}_t]] + \frac{\sqrt{K} \Gamma\qty(1-\frac{1}{\alpha})}{c} + \frac{\alpha}{2} \log(T+1), \numberthis{\label{eq: stoc Fr term all}}
\end{align*}
where the last inequality follows from
\begin{align*}
    \frac{1}{\eta_{t+1}} - \frac{1}{\eta_t} = \frac{1}{cK_\alpha } (\sqrt{t+1} - \sqrt{t}) = \frac{\sqrt{t}}{cK_\alpha}(\sqrt{1+1/t}-1) \leq \frac{1}{2cK_\alpha\sqrt{t}}.
\end{align*}
Note that $w_t = \phi(\eta_t \hat{L}_t)$ and $w_{t+1} = \phi(\eta_t (\hat{L}_t + \ell_t))$ by definition of $\phi$.

On $F_t$, where $\eta_t \hat{\uL}_{t,j}\geq 1$ for $j\ne i^*$, we have for $\alpha \geq 2$
\begin{align*}
    \E\qty[\inp{\hat{\ell}_t}{w_t - w_{t+1}} + \frac{r_{t+1, I_{t+1}}- r_{t+1, i^*}}{2cK_\alpha \sqrt{t} } \eval \hat{L}_t] \hspace{-13em}&\\
    &\leq \sum_{i\ne i^*} \frac{2\alpha}{\hat{\uL}_{t,i}} + \frac{1}{2cK_\alpha\sqrt{t}} \frac{\alpha}{\alpha-1} \frac{1}{(\eta_t \hat{\uL}_{t,i})^{\alpha-1}} + \frac{2\alpha e}{(1-\zeta)^{\alpha+1}} \frac{1}{\hat{\uL}_{t,i}} + \sum_{t=1}^T \frac{(1-e^{-1})^{\frac{\zeta}{\eta_t}}}{1-e^{-1}} \qty(\frac{\zeta}{\eta_t}+e) \tag*{(by Lemmas~\ref{lem: penalty}, \ref{lem: stability_adv_i}, and~\ref{lem: sto opt fr})}\\
    &\leq \sum_{i\ne i^*} \frac{2\alpha}{\hat{\uL}_{t,i}} + \frac{1}{2cK_\alpha\sqrt{t}} \frac{\alpha}{\alpha-1} \frac{1}{(\eta_t \hat{\uL}_{t,i})} + \frac{2\alpha e}{(1-\zeta)^{\alpha+1}} \frac{1}{\hat{\uL}_{t,i}} + \sum_{t=1}^T \frac{(1-e^{-1})^{\frac{\zeta}{\eta_t}}}{1-e^{-1}} \qty(\frac{\zeta}{\eta_t}+e) \numberthis{\label{eq: Ft sto all 1}} \\
    &\leq \sum_{i\ne i^*} \frac{2\alpha}{\hat{\uL}_{t,i}} + \frac{1}{2(cK_\alpha)^2} \frac{\alpha}{\alpha-1} \frac{1}{\hat{\uL}_{t,i}} + \frac{2\alpha e}{(1-\zeta)^{\alpha+1}} \frac{1}{\hat{\uL}_{t,i}} + \sum_{t=1}^T \frac{(1-e^{-1})^{\frac{\zeta}{\eta_t}}}{1-e^{-1}} \qty(\frac{\zeta}{\eta_t}+e)  \\
    &\leq \sum_{i\ne i^*} \frac{2\alpha + \frac{2\alpha e}{(1-\zeta)^{\alpha}}+ \frac{\alpha}{2(cK_\alpha)^2 (\alpha-1)}}{\hat{L}_{t,i}} + \mathcal{O}\qty(c^2 K_\alpha^2) \tag*{(by Lemma~\ref{lem: hnd stoc})}  \\
    &=\sum_{i\ne i^*} \frac{2\alpha(1+e^2)+ \frac{\alpha}{2(cK_\alpha)^2 (\alpha-1)}}{\hat{L}_{t,i}} + \mathcal{O}\qty(c^2 K_\alpha^2)  \numberthis{\label{eq: stoc Fr term 1}}
\end{align*}
where (\ref{eq: Ft sto all 1}) follows from $\eta_t \hat{\uL}_{t,i} \geq 1$ for all $i \ne i^*$ on $F_t$ and $\alpha \geq 2$ and we chose $\zeta = 1-e^{-\frac{1}{\alpha}} \in (0,1)$ for simplicity.

On $F_t^c$, we have
\begin{align*}
    \E\qty[\inp{\hat{\ell}_t}{w_t - w_{t+1}} + \frac{r_{t+1, I_{t+1}}- r_{t+1, i^*}}{2cK_\alpha \sqrt{t} } \eval \hat{L}_t] \hspace{-10em}&\\
    &\leq \frac{2\alpha(\alpha+1)}{\alpha-1}\Gamma\qty(1+\frac{1}{\alpha})K^{1-\frac{1}{\alpha}} \eta_t + \frac{K^{\frac{1}{\alpha}}}{2cK_\alpha \sqrt{t}} \frac{\alpha^2(2\alpha + e-2)}{(\alpha-1)(2\alpha-1)e} \tag*{(by Lemmas~\ref{lem: penalty} and~\ref{lem: stability})} \\
    &= \qty(\frac{2c\alpha(\alpha+1)}{\alpha-1}\Gamma\qty(1+\frac{1}{\alpha}) + \frac{\alpha^2(2\alpha + e-2)}{2ce(\alpha-1)(2\alpha-1)} )\sqrt{\frac{K}{t}}. \numberthis{\label{eq: stoc Fr term 2}}
\end{align*}
Combining (\ref{eq: stoc Fr term 1}) and (\ref{eq: stoc Fr term 2}) with (\ref{eq: stoc Fr term all}) provides
\begin{align*}
    \gR(T)& \leq \sum_{t=1}^T \E\Bigg[\I[F_t] \sum_{i\ne i^*}\frac{2\alpha(1+e^2)+ \frac{\alpha}{2(cK_\alpha)^2 (\alpha-1)}}{\hat{L}_{t,i}} \\
    &\hspace{7em}+ \I[F_t^c] \qty(\frac{2c\alpha(\alpha+1)}{\alpha-1}\Gamma\qty(1+\frac{1}{\alpha}) + \frac{\alpha^2(2\alpha + e-2)}{2ce(\alpha-1)(2\alpha-1)} )\sqrt{\frac{K}{t}} \Bigg] \\
    &\hspace{15em}+ \frac{\sqrt{K} \Gamma\qty(1-\frac{1}{\alpha})}{c} + \frac{\alpha}{2} \log(T+1)+ \mathcal{O}\qty(c^2 K_\alpha^2). \numberthis{\label{eq: stoc Fr term after}}
\end{align*}
On the other hand, by Lemma~\ref{lem: lb frechet}, we have
\begin{equation}\label{eq: stoc Fr lb term}
    \gR(T) \geq \sum_{t=1}^T \E\qty[\I[F_t] c_{s,1}(\gF_\alpha) \frac{\Delta_i t^{\frac{\alpha}{2}}}{(cK_\alpha \hat{\uL}_{t,i})^\alpha} + \I[F_t^c] \frac{\Delta}{2^{\alpha+1}+1}].
\end{equation}
By applying self-bounding technique, (\ref{eq: stoc Fr term after}) - (\ref{eq: stoc Fr lb term})$/2$, we have
\begin{align*}
    \frac{\gR(T)}{2} &\leq \sum_{t=1}^T \E\qty[\I[F_t] \sum_{i \ne i^*} \qty(\frac{2\alpha(1+e^2)+ \frac{\alpha}{2(cK_\alpha)^2 (\alpha-1)}}{\hat{L}_{t,i}} - c_{s,1}(\gF_\alpha) \frac{\Delta_i t^{\frac{\alpha}{2}}}{2(cK_\alpha \hat{\uL}_{t,i})^\alpha})] \\
    &\hspace{1em}+ \sum_{t=1}^T\E\qty[\I[F_t^c] \qty(\qty(\frac{2c\alpha(\alpha+1)}{\alpha-1}\Gamma\qty(1+\frac{1}{\alpha}) + \frac{\alpha^2(2\alpha + e-2)}{2ce(\alpha-1)(2\alpha-1)} )\sqrt{\frac{K}{t}} - \frac{\Delta}{2^{\alpha+1}+1}) ] \\
    &\hspace{4em} + \frac{\sqrt{K} \Gamma\qty(1-\frac{1}{\alpha})}{c} + \frac{\alpha}{2} \log(T+1)+ \mathcal{O}\qty(c^2 K_\alpha^2). \numberthis{\label{eq: stoc Fr last term}}
\end{align*}
For the first term of (\ref{eq: stoc Fr last term}), we have
\begin{align*}
    \qty(\frac{2\alpha(1+e^2)+ \frac{\alpha}{2(cK_\alpha)^2 (\alpha-1)}}{\hat{L}_{t,i}} - c_{s,1}(\gF_\alpha) \frac{\Delta_i t^{\frac{\alpha}{2}}}{2(cK_\alpha \hat{\uL}_{t,i})^\alpha}) \hspace{-20em}& \\
    &\leq \qty(2\alpha(1+e^2)+ \frac{\alpha}{2(cK_\alpha)^2 (\alpha-1)})\frac{\alpha-1}{\alpha} \qty(\frac{4\alpha(1+e^2)+ \frac{\alpha}{(cK_\alpha)^2 (\alpha-1)}}{\alpha c_{s,1}(\gF_\alpha) \Delta_i})^{\frac{1}{\alpha-1}} \qty(\frac{cK_\alpha}{\sqrt{t}})^{\frac{\alpha}{\alpha-1}} \\
    &= \qty(4(\alpha-1) + \frac{1}{2(cK_\alpha)^2}) \qty(\frac{4\alpha(1+e^2)+ \frac{\alpha}{(cK_\alpha)^2 (\alpha-1)}}{\alpha c_{s,1}(\gF_\alpha) \Delta_i})^{\frac{1}{\alpha-1}} \frac{(cK_\alpha)^{\frac{\alpha}{\alpha-1}}}{t^{\frac{\alpha}{2(\alpha-1)}}} \\
    &= \mathcal{O}\qty(\frac{1}{\Delta_i^{\frac{1}{\alpha-1}}K^{\frac{\alpha-2}{2(\alpha-1)}}t^{\frac{\alpha}{2(\alpha-1)}}}), \numberthis{\label{eq: stoc Fr last term 1}}
\end{align*}
since $Ax - Bx^{\alpha} \leq A \frac{\alpha-1}{\alpha} \qty(\frac{A}{\alpha B})^{\frac{1}{\alpha-1}}$ holds for $A, B>0$ and $\alpha >1$.

For the second term of (\ref{eq: stoc Fr last term}), we have
\begin{align*}
  \sum_{t=1}^T \qty(\frac{2c\alpha(\alpha+1)}{\alpha-1}\Gamma\qty(1+\frac{1}{\alpha}) + \frac{\alpha^2(2\alpha + e-2)}{2ce(\alpha-1)(2\alpha-1)} )\sqrt{\frac{K}{t}} - \frac{\Delta}{2^{\alpha+1}+1} \hspace{-25em}&\\
  &\leq \sum_{t=1}^T \max\qty{\qty(\frac{2c\alpha(\alpha+1)}{\alpha-1}\Gamma\qty(1+\frac{1}{\alpha}) + \frac{\alpha^2(2\alpha + e-2)}{2ce(\alpha-1)(2\alpha-1)} )\sqrt{\frac{K}{t}}-\frac{\Delta}{2^{\alpha+1}+1} , 0} \\
  &\leq \qty(\frac{\frac{2c\alpha(\alpha+1)}{\alpha-1}\Gamma\qty(1+\frac{1}{\alpha}) + \frac{\alpha^2(2\alpha + e-2)}{2ce(\alpha-1)(2\alpha-1)}}{\frac{\Delta}{2^{\alpha+1}+1}})^2 K =\mathcal{O}(K) \numberthis{\label{eq: stoc Fr last term 2}}
\end{align*}
Therefore, by combining (\ref{eq: stoc Fr last term 1}) and (\ref{eq: stoc Fr last term 2}) with (\ref{eq: stoc Fr last term}), we obtain
\begin{align*}
    \frac{\gR(T)}{2} &\leq \mathcal{O}\qty( \sum_{i\ne i^*} \sum_{t=1}^T \frac{1}{\Delta_i^{\frac{1}{\alpha-1}}K^{\frac{\alpha-2}{2(\alpha-1)}}t^{\frac{\alpha}{2(\alpha-1)}}}) + \mathcal{O}(K) +  \frac{\sqrt{K} \Gamma\qty(1-\frac{1}{\alpha})}{c} + \frac{\alpha}{2} \log(T+1)+ \mathcal{O}\qty(c^2 K_\alpha^2) \\
    &\leq  \mathcal{O}(K) +  \frac{\sqrt{K} \Gamma\qty(1-\frac{1}{\alpha})}{c} + \frac{\alpha}{2} \log(T+1) \\
    &\hspace{6em}+ \mathcal{O}\qty(c^2 K_\alpha^2) + \begin{cases}
        \mathcal{O} \qty(\sum_{i\ne i^*} \frac{\log T}{\Delta_i}), &\text{if } \alpha =2 \\
        \mathcal{O}\qty(\sum_{i\ne i^*} \frac{1}{(\alpha-2)}\frac{T^{\frac{\alpha-2}{2(\alpha-1)}}}{\Delta_i^{\frac{1}{\alpha-1}}K^{\frac{\alpha-2}{2(\alpha-1)}}}), &\text{if } \alpha > 2,
    \end{cases}
\end{align*}
which concludes the proof for $\gF_\alpha$ with $\alpha \geq 2$.

\subsubsection{\Fr distribution with $\alpha \in (1,2)$}
The proof for $\alpha \in (1,2)$ begin by modifying (\ref{eq: Ft sto all 1}), where we obtain
\begin{align*}
   & \E\qty[\inp{\hat{\ell}_t}{w_t - w_{t+1}} + \frac{r_{t+1, I_{t+1}}- r_{t+1, i^*}}{2cK_\alpha \sqrt{t} } \eval \hat{L}_t] \\
    &\leq \sum_{i\ne i^*} \frac{2\alpha}{\hat{\uL}_{t,i}} + \frac{1}{2cK_\alpha\sqrt{t}} \frac{\alpha}{\alpha-1} \frac{1}{(\eta_t \hat{\uL}_{t,i})^{\alpha-1}} + \frac{2\alpha e}{(1-\zeta)^{\alpha+1}} \frac{1}{\hat{\uL}_{t,i}} + \sum_{t=1}^T \frac{(1-e^{-1})^{\frac{\zeta}{\eta_t}}}{1-e^{-1}} \qty(\frac{\zeta}{\eta_t}+e)  \\
    &\leq \sum_{i\ne i^*} \frac{2\alpha + \frac{2\alpha e}{(1-\zeta)^{\alpha}}}{\hat{L}_{t,i}} + \frac{1}{2cK_\alpha\sqrt{t}} \frac{\alpha}{\alpha-1} \frac{1}{(\eta_t \hat{\uL}_{t,i})^{\alpha-1}} + \mathcal{O}\qty(c^2 K_\alpha^2) \tag*{(by Lemma~\ref{lem: hnd stoc})}  \\
    &=  \sum_{i\ne i^*}  \frac{2\alpha + \frac{2\alpha e}{(1-\zeta)^{\alpha}}}{\hat{L}_{t,i}} + \frac{1}{2(cK_\alpha)^{\alpha}} \frac{\alpha}{\alpha-1} \frac{1}{t^{1-\frac{\alpha}{2}}(\hat{\uL}_{t,i})^{\alpha-1}} + \mathcal{O}\qty(c^2 K_\alpha^2). 
\end{align*}
By following the same steps from (\ref{eq: stoc Fr term 2}), one can obtain
\begin{align*}
    \frac{\gR(T)}{2} &\leq \sum_{t=1}^T \E\qty[\I[F_t] \sum_{i \ne i^*} \qty(\frac{2\alpha(1+e^2)}{\hat{L}_{t,i}} + \frac{1}{2(cK_\alpha)^{\alpha}} \frac{\alpha}{\alpha-1} \frac{1}{t^{1-\frac{\alpha}{2}}(\hat{\uL}_{t,i})^{\alpha-1}} - c_{s,1}(\gF_\alpha) \frac{\Delta_i t^{\frac{\alpha}{2}}}{2(cK_\alpha \hat{\uL}_{t,i})^\alpha})] \\
    &\hspace{1em}+ \sum_{t=1}^T\E\qty[\I[F_t^c] \qty(\qty(\frac{2c\alpha(\alpha+1)}{\alpha-1}\Gamma\qty(1+\frac{1}{\alpha}) + \frac{\alpha^2(2\alpha + e-2)}{2ce(\alpha-1)(2\alpha-1)} )\sqrt{\frac{K}{t}} - \frac{\Delta}{2^{\alpha+1}+1}) ] \\
    &\hspace{4em} + \frac{\sqrt{K} \Gamma\qty(1-\frac{1}{\alpha})}{c} + \frac{\alpha}{2} \log(T+1)+ \mathcal{O}\qty(c^2 K_\alpha^2).
\end{align*}
Here, the first term can be written as
\begin{align*}
    &\frac{2\alpha(1+e^2)}{\hat{L}_{t,i}} + \frac{1}{2(cK_\alpha)^{\alpha}} \frac{\alpha}{\alpha-1} \frac{1}{t^{1-\frac{\alpha}{2}}(\hat{\uL}_{t,i})^{\alpha-1}} - c_{s,1}(\gF_\alpha) \frac{\Delta_i t^{\frac{\alpha}{2}}}{2(cK_\alpha \hat{\uL}_{t,i})^\alpha} \\
    &= \qty(\frac{2\alpha(1+e^2)}{\hat{L}_{t,i}} - \frac{c_{s,1}(\gF_\alpha)}{2}\frac{\Delta_i t^{\frac{\alpha}{2}}}{(cK_\alpha \hat{\uL}_{t,i})^\alpha}) + \frac{1}{2(cK_\alpha)^\alpha}\qty(\frac{\alpha}{\alpha-1} \frac{1}{t^{1-\frac{\alpha}{2}}(\hat{\uL}_{t,i})^{\alpha-1}} -  c_{s,1}(\gF_\alpha)\frac{\Delta_i t^{\frac{\alpha}{2}}}{(\hat{\uL}_{t,i})^\alpha}) \numberthis{\label{eq: small alpha Fr}}
\end{align*}
The first term of (\ref{eq: small alpha Fr}) can be bounded in the same way of (\ref{eq: stoc Fr last term 1}) and the second term is bounded as
\begin{align*}
    \frac{\alpha}{\alpha-1} \frac{1}{t^{1-\frac{\alpha}{2}}(\hat{\uL}_{t,i})^{\alpha-1}} -  c_{s,1}(\gF_\alpha)\frac{\Delta_i t^{\frac{\alpha}{2}}}{(\hat{\uL}_{t,i})^\alpha} &\leq \frac{1}{\alpha-1}\frac{1}{t^{1-\frac{\alpha}{2}}} \qty(\frac{1}{\Delta_i t c_{s,1}(\gF_\alpha)})^{\alpha-1} \\
    &= \frac{1}{\alpha-1} \frac{1}{(\Delta_i c_{s,1}(\gF_\alpha))^{\alpha-1}} \frac{1}{t^{\frac{\alpha}{2}}},
\end{align*}
by $Ax^{\alpha-1}-Bx^{\alpha} \leq \frac{A}{\alpha}\qty(\frac{\alpha-1}{\alpha} \frac{A}{B})^{\alpha-1}$ for $A,B>0$ and $\alpha >1$.
Therefore, (\ref{eq: small alpha Fr}) is bounded by
\begin{align*}
    &\qty(\frac{2\alpha(1+e^2)}{\hat{L}_{t,i}} - \frac{c_{s,1}(\gF_\alpha)}{2}\frac{\Delta_i t^{\frac{\alpha}{2}}}{(cK_\alpha \hat{\uL}_{t,i})^\alpha}) + \frac{1}{2(cK_\alpha)^\alpha}\qty(\frac{\alpha}{\alpha-1} \frac{1}{t^{1-\frac{\alpha}{2}}(\hat{\uL}_{t,i})^{\alpha-1}} -  c_{s,1}(\gF_\alpha)\frac{\Delta_i t^{\frac{\alpha}{2}}}{(\hat{\uL}_{t,i})^\alpha}) \\
    &\hspace{3em}\leq \qty(4(\alpha-1)) \qty(\frac{2\alpha(1+e^2)}{\alpha c_{s,1}(\gF_\alpha) \Delta_i})^{\frac{1}{\alpha-1}} \frac{(cK_\alpha)^{\frac{\alpha}{\alpha-1}}}{t^{\frac{\alpha}{2(\alpha-1)}}} + \frac{1}{2(cK_\alpha)^\alpha} \frac{1}{\alpha-1} \frac{1}{(\Delta_i c_{s,1}(\gF_\alpha))^{\alpha-1}} \frac{1}{t^{\frac{\alpha}{2}}}. \numberthis{\label{eq: stoc small Fr}}
\end{align*}
Since $\frac{\alpha}{2(\alpha-1)} > 1$, the summation over the first term in (\ref{eq: stoc small Fr}) is constant.
Therefore, by following the same steps from (\ref{eq: stoc Fr last term 2}), we can obtain for $\alpha \in (1,2)$ that
\begin{equation*}
    \gR(T) \leq \mathcal{O}\qty(\sum_{i\ne i^*} \frac{1}{2-\alpha}\frac{1}{c^\alpha K^{1-\frac{\alpha}{2}}} \frac{T^{1-\frac{\alpha}{2}}}{\Delta_i^{\alpha-1}}) + \mathcal{O}(K) +  \frac{\sqrt{K} \Gamma\qty(1-\frac{1}{\alpha})}{c} + \frac{\alpha}{2} \log(T+1),
\end{equation*}
which concludes the proof.

\subsubsection{Fr\'{e}chet-type distributions with bounded slowly varying function}
Let us begin by replacing terms in (\ref{eq: stoc Fr term all}) with the corresponding terms for $\fD_\alpha$, which gives
\begin{multline*}
    \gR(T) \leq \sum_{t=1}^T \E\qty[\E\qty[\inp{\hat{\ell}_t}{w_t - w_{t+1}} + \frac{r_{t+1, I_{t+1}}- r_{t+1, i^*}}{2cK_\alpha \sqrt{t} } \eval \hat{L}_t]] \\ + \frac{MA_u \sqrt{K}}{c} + \frac{\rho_1 (e^2+1)}{2} \log(T+1). \numberthis{\label{eq: stoc gen term all}}
\end{multline*}

On $D_t$, where $\eta_t \hat{\uL}_{t,j}\geq 1$ for $j\ne i^*$, we have for $\alpha \geq 2$
\begin{align*}
    \E\qty[\inp{\hat{\ell}_t}{w_t - w_{t+1}} + \frac{r_{t+1, I_{t+1}}- r_{t+1, i^*}}{2cK_\alpha \sqrt{t} } \eval \hat{L}_t] \hspace{-18em}&\\
    &\leq \sum_{i\ne i^*} \frac{\frac{2 e \alpha A_u \rho_2 }{A_l(\alpha+1)}}{\hat{\uL}_{t,i}} + \frac{1}{2cK_\alpha\sqrt{t}} \frac{e \rho_1 A_u^{\alpha}}{\alpha-1} \frac{1}{(\eta_t \hat{\uL}_{t,i})^{\alpha-1}} + \frac{ 14.4 A_u^{\alpha} \rho_1 e (1-e^{-1})}{(1-\zeta)^{\alpha+1}} \frac{1}{\hat{\uL}_{t,i}} \\
    & \hspace{10em}+ \sum_{t=1}^T \frac{(1-e^{-1})^{\frac{\zeta}{\eta_t}}}{1-e^{-1}} \qty(\frac{\zeta}{\eta_t}+e) \tag*{(by Lemmas~\ref{lem: stability_adv_i},~\ref{lem: penalty} and~\ref{lem: sto opt gen})}\\
    &\leq \sum_{i\ne i^*}\frac{\frac{2 e \alpha A_u \rho_2 }{A_l(\alpha+1)}}{\hat{\uL}_{t,i}} + \frac{1}{2cK_\alpha\sqrt{t}} \frac{e \rho_1 A_u^{\alpha}}{\alpha-1} \frac{1}{\eta_t \hat{\uL}_{t,i}}  + \frac{14.4 A_u^{\alpha} \rho_1 e (1-e^{-1})}{(1-\zeta)^{\alpha+1}} \frac{1}{\hat{\uL}_{t,i}} \\
    &\hspace{10em}+ \sum_{t=1}^T \frac{(1-e^{-1})^{\frac{\zeta}{\eta_t}}}{1-e^{-1}} \qty(\frac{\zeta}{\eta_t}+e) \numberthis{\label{eq: Dt sto all 1}} \\
    &\leq \sum_{i\ne i^*}\frac{\frac{2 e \alpha A_u \rho_2 }{A_l(\alpha+1)}}{\hat{\uL}_{t,i}} + \frac{1}{2(cK_\alpha)^2} \frac{e \rho_1 A_u^{\alpha}}{\alpha-1} \frac{1}{\hat{\uL}_{t,i}}  + \frac{14.4 A_u^{\alpha} \rho_1 e (1-e^{-1})}{(1-\zeta)^{\alpha+1}} \frac{1}{\hat{\uL}_{t,i}} + \mathcal{O}\qty(c^2 K_\alpha^2) \tag*{(by Lemma~\ref{lem: hnd stoc})}  \\
    &=\sum_{i\ne i^*} \frac{ \frac{2 e \alpha \rho_2  A_u}{A_l(\alpha+1)} S^{\frac{1}{\alpha}}+ 14.4 A_u^{\alpha} \rho_1 e^2 (1-e^{-1}) +\frac{e \rho_1  A_u^{\alpha}}{2(cK_\alpha)^2 (\alpha-1)}}{\hat{L}_{t,i}} + \mathcal{O}\qty(c^2 K_\alpha^2), \numberthis{\label{eq: stoc gen term 1}}
\end{align*}
where (\ref{eq: Dt sto all 1}) follows from $\eta_t \hat{\uL}_{t,i} \geq 1$ for all $i \ne i^*$ on $D_t$ and $\alpha \geq 2$ and we chose $\zeta = 1-e^{\frac{-1}{\alpha+1}} \in (0,1)$ in (\ref{eq: stoc gen term 1}) for simplicity.

On $D_t^c$, we have
\begin{align*}
    \E\qty[\inp{\hat{\ell}_t}{w_t - w_{t+1}} + \frac{r_{t+1, I_{t+1}}- r_{t+1, i^*}}{2cK_\alpha \sqrt{t} } \eval \hat{L}_t] \hspace{-8em}&\\ 
    &\leq \frac{2\alpha \rho_1 mA_u}{A_l(\alpha-1)} K^{1-\frac{1}{\alpha}}\eta_t + \frac{K^{\frac{1}{\alpha}}C_{1,1}(\gD_\alpha)}{2cK_\alpha \sqrt{t}} \tag*{(by Lemmas~\ref{lem: stability} and~\ref{lem: penalty})} \\
    &= \qty(\frac{2\alpha \rho_1 mA_u c}{A_l(\alpha-1)} + \frac{C_{1,1}(\gD_\alpha)}{2c})\sqrt{\frac{K}{t}}. \numberthis{\label{eq: stoc gen term 2}}
\end{align*}
Combining (\ref{eq: stoc gen term 1}) and (\ref{eq: stoc gen term 2}) with (\ref{eq: stoc gen term all}) provides
\begin{align*}
    \gR(T)& \leq \sum_{t=1}^T \E\Bigg[\I[F_t] \frac{ \frac{2 e \alpha \rho_2  A_u}{A_l(\alpha+1)} S^{\frac{1}{\alpha}}+ 14.4 A_u^{\alpha} \rho_1 e^2 (1-e^{-1}) +\frac{e \rho_1  A_u^{\alpha}}{2(cK_\alpha)^2 (\alpha-1)}}{\hat{L}_{t,i}} \\
    & \hspace{11em}+  \I[F_t^c] \qty(\frac{2\alpha \rho_1 mA_u c}{A_l(\alpha-1)} + \frac{C_{1,1}(\gD_\alpha)}{2c})\sqrt{\frac{K}{t}} \Bigg] \\ 
    &\hspace{6em}+ \frac{M A_u \sqrt{K}}{c} + \frac{\rho_1 (e^2+1)}{2} \log(T+1)+ \mathcal{O}\qty(c^2 K_\alpha^2). \numberthis{\label{eq: stoc gen term after}}
\end{align*}
On the other hand, by Lemma~\ref{lem: lb bounded}, we have
\begin{equation}\label{eq: stoc gen lb term}
    \gR(T) \geq \sum_{t=1}^T \E\qty[\I[F_t] c_{s,1}(\gD_\alpha) \frac{\Delta_i t^{\frac{\alpha}{2}}}{(cK_\alpha \hat{\uL}_{t,i})^\alpha} + \I[F_t^c] c_{s,2}(\gD_\alpha) \Delta].
\end{equation}
By applying self-bounding technique, (\ref{eq: stoc gen term after}) - (\ref{eq: stoc gen lb term})$/2$, we have
\begin{align*}
    \frac{\gR(T)}{2} &\leq \sum_{t=1}^T \E\Bigg[\I[F_t] \sum_{i \ne i^*} \Bigg(\frac{ \frac{2 e \alpha \rho_2  A_u}{A_l(\alpha+1)} S^{\frac{1}{\alpha}}+ 14.4 A_u^{\alpha} \rho_1 e^2 (1-e^{-1}) +\frac{e \rho_1  A_u^{\alpha}}{2(cK_\alpha)^2 (\alpha-1)}}{\hat{L}_{t,i}} \\
    &\hspace{20em}- c_{s,1}(\gD_\alpha) \frac{\Delta_i t^{\frac{\alpha}{2}}}{2(cK_\alpha \hat{\uL}_{t,i})^\alpha}\Bigg)\Bigg] \\
    &\hspace{2em}+ \sum_{t=1}^T\E\qty[\I[F_t^c] \qty(\qty(\frac{2\alpha \rho_1 mA_u c}{A_l(\alpha-1)} + \frac{C_{1,1}(\gD_\alpha)}{2c})\sqrt{\frac{K}{t}}  -  c_{s,2}(\gD_\alpha) \Delta) ] \\
    &\hspace{10em} +  \frac{M A_u \sqrt{K}}{c}  + \frac{\rho_1 (e^2+1)}{2} \log(T+1)+ \mathcal{O}\qty(c^2 K_\alpha^2).
\end{align*}
Therefore, following the same steps as the \Fr distribution from (\ref{eq: stoc Fr last term}) concludes the proof.
For $\alpha \in (1,2)$, one can follow the same steps in the Fr\'{e}chet case.

\section{Numerical validation}
This section presents simulation results to verify our theoretical findings.
Following \citet{zimmert2021tsallis} and \citet{pmlr-v201-honda23a}, we consider the stochastically constrained adversarial setting.
The results in this section are the averages of $100$ independent trials.
Following \citet{pmlr-v201-honda23a}, we consider FTPL with a stable variant of geometric resampling (GR 10). 
In the stable variant, resampling (Lines 7--9 in Algorithm~\ref{alg: FTPL}) is iterated ten times, and the mean is calculated, leading to a reduction in the variance of $\widehat{w^{-1}_{t,i}}$.
We consider this stable variant to examine the effect of perturbations in FTPL more accurately.
\begin{figure*}
\begin{center}
    \begin{minipage}[b]{0.48\textwidth}
    \centering
    \includegraphics[width=\textwidth]{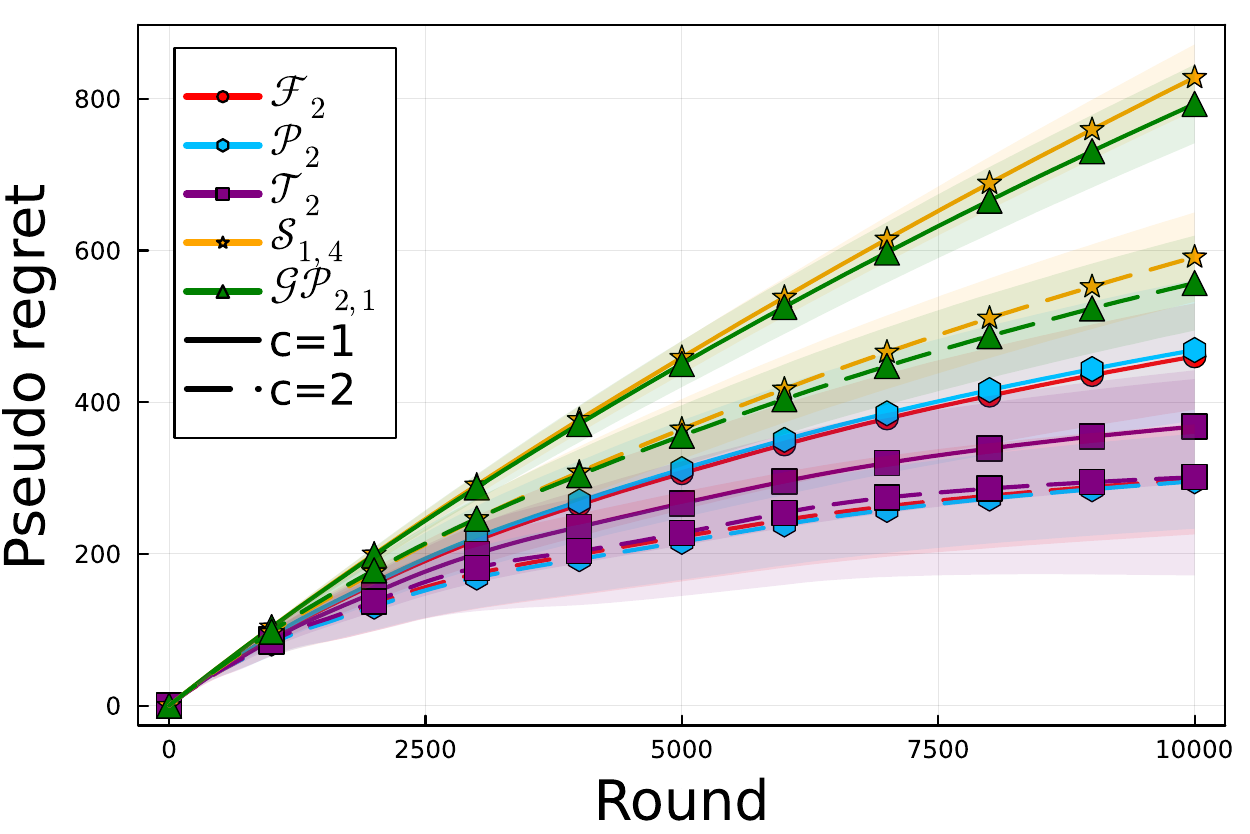}
    \vspace{-2em}
    \caption{Adversarial setting with $K=8$.}
    \label{fig: adv_8}
  \end{minipage}
  \begin{minipage}[b]{0.48\textwidth}
    \centering
    \includegraphics[width=\textwidth]{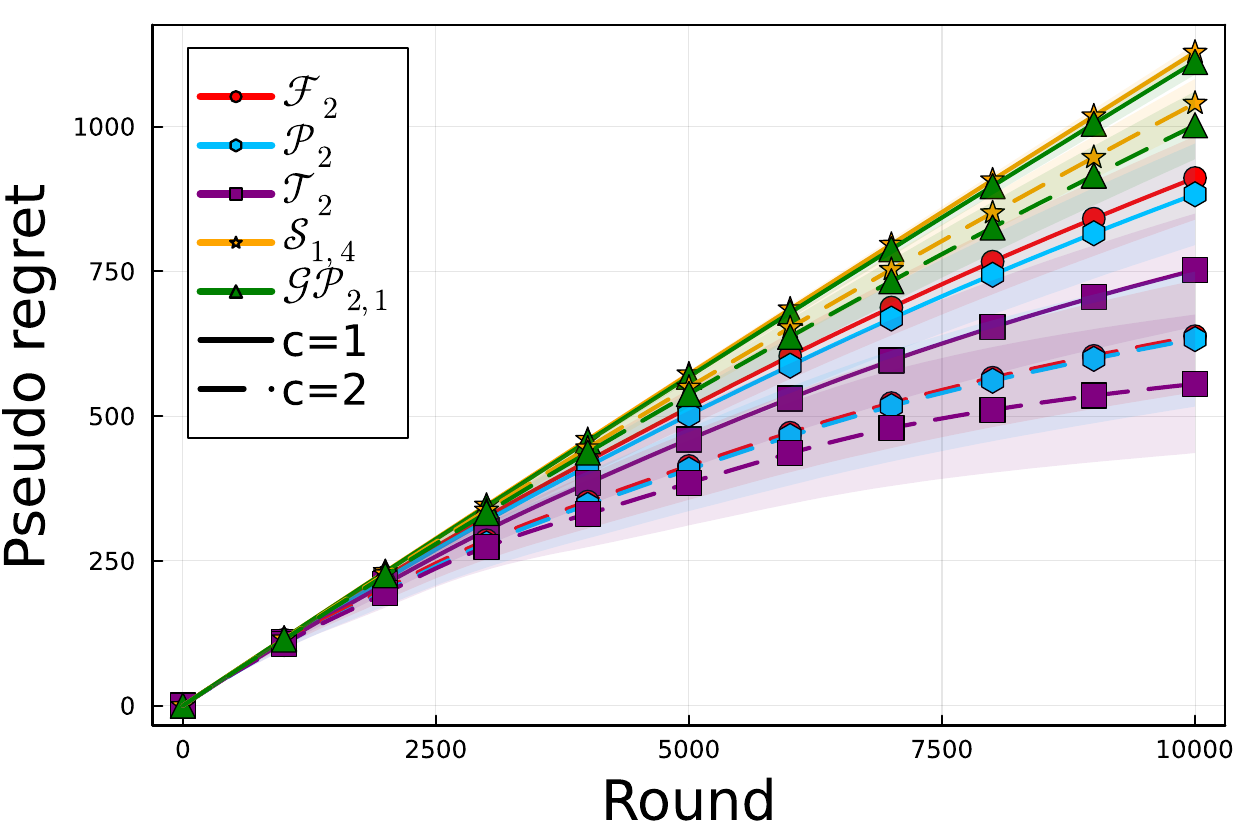}
        \vspace{-2em}
    \caption{Adversarial setting with $K=16$.}
    \label{fig: adv_16}
  \end{minipage}
    \begin{minipage}[b]{0.48\textwidth}
    \centering
    \includegraphics[width=\textwidth]{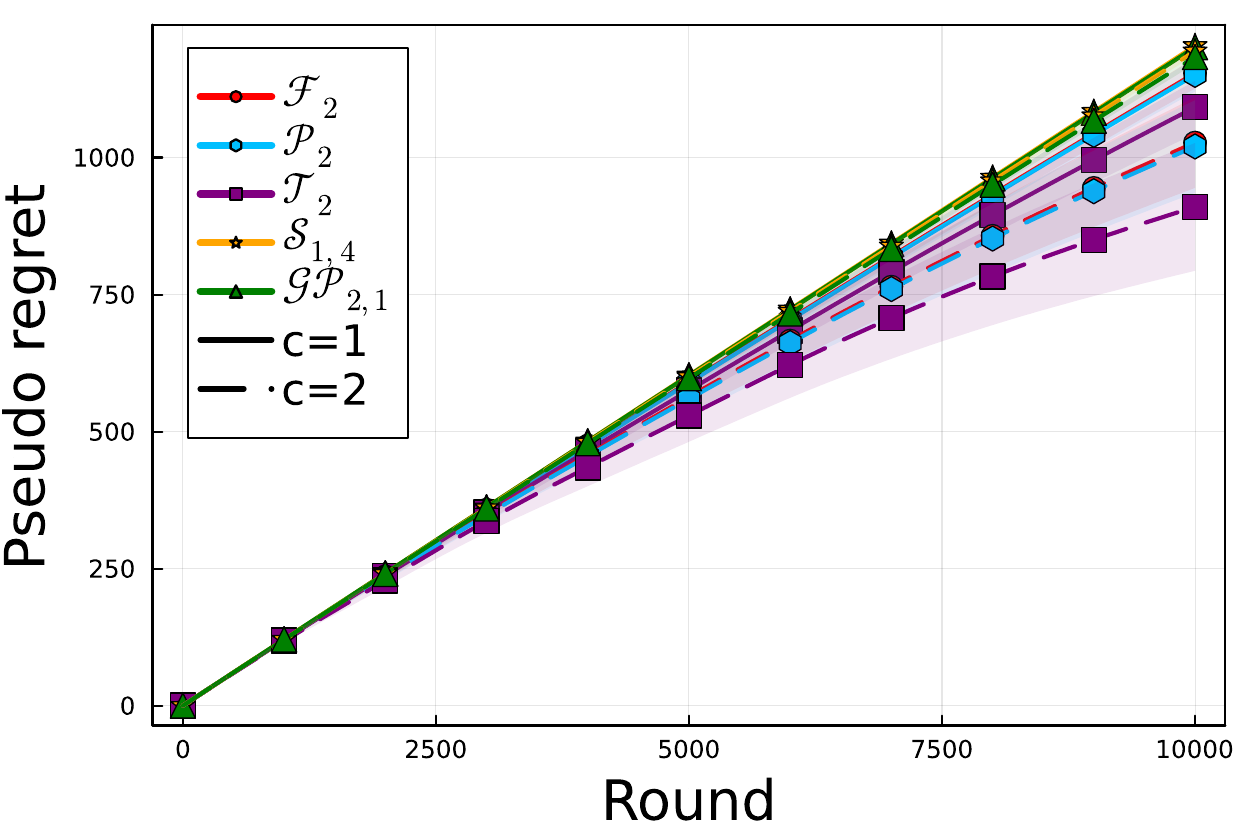}
        \vspace{-2em}
    \caption{Adversarial setting with $K=32$.}
    \label{fig: adv_32}
  \end{minipage}
\end{center}
  \vspace{-1em}
\end{figure*}

Since $K$ perturbations are independently generated from a common distribution, the behavior of FTPL is influenced by the distribution of maximum perturbations.
Therefore, in this experiment, we consider perturbations whose limiting distribution converges to the same Fr\'{e}chet distribution with shape $\alpha$. 
Since one can rewrite (\ref{eq: limiting}) as
\begin{equation*}
    \Pr[M_K/a_K \geq x] \stackrel{ K \to \infty}{\to} \I[x\geq 0] \exp(-x^{-\alpha}),
\end{equation*}
for $a_K = \inf \qty{x : F(x) \geq 1-1/K}$, we use denormalized perturbations $X=r a_K$ instead of $r$ generated from a common distribution $\gD_\alpha$.
This ensures that normalized block maxima of different perturbations converge to the same extreme distribution as $K$ increases.

Figures~\ref{fig: adv_8},~\ref{fig: adv_16}, and~\ref{fig: adv_32} are the results examining the behavior of FTPL in the adversarial setting using distributions from FMDA with index $\alpha=2$.
In these figures, the legends represent the original perturbations denoted by $r$, while FTPL employs denormalized perturbations $X$.
 Despite the absence of variance in $r$ for $\alpha=2$, the behavior of FTPL is almost the same as $K$ becomes sufficiently large.
This experimental observation supports our theoretical findings, demonstrating that the dominating factor in the behavior of FTPL is determined by the limiting distributions.


\section{Technical lemmas}
\begin{lemma}[Equation 8.10.2 of \citet{olver2010nist}]\label{lem: olver_lemma}
    For $x>0$ and $a \geq 1$,
    \begin{equation*}
        \gamma(a,x) \leq \frac{x^{a-1}}{a}(1-e^{-x}).
    \end{equation*}
\end{lemma}

\begin{lemma}[Gautschi's inequality]\label{lem: Gautschi_lemma}
    For $x>0$ and $s \in (0,1)$,
    \begin{equation*}
        x^{1-s} < \frac{\Gamma(x+1)}{\Gamma(x+s)} < (x+1)^{1-s}.
    \end{equation*}
\end{lemma}

\begin{lemma}\label{lem: incomplete beta}
    For any $\alpha >1$, $\frac{B\qty(x;1+\frac{1}{\alpha}, i)}{B(x;1,i)}$ is monotonically increasing with respect to $x \in (0,1]$.
\end{lemma}
\begin{proof}
    From the definition of the incomplete Beta function, $B(x;a,b)= \int_0^x t^{a-1}(1-t)^{b-1}\dd t$, we obtain
    \begin{align*}
        \dv{x} \frac{B\qty(x;1+\frac{1}{\alpha}, i)}{B(x;1,i)} &= \frac{1}{B^2(x;1,i)} \qty(x^{\frac{1}{\alpha}}(1-x)^{i-1} \int_0^x (1-t)^{i-1}\dd t - \int_0^x t^{\frac{1}{\alpha}}(1-t)^{i-1} \dd t (1-x)^{i-1}) \\
        &= \frac{(1-x)^{i-1}}{B^2(x;1,i)} \qty(\int_0^x x^{\frac{1}{\alpha}} (1-t)^{i-1}\dd t - \int_0^x t^{\frac{1}{\alpha}}(1-t)^{i-1} \dd t) \geq 0,
    \end{align*}
    which concludes the proof.
\end{proof}

\begin{lemma}[Potter bounds~\citep{beirlant2006statistics}]\label{lem: potter}
    Let $S(x)$ be a slowly varying function. Given $A>1$ and $\delta>0$, there exists a constant $x_0$ such that
    \begin{equation*}
        \frac{S(y)}{S(x)}\leq A \max\qty{\qty(\frac{y}{x})^{\delta}, \qty(\frac{x}{y})^{\delta}}, \qquad \forall x,y \geq x_0.
    \end{equation*}
\end{lemma}



\end{document}

%% file: math_commands.tex
\usepackage{amsmath, amsthm, bm, amssymb, mathtools}
\usepackage{physics}

\def\eqref#1{equation~\ref{#1}}

\def\1{\bm{1}}

\def\inp#1#2{\left\langle #1, #2 \right\rangle}

\def\fD{{\mathfrak{D}}}

\DeclareMathAlphabet{\mathsfit}{\encodingdefault}{\sfdefault}{m}{sl}
\SetMathAlphabet{\mathsfit}{bold}{\encodingdefault}{\sfdefault}{bx}{n}

\def\gD{{\mathcal{D}}}

\def\gF{{\mathcal{F}}}
\def\gG{{\mathcal{G}}}

\def\gO{{\mathcal{O}}}
\def\gP{{\mathcal{P}}}

\def\gR{{\mathcal{R}}}
\def\gS{{\mathcal{S}}}
\def\gT{{\mathcal{T}}}

\def\sR{{\mathbb{R}}}

\newcommand{\E}{\mathbb{E}}

\newcommand{\I}{\mathbbm{1}}

\DeclareMathOperator*{\argmin}{arg\,min}